\pdfoutput=1
\documentclass{article}
\usepackage[T1]{fontenc}
\usepackage{iclr2027_conference,times}
\usepackage{graphicx}
\usepackage{subcaption}
\usepackage{amsmath,amssymb,amsthm}
\usepackage{booktabs}
\usepackage{microtype}
\usepackage{bm}
\usepackage{float}
\usepackage{tikz}
\usetikzlibrary{positioning,arrows.meta,fit,calc}
\usepackage[hidelinks]{hyperref}
\usepackage{url}

\usepackage{aliascnt}
\newtheorem{theorem}{Theorem}[section]
\newaliascnt{corollary}{theorem}
\newtheorem{corollary}[corollary]{Corollary}
\aliascntresetthe{corollary}
\newaliascnt{conjecture}{theorem}

\aliascntresetthe{conjecture}
\newaliascnt{lemma}{theorem}
\newtheorem{lemma}[lemma]{Lemma}
\aliascntresetthe{lemma}
\newaliascnt{proposition}{theorem}
\newtheorem{proposition}[proposition]{Proposition}
\aliascntresetthe{proposition}
\theoremstyle{definition}
\newaliascnt{definition}{theorem}
\newtheorem{definition}[definition]{Definition}
\aliascntresetthe{definition}
\newaliascnt{problem}{theorem}

\aliascntresetthe{problem}
\newaliascnt{assumption}{theorem}

\aliascntresetthe{assumption}
\theoremstyle{remark}
\newaliascnt{remark}{theorem}
\newtheorem{remark}[remark]{Remark}
\aliascntresetthe{remark}

\newcommand{\Var}{\operatorname{Var}}
\newcommand{\Cov}{\operatorname{Cov}}

\newcommand{\bX}{\mathbf{X}}

\newcommand{\bY}{\mathbf{Y}}

\newcommand{\bSigma}{\mathbf{\Sigma}}

\newcommand{\bg}{\mathbf{g}}

\newcommand{\bx}{\mathbf{x}}

\newcommand{\bz}{\mathbf{z}}
\newcommand{\bh}{\mathbf{h}}

\newcommand{\cO}{\mathcal{O}}
\newcommand{\cOs}{\widetilde{\mathcal{O}}}

\newcommand{\cS}{\mathcal{S}}
\newcommand{\cN}{\mathcal{N}}

\newcommand{\cL}{\mathcal{L}}

\newcommand{\cD}{\mathcal{D}}

\newcommand{\bbE}{\mathbb{E}}

\newcommand{\bbP}{\mathbb{P}}
\newcommand{\bbR}{\mathbb{R}}

\newcommand{\csch}{\operatorname{csch}}
\newcommand{\sech}{\operatorname{sech}}

\newcommand{\inner}[1]{\left\langle #1 \right\rangle}
\newcommand{\rbr}[1]{\left( #1 \right)}
\newcommand{\sbr}[1]{\left[ #1 \right]}
\newcommand{\wbr}[1]{\left\{ #1 \right\}}

\newcommand{\sumin}{\sum_{i = 1}^n}

\newcommand{\ttheta}{t_{\theta}}
\newcommand{\tlambda}{t_{\lambda}}
\newcommand{\tcrit}{t_{\text{crit}}}

\title{Adaptive Mean Estimation by In-Context Learning: A Gradient-Flow Analysis}

\author{
Martin Eppert$^{1}$ \And
Krishna Balasubramanian$^{2}$ \And
Subhro Ghosh$^{1}$ \AND
Jason Klusowski$^{3}$ \And
Yan Shuo Tan$^{1}$\\
$^{1}$National University of Singapore\\
$^{2}$University of California, Davis\\
$^{3}$Princeton University
}

\iclrfinalcopy %

\begin{document}

\maketitle
\lhead{Preprint}
\begin{abstract}
Prior Fitted Networks (PFNs) such as TabPFN~\citep{zhang2025tabpfnonemodel} now rival established statistical procedures across prediction and estimation tasks.
A natural explanation is that PFNs have the property of \emph{statistical adaptivity}, that is, they perform nearly as well as a method tailored to the true data-generating model for a heterogeneous set of models, while not being told which model the data comes from.
Prior work does not give a conclusive explanation of how gradient-based pretraining learns an adaptive estimator.
We study how such adaptivity is learned in a controlled location-estimation problem.
Each task is an unlabeled sample whose family is hidden: Gaussian data call for averaging, with error of order $n^{-1}$, whereas uniform data are best estimated from their extremes, at the faster rate $n^{-2}$.
We also provide the example of a symmetric Gaussian mixture, for which a rate of $\sigma^2_n/n$ can be attained.
On scalar inputs, softmax attention computes the derivative of the empirical cumulant-generating function.
A single primitive therefore both supplies features that distinguish the families and forms estimators interpolating between the sample mean and the mid-range.
Building on this, we combine attention experts through either a softmax mixture of experts or a gated linear unit (GLU), and analyze stagewise gradient flow.
With $\widetilde{\Omega}(n^{1+\epsilon})$ pretraining tasks, the learned estimator is asymptotically efficient on Gaussian tasks, within a factor $n^{\epsilon}$ of the minimax rate on uniform tasks, and order-optimal on mixtures in a shrinking-variance regime.
These guarantees extend to new locations and longer contexts.
A risk decomposition separates expert error, routing error and normalization error, which clarifies the architectural contrast.
Softmax gating enforces normalization and exact translation equivariance, whereas the GLU must learn it: its dynamics separate into fast bias removal followed by slow expert selection.
End-to-end experiments recover the predicted specialization.
\end{abstract}

\section{Introduction}
\label{sec:introduction}
In a landmark paper, \citet{brown2020languagemodelsfewshotlearners} showed that large language models can learn to perform new tasks from examples supplied in a prompt (also called a context), without updating their parameters at inference time, a capability they called in-context learning (ICL).
\citet{garg2023transformerslearnincontextcase} formulated ICL as a new paradigm for statistical learning: whereas the classical approach fits a \emph{predictor} to each new dataset through parameter optimization, ICL pretrains a reusable prediction \emph{procedure} across datasets and applies it to a new dataset through context alone.
Tabular foundation models (TFMs) have made this paradigm practical for tabular prediction, with evaluations on the broad TabArena and TALENT benchmarks showing that TFMs can even outperform previously state-of-the-art methods such as boosted tree ensembles~\citep{erickson2025tabarena,ye2025closer}.
Other recent studies show how the ICL paradigm can outperform well-established classical methods in estimating various statistical quantities, including heterogeneous treatment effects~\citep{balazadeh2025causalpfn}, survival curves~\citep{qi2026survivalpfn}, mutual information~\citep{peyrard2025metastatisticallearningsupervisedlearning}, and Poisson means in empirical Bayes problems~\citep{teh2025empiricalbayes}.
These works suggest that the central theoretical question for ICL should not be whether it can learn to approximate an existing algorithm or statistical procedure, but \emph{how it is able to outperform established procedures across diverse statistical problems}.

Several theoretical explanations have been developed by analyzing empirical risk minimization (ERM) of transformers.
\citet{bai2023transformersstatisticiansprovableincontext} show that transformers can implement in-context algorithm selection, including choosing among ridge penalties using a train-validation split.
\citet{wakayama2025bayesian} and \citet{ma2025adaptivity} establish approximation of the posterior predictive mean under hierarchical priors over regression-function classes.
In particular, \citet{ma2025adaptivity} use posterior contraction to show that this predictor achieves optimal rates corresponding to the unknown smoothness and effective dimension of the task.
These results suggest that part of ICL's strength may lie in its ability to adjust the statistical procedure to the structure revealed by the context.
This is a form of \emph{statistical adaptivity}, which is when a single estimator attains a close to optimal accuracy appropriate to the underlying model for multiple models, without knowing that model in advance.
Such adaptation is a longstanding objective in statistics, often achieved through procedures tailored to particular models and inferential targets~\citep{bickel1993efficient}.
Pretraining offers a potentially general recipe to acquiring this behavior, but guarantees for ERM do not establish that gradient-based training learns the required mechanisms.

Understanding how such adaptation is learned requires an analysis of pretraining itself.
A complementary literature studies pretraining dynamics directly.
Results include gradient-flow convergence for linear regression with linear attention~\citep{zhang2023trainedtransformerslearnlinear,zhang2024incontextlearninglineartransformer}, task allocation in multi-head softmax attention for multi-task linear regression~\citep{chen2024trainingdynamicsmultiheadsoftmax}, and gradient-based learning of shared low-dimensional structure for nonlinear single-index regression~\citep{oko2024pretrainedtransformer}.
These results explain important forms of learned statistical behavior, including specialization and adaptation to shared structure.

In this paper, \emph{we study how adaptation can emerge from gradient-based pretraining, focusing on a simple one-dimensional location-estimation problem}.
Here, each task draws a context comprising $n$ i.i.d.\ observations from either a Gaussian or a uniform distribution, and the estimation target is the location parameter of the underlying distribution.
The appropriate estimator depends on the distributional family.
For Gaussian observations, a sample average is optimal, with mean-squared error of order $n^{-1}$, whereas uniform observations permit the faster $n^{-2}$ rate through the sample mid-range.
Neither estimator attains both rates.
Moreover, the family label is hidden and has to be inferred from the context.
During pretraining, the learner observes the target mean but receives neither the family label nor a prescribed estimator choice.
We also consider a symmetric two-component Gaussian mixture as the alternative to the Gaussian family.
Its shrinking component variance $\sigma_n^2$ permits risk of order $\sigma_n^2/n$, providing a second example of adaptation to distributional structure.

Multi-head attention can be used to identify the distribution of the samples in context as well as to estimate their location.
With scalar observations, the query and key matrices of an attention head collapse to a single scalar which we denote by $\lambda$. %
Thus, attention effectively computes the exponentially tilted average $h(\bx,\lambda)=\sum_i x_i e^{\lambda x_i}/\sum_i e^{\lambda x_i}$.
This is the derivative of the empirical cumulant-generating function (CGF).
Since the population CGF can be used to characterize distributions, evaluating this function at several locations supplies features for distinguishing between the two families.
The same operation also constructs the relevant estimators: setting $\lambda=0$ gives the sample mean, while sending $\lambda \to \pm \infty$ approaches the sample maximum and minimum, which are required to compute the mid-range. %

We use these two roles of multi-head attention to construct a bank of estimators, or \emph{experts}, as well as a feature embedding, that allows distinguishing different distributions.
A learned gating parameter $\theta_i$ is used within a gating layer, either a Softmax MoE or a Gated Linear Unit, to output a weighted combination of the outputs of the different experts.
This weighted sum of the expert estimates forms the location estimate.
Both of these architectures are minor modifications of commonly used layers within standard Transformer architectures~\citep{shazeer2020glu}.
We analyze stagewise gradient flow that first learns the gating parameters $\theta_i$ with the expert parameters $\lambda_i$ fixed, then learns the expert parameters $\lambda_i$ with the gating parameters $\theta_i$ fixed.
This is also done to simplify the proof structure, while our experiments demonstrate that this is not strictly necessary.
Notably, our results are stronger than an expressivity statement, as they demonstrate that gradient flow on the empirical loss is able to learn how to do in context location estimation with close to optimal rates.

Our contributions are fourfold:
\begin{itemize} %
\item \textbf{Adaptive estimation and transfer.} 
We establish family-specific risk guarantees for the output of stagewise pretraining under both the Softmax MoE and the GLU architecture.
Both architectures achieve Gaussian asymptotic efficiency and nearly (up to a $n^\epsilon$ factor) optimal uniform estimation rates.
For a Gaussian mixture distribution, we also obtain order-optimal rates.
This implies that our architecture performs nearly as well as the sample mean on Gaussian data and the sample mid-range on uniform data. 
We also establish out-of-distribution risk bounds for longer contexts as well as for samples with a fixed location.
\item \textbf{Statistical mechanism and architectural comparison.} 
Our results show how the risk can be decomposed into the variance of the experts, the bias of the final estimator, and the error accumulated through the gating mechanism (\autoref{thm:accuracy_final_estimator}).
They explain how distribution-sensitive weighting preserves the accuracy of the appropriate expert, and how the architectures differ: softmax enforces normalization of the weights and thus exact translation equivariance, whereas the GLU must learn approximate normalization of the weights (\autoref{thm:fresh_context}).
\item \textbf{Learning dynamics.} We show how training learns the gating parameters and expert parameters required by the statistical analysis (\autoref{thm:softmax_dynamics} and \autoref{convergence_of_first_layer}).
Since the loss admits a bias-variance decomposition, our analysis tracks the movement of the bias and variance separately.
The major terms of the bias are created by the sum of the weights, while variance is accumulated through the variance of the separate experts.
Additionally, we observe that the bias is scaled at a different order, proportional to the prior variance $\sigma^2_\mu$, than the variance which is of order $\frac{1}{n}$ where $n$ is the sample size.
Thus, bias and variance move at different timescales.
We exploit this fact to characterize the dynamics of the parameters (\autoref{convergence_to_critical_manifold} and \autoref{application_of_thikonovs_thm}). 
Subsequent training of the attention layer lets the experts specialize their predictions to their assigned distribution (\autoref{convergence_of_first_layer}).
Here, our analysis exploits that the second layer has already converged to a near-perfect gating of the different distributions.
\item \textbf{Finite Sample Guarantees.}
Our analysis demonstrates that for $m = \widetilde{\Omega}(n^{1+\epsilon})$ pretraining tasks, the gradient flow on the empirical loss $\cL_m$ converges to the correct solution.
\end{itemize}

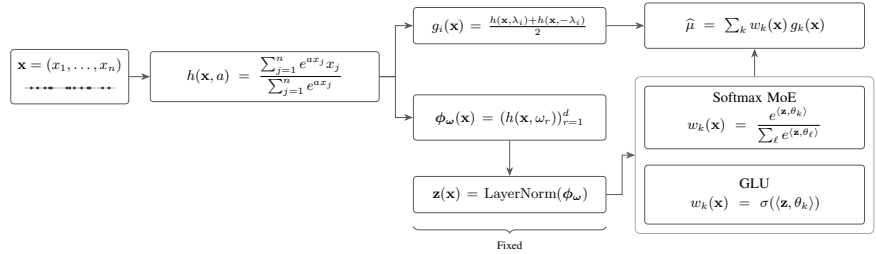
\begin{figure}[H]
\centering
\hspace{-0.5cm}
\scalebox{0.6}{
    \usetikzlibrary{decorations.pathreplacing}
\begin{tikzpicture}[
  font=\small, >={Stealth[length=2mm]},
  bx/.style   ={draw=black!65, rounded corners=2pt, align=center, inner sep=3.5pt},
  expert/.style={bx, text width=40mm, minimum height=8mm},
  gate/.style ={bx, text width=46mm, minimum height=13mm},
  ar/.style   ={->, semithick, draw=black!65},
  hd/.style   ={font=\scriptsize\bfseries, text=black!55}
]

\node[bx, minimum width=26mm, minimum height=12mm] (ctx) at (0,0) {};
\node at ([yshift=2.5mm]ctx.center) {$\mathbf x=(x_1,\dots,x_n)$};
\draw[black!30] ([xshift=3mm,yshift=-2.5mm]ctx.west) -- ([xshift=-3mm,yshift=-2.5mm]ctx.east);
\foreach \p in {0.08,0.15,0.23,0.28,0.3,0.47,0.50,0.55,0.63,0.7,0.72,0.86,0.93}
  \fill[black!80] ($([xshift=3mm,yshift=-2.5mm]ctx.west)!\p!([xshift=-3mm,yshift=-2.5mm]ctx.east)$) circle (.9pt);

\node[bx, text width=48mm] (attn) at (4.3,0)
  {$h(\mathbf x, a)=\dfrac{\sum_{j=1}^n e^{a x_j}x_j}{\sum_{j=1}^n e^{a x_j}}$};

\node[expert] (e1) at (9.7,1.125) {$g_i(\bx)=\tfrac{h(\mathbf x,\lambda_i)+h(\mathbf x,-\lambda_i)}{2}$};

\node[expert, minimum height=10mm] (phi) at (9.7,-0.90)
  {$\boldsymbol\phi_{\boldsymbol\omega}(\mathbf x)=(h(\mathbf x,\omega_r))_{r=1}^d$};
\node[expert] (z) at (9.7,-2.60)
  {$\mathbf z(\bx)=\operatorname{LayerNorm}(\boldsymbol\phi_{\boldsymbol\omega})$};
\draw[decorate, decoration={brace, mirror, amplitude=4pt}, draw=black!65]
  ([yshift=-3mm]z.south west) -- node[below=6pt, font=\scriptsize] {Fixed}
  ([yshift=-3mm]z.south east);

\node[gate] (sm) at (15.1,-0.90)
  {Softmax MoE\\[2pt]
   $w_k(\bx)=\dfrac{e^{\langle\mathbf z,\theta_k\rangle}}{\sum_\ell e^{\langle\mathbf z,\theta_\ell\rangle}}$};
\node[gate] (ig) at (sm.center |- z.center)
  {GLU\\[2pt]
   $w_k(\bx)=\sigma\!\left(\langle\mathbf z,\theta_k\rangle\right)$};
\node[draw=black!35, rounded corners=3pt, inner sep=2mm,
      fit=(sm)(ig)] (gating) {};

\node[bx, text width=46mm, minimum height=10mm] (out)
  at (sm.center |- e1.center)
  {$\widehat\mu=\sum_{k} w_k(\mathbf x)\,g_k(\mathbf x)$};

\draw[ar] (ctx.east) -- (attn.west);
\draw[ar] (attn.east) -- ++(0.35,0) |- (e1.west);
\draw[ar] (attn.east) -- ++(0.35,0) |- (phi.west);
\draw[ar] (phi.south) -- (z.north);
\draw[ar] (z.east) -- ++(0.35,0) |- (gating.west);
\draw[ar] (e1.east) -- (out.west);
\draw[ar] (gating.north) -- (out.south);
\end{tikzpicture}
}
\caption{
Illustration of our studied architecture. 
We also highlight the parts which were fixed for simplicity.
} %
\label{fig:architecture}
\end{figure}

\paragraph{Notation.}
\label{par:notation}
We write $[n]=\{1,\ldots,n\}$.
The norms $\|\cdot\|_2$ and $\|\cdot\|_{\mathrm F}$ denote the Euclidean and Frobenius norms, respectively.
Parameter collections and their gradients are identified with stacked vectors.
Let $e_i$ be the $i$-th coordinate vector.
The notation $\cO$, $\Theta$, $\Omega$, and $o$ has its usual meaning, $\cOs$ and $\widetilde\Omega$ suppress logarithmic factors, and $\lesssim$ denotes inequality up to an absolute multiplicative constant.
Constants may change from line to line.
Uppercase $X_i$ and $\bX$ denote random observations and contexts, while lowercase $x_i$ and $\bx=(x_1,\ldots,x_n)$ denote their realizations.
We define estimators as functions of $\bx$ and evaluate their risks at the random context $\bX$.
We use
$\bbE_j[\,\cdot\,]=\bbE[\,\cdot\mid\tau=j]$
for expectation under family $j$, including both the location prior and the conditional sample.
We write $\bbP_j$ for this family-conditional law.
For $1\le p<\infty$, $\|Y\|_{L^p(\bbP_j)}=(\bbE_j[|Y|^p])^{1/p}$, using the Euclidean norm for random vectors and the unconditional law when the family subscript is omitted.

\section{Problem and architectures}
\label{sec:problem}
We consider a location-estimation task.
During training and inference, the learner is presented with observations from a distribution in context and is tasked with estimating their location.
The learner is not explicitly told which distribution the observations are from.
To allow a fair comparison, the distributions differ only in shape, not in variance.

\begin{definition}[Adaptive mean-estimation tasks]
\label{adaptiveMeanEstimation}
Sample a location $\mu$ and a family label $\tau$ independently:
\[
 \mu\sim\cN(0,\sigma_\mu^2),\qquad
 \bbP(\tau=\mathrm G)=\bbP(\tau=\mathrm A)=\tfrac12
\]
Conditional on $(\mu,\tau)$, draw $\xi_i\overset{\mathrm{i.i.d.}}{\sim}Q_\tau$ and set
\[
 X_i=\mu+\xi_i,\qquad i\in[n],\qquad
 \bX=(X_1,\ldots,X_n).
\]
The centered Gaussian law is $Q_{\mathrm G}=\cN(0,1)$.
The alternative centered law $Q_{\mathrm A}$ is either %
\begin{align*}
 \text{Uniform:}\quad &Q_{\mathrm A}=Q_{\mathrm U}=\mathrm{Unif}[-\sqrt3,\sqrt3],\\
 \text{Gaussian mixture:}\quad
 &Q_{\mathrm A}=Q_{\mathrm M}=\tfrac12\cN(-\sqrt{1-\sigma_n^2},\sigma_n^2)
       +\tfrac12\cN(\sqrt{1-\sigma_n^2},\sigma_n^2),
\end{align*}
where $\sigma_n^2\in(0,1)$.
In statements specific to one alternative, we write $\mathrm U$ or $\mathrm M$ in place of $\mathrm A$.
Write $P_{\mu,j}$ for the law of $\mu+\xi$ with $\xi\sim Q_j$.
Each observed task $(X, \mu)$ is the tuple of the context and the location and thus, the family label $\tau$ is never observed.
\end{definition}

\textbf{Risk and Lower Bounds.}
For $j\in\{\mathrm G,\mathrm A\}$, define the risk at location $\mu$ and its average over the location prior by
\[
 R_{j,n}(\widehat\mu;\mu):=\bbE_{\bX\sim P_{\mu,j}^{\otimes n}}[(\widehat\mu(\bX)-\mu)^2],
 \qquad
 R_{j,n}(\widehat\mu):=\bbE_j[(\widehat\mu(\bX)-\mu)^2].
\]
For the trained estimator, these expectations are over an independent test task with the trained parameters held fixed.
For a fixed family $j$, define the location-minimax risk
\[
 R_{j,n}^{\star}:=\inf_{\widetilde\mu}\sup_{\mu\in\bbR}
 \bbE_{\bX\sim P_{\mu,j}^{\otimes n}}
 [(\widetilde\mu(\bX)-\mu)^2].
\]  
The sample mean has risk $1/n$ under every family.
For Gaussian data, the sample mean attains the exact location-minimax risk $n^{-1}$.
For uniform data, the sample mid-range attains the faster location-minimax rate $n^{-2}$.
For completeness, we give the proofs of these benchmarks in \autoref{app:benchmark_proof}.
For the Gaussian mixture, revealing component labels reduces the problem to Gaussian location estimation with variance $\sigma_n^2$, yielding the lower bound $\sigma_n^2/n$ (see \autoref{gsr_fisher}).

\subsection{Architecture}
\label{sec:architecture}
We consider a two-layer transformer architecture that consists of one multi-head attention layer followed by a second layer, which is either a Softmax Mixture of Experts (MoE) layer or a Gated Linear Unit.
In the following, we will fix some of the parameters of the transformer architecture to decide before training which attention heads will be used as estimators and which will be used to identify the distribution.
\begin{definition}[Transformer architecture]
\label{TransformerSimplified}
Let $K$ be the number of experts and $d \geq 3$ be the number of heads dedicated to identifying the distribution family.
Let $\bx$ be the context of size $n$, containing the input dataset.
For $\lambda \in \bbR$, define
\begin{equation}
 h(\bx, \lambda)=\frac{\sum_{i=1}^n e^{\lambda x_i}x_i}{\sum_{i=1}^n e^{\lambda x_i}}.
 \label{eq:attention_output}
\end{equation}
Expert $k\in[K]$ is $g_k(\bx)=\tfrac12\rbr{h(\bx,\lambda_k)+h(\bx,-\lambda_k)}$ with $\lambda_k(0)\stackrel{\mathrm{i.i.d.}}{\sim}\mathrm{Unif}[-1,1]$.
For gating, sample $\omega_r\stackrel{\mathrm{i.i.d.}}{\sim}\cN(0,1)$, $r\in[d]$, and hold $\boldsymbol\omega=(\omega_1,\ldots,\omega_d)$ fixed.
We define 
$
\boldsymbol\phi_{\boldsymbol\omega}(\bx) = (h(\bx,\omega_1),\ldots,h(\bx,\omega_d))^\top
$
and then apply a LayerNorm with fixed parameters to obtain $\bz$
\begin{align}
 \bz(\bx)
 = \operatorname{LayerNorm}(\boldsymbol\phi_{\boldsymbol\omega}(\bx))
 =\frac{P_d\boldsymbol\phi_{\boldsymbol\omega}(\bx)}
 {\|P_d\boldsymbol\phi_{\boldsymbol\omega}(\bx)\|_2},
 \label{eq:routing_features}
\end{align}
where $P_d = I_d-\mathbf1\mathbf1^\top/d$.
For $k\in[K]$, initialize $\theta_k(0)\stackrel{\mathrm{i.i.d.}}{\sim}\cN(0,\sigma_\theta^2I_d)$.
We will consider two different choices for $w$, the Softmax Mixture of Experts \eqref{eq:softmax_moe} and the Gated Linear Unit \eqref{eq:glu}:

\noindent
\begin{minipage}{0.42\linewidth}
\begin{equation}
 w_k(\bx;\theta)=\frac{\exp(\inner{\bz(\bx),\theta_k})}
 {\sum_{\ell=1}^K\exp(\inner{\bz(\bx),\theta_\ell})},
 \label{eq:softmax_moe}
\end{equation}
\end{minipage}%
\begin{minipage}{0.58\linewidth}
\begin{equation}
w_k(\bx;\theta)=\sigma\!\left(\inner{\bz(\bx),\theta_k}\right) = \frac1{1+e^{-\inner{\bz(\bx),\theta_k}}}.
\label{eq:glu}
\end{equation}
\end{minipage}
Finally, we define the final network output as
\begin{align}
\widehat\mu_{\theta,\lambda}(\bx) = \sum_{k=1}^K w_k(\bx;\theta)g_k(\bx).
\label{eq:network_output}
\end{align}
The parameters $\theta=(\theta_1,\ldots,\theta_K)$ and $\lambda=(\lambda_1,\ldots,\lambda_K)$ are trained, while $\boldsymbol\omega$ remains fixed.
\end{definition}
The architecture in \autoref{TransformerSimplified} fixes some parameters to remove permutation invariances.
The attention heads are split into the experts $g_k$, which implement the candidate estimators, and the heads forming the feature $\bz$, which identifies the distribution family.
When the routing feature is replaced by its population value, the unbiased experts $g_k$ give an unbiased Softmax MoE output, whereas the GLU must also learn normalization.
For sample-dependent gates, softmax guarantees exact translation equivariance but need not be exactly unbiased.
The feature does not depend on the location, $\bz(\bX+\mu\mathbf1)=\bz(\bX)$, so its distribution depends only on the family.
Under family $j$, $\bz(\bX)\to\bar\bz_j$ in probability as $n\to\infty$, where $\bar\bz_j$ is the centered and normalized vector of derivatives of the CGF of $Q_j$ at $\omega_1,\ldots,\omega_d$.
Since $\bar\bz_{\mathrm G}\ne\bar\bz_{\mathrm A}$ almost surely (\autoref{feature_geometry}), $\bz$ distinguishes the families for large $n$.

Our architecture introduces a LayerNorm for the gating input to the second layer, as this simplifies calculations.
We believe this is not strictly necessary. 
For details, see \autoref{role_of_layernorm}.
\subsection{Training protocol}
\label{sec:training_protocol}
For our training dataset, we sample $m$ i.i.d.\ tasks of size $n$ from \autoref{adaptiveMeanEstimation}
and denote our dataset by $\{(\bX_b,\mu_b)\}_{b=1}^m$.
\begin{definition}[Stagewise gradient flow training]
\label{training_procedure}
For either architecture, \eqref{eq:softmax_moe} or \eqref{eq:glu}, define the population loss 
$\cL(\theta,\lambda)=\frac12\bbE[(\widehat\mu(\bX)-\mu)^2]$
and the empirical loss
$\hat\cL_m(\theta,\lambda) = \frac1{2m}\sum_{b=1}^m(\widehat\mu(\bX_b)-\mu_b)^2.$
Training proceeds in two stages:
\begin{enumerate}
\item For $0\le t<\ttheta$, run $\dot\theta=-\nabla_\theta\hat\cL_m$ with $\lambda$ fixed.
\item For $\ttheta\le t<\tlambda$, run $\dot\lambda=-\nabla_\lambda\hat\cL_m$ with $\theta$ fixed.
\end{enumerate}
The full population flow is the same procedure with $\hat\cL_m$ replaced by $\cL$ in both stages.
\end{definition}

\section{Main results}
\label{sec:main}
\begin{theorem}[Main result; proof in \autoref{pf_theorem_assembly}]
\label{thm:fresh_context}
Let $K=2$ be the number of experts.
Fix any $\eta,\epsilon\in(0,1)$, a gating dimension $d\ge3$, and sufficiently small $\sigma_\mu^2>0$, all independent of $n$.
Sample $m=\widetilde\Omega(n^{1+\epsilon})$ independent tasks according to \autoref{adaptiveMeanEstimation}, with $m\in\operatorname{poly}(n)$.
In the Gaussian-mixture case, assume $n^{-1+\epsilon}\le\sigma_n^2 \le o(1)$.
Initialize as in \autoref{TransformerSimplified}, with fixed $\sigma_\theta^2>0$, sufficiently small depending on $\eta$.
Relabel the experts and permute $\theta$ accordingly so that $|\lambda_1(0)|<|\lambda_2(0)|$.
Let
$\widehat\mu_{\mathrm{trained}}=\widehat\mu_{\theta(\ttheta),\lambda(\tlambda)}$, with $\tlambda$ as in \autoref{convergence_of_first_layer}.
For sufficiently large $n$, with probability at least $1-\eta$ over the initialization and training tasks,
\[
 R_{\mathrm G,n}(\widehat\mu_{\mathrm{trained}})=\frac{1+o(1)}n,\qquad
 R_{\mathrm A,n}(\widehat\mu_{\mathrm{trained}})=
 \begin{cases}
 \cO(n^{-2+\epsilon}),&\text{uniform alternative},\\
 \cO(\sigma_n^2/n),&\text{Gaussian-mixture alternative}.
 \end{cases}
\]
For the GLU, we have, for every $\mu\in\bbR$ and $j\in\{\mathrm G,\mathrm A\}$,
\begin{gather*}
 R_{j,n}(\widehat\mu_{\mathrm{trained}};\mu)
 \le 2R_{j,n}(\widehat\mu_{\mathrm{trained}})+\frac{C\mu^2}{n^2}.
\end{gather*}
For the softmax MoE, for every $\mu\in\bbR$ and $j\in\{\mathrm G,\mathrm A\}$,
\[
 R_{j,n}(\widehat\mu_{\mathrm{trained}};\mu)
 =R_{j,n}(\widehat\mu_{\mathrm{trained}}).
\]
\end{theorem}
The Gaussian and Gaussian-mixture rates match the lower bounds of \autoref{prop:location_benchmarks} and \autoref{gsr_fisher} up to constants, and the uniform rate is within a factor of $n^{\epsilon}$ of the minimax rate $n^{-2}$.
Table~\ref{tab:main_results} summarizes the results.
\begin{table}[ht]
\centering
\caption{The terminal expert parameters and fresh-context risks certified by the main theorem.}
\label{tab:main_results}
\begin{tabular}{@{}lllll@{}}
\toprule
Alternative family & terminal $|\lambda_2|$ & Gaussian risk & Alternative risk & $R_{\mathrm A,n}^{\star}$ \\
\midrule
Uniform & $n^{1-\epsilon}$ & $(1+o(1))/n$ & $\cO(n^{-2+\epsilon})$ & $\asymp n^{-2}$ \\
Gaussian mixture & $\frac{\log(1/\sigma_n^2)}{4\sqrt{1-\sigma_n^2}}$ & $(1+o(1))/n$ & $\cO(\sigma_n^2/n)$ & $\ge\sigma_n^2/n$ \\
\bottomrule
\end{tabular}
\end{table}
The Softmax MoE architecture has a stronger guarantee when $|\mu|$ is large.
This follows from its exact translation equivariance, while the GLU has to learn approximate normalization.

\begin{corollary}[Length generalization; proof in \autoref{pf_length_transfer}]
\label{cor:length_transfer}
Under the assumptions of \autoref{thm:fresh_context}, for sufficiently large $n$, with probability at least $1-\eta$ over the initialization and pretraining tasks, for every $n'\ge n$,
\begin{gather*}
 R_{\mathrm G,n'}(\widehat\mu_{\mathrm{trained}})=\cO\rbr{\frac1{n'}+\frac{\log n'}{n^2}},\\
 R_{\mathrm A,n'}(\widehat\mu_{\mathrm{trained}})=
 \begin{cases}
 \cO(n^{-1+\epsilon}/n'+n^{-2}),&\text{Uniform},\\
 \cO(\sigma_n^2/n'+\log n'/n^2),&\text{Gaussian-mixture}.
 \end{cases}
\end{gather*}
\end{corollary}

\subsection{Pretraining dynamics}
\label{sec:dynamics}
\label{sec:theory}
We establish the gate bound of \autoref{thm:softmax_dynamics} and the expert-parameter bounds of \autoref{convergence_of_first_layer}: first train the gating parameters with fixed experts, then specialize the expert parameters with fixed gates.
\autoref{sec:statistics} supplies the resulting risk guarantees.

We specialize to two experts and two families.
Define the population attention responses and feature direction by
\[
m_j(a)=\frac{\bbE_j\sbr{e^{a(X-\mu)}(X-\mu)}}{\bbE_j\sbr{e^{a(X-\mu)}}},
\qquad
\bar{\boldsymbol\phi}_j=\rbr{m_j(\omega_1),\ldots,m_j(\omega_d)}^\top,
\qquad
\bar\bz_j=\frac{P_d\bar{\boldsymbol\phi}_j}{\|P_d\bar{\boldsymbol\phi}_j\|_2}.
\]
Let $\bar w_{kj}$ denote $w_k$ with $\bz(\bX)$ replaced by $\bar\bz_j$, and $\bar w_{\cdot j}=(\bar w_{1j},\bar w_{2j})^\top$.
Throughout, bars indicate evaluation at population moments rather than expectation of the corresponding nonlinear statistic.
In particular, $\bar\bz_j\ne\bbE_j[\bz(\bX)]$ in general. %
In the proofs, $\Delta$ denotes deviation from the corresponding population-moment value.
In particular, under family $j$, $\Delta\bz=\bz-\bar\bz_j$.

\subsection{Stage 1: gating dynamics}
\begin{lemma}[Gating dynamics; proof in \autoref{pf_glu_gating} for the GLU]
\label{thm:softmax_dynamics}
\label{thm:glu_dynamics}
Under the assumptions of \autoref{thm:fresh_context}, run the first stage of \autoref{training_procedure}.
For sufficiently large $n$, with probability at least $1-\eta$ over the initialization and training tasks, the first stage ends at a finite time $\ttheta$ with
$$
\max\{\|\bar w_{\cdot\mathrm G}(\ttheta)-e_1\|_1,\|\bar w_{\cdot\mathrm A}(\ttheta)-e_2\|_1\}=\cO(n^{-1}),
\qquad
\max_q\|\theta_q(\ttheta)\|_2=\cO(\log n).
$$
\end{lemma}
\paragraph{Proof sketch}
The population features $\bar\bz_j$ distinguish the families.
Let $\bSigma^{(j)}$ be the asymptotic covariance matrix of $g$ given data from family $j$, that is,
$$
 \bbE_j[(g_i-\mu)(g_k-\mu)]=\frac{\Sigma_{ik}^{(j)}}n+o(n^{-1}).
$$
We show that, for the families we study, there must exist an expert $i$ such that for all $k \neq i$ we have almost surely
\begin{equation}
\Sigma_{ii}^{(j)} < \Sigma_{ik}^{(j)} < \Sigma_{kk}^{(j)}.
\label{eq:hierarchy}
\end{equation}
This implies that the optimal gating must exclusively choose the ``best'' expert.
Finally, we observe that the main component of the loss can be identified as
\begin{equation}
\cL^{\mathrm{truncated}}(\theta)
=\frac14\sum_j\left[
\sigma_\mu^2\left(\sum_{k=1}^2\bar w_{kj}-1\right)^2
+\frac1n\bar w_{\cdot j}^{\top}\bSigma^{(j)}\bar w_{\cdot j}
\right].
\label{eq:shared_reduced_loss}
\end{equation}
For softmax, the normalization term vanishes as $\sum_{k=1}^2\bar w_{kj} = 1$.

\paragraph{Softmax MoE}
First, note that this problem is highly nonconvex due to $w$ being generated by a softmax. %
Since $\sum_iw_i=1$ and the routing feature concentrates at $\bar\bz_j$, the loss expansion gives $\cL(\theta)=\cL^{\mathrm{sm}}(\theta)+o\rbr{1/n}$, where
\begin{equation}
\cL^{\mathrm{sm}}(\theta) = \frac1{4n}\sum_j\bar w_{\cdot j}^{\top}\bSigma^{(j)}\bar w_{\cdot j}.
\label{eq:softmax_truncated_loss}
\end{equation}
We can analyze the gradient flow on $\cL^{\mathrm{sm}}$ and show global convergence.
Correcting for the $o(1/n)$ remainder finishes the proof.

\paragraph{Gated linear unit}
For the GLU, unbiasedness is not automatic even when the gates are evaluated at $\bar\bz_j$.
Our approach will be to analyze the gradient flow on $\cL^{\mathrm{truncated}}$ and correct for the perturbation $\cL(\theta)-\cL^{\mathrm{truncated}}(\theta)$.
The loss $\cL^{\mathrm{truncated}}(\theta)$ has two distinct scales, where one scales with $\sigma^2_\mu$ and one scales with $1/n$.
The term $\sigma_\mu^2\left(\sum_{k=1}^2\bar w_{kj}-1\right)^2$ corresponds (roughly) to the bias of the predictor.
In the truncated loss, there exists a whole manifold of unbiased predictors, and we show that the gradient flow finds a solution near this manifold, while optimizing for the variance of the predictor.
Choosing $n$ sufficiently large lets us attack this problem using Tikhonov's theorem and determine where on the manifold the gradient flow converges.
Notably, $\cL(\theta)-\cL^{\mathrm{truncated}}(\theta)$ comes from the variance of $w$ and thus scales with $\sigma^2_\mu$.
Hence, the proof depends on $\sigma^2_\mu$ being below some threshold.
This is natural: a larger prior variance $\sigma^2_\mu$ increases the cost of normalization error, amplifying inaccuracies in the gating and potentially leading to a collapse of the gating layer towards a single expert for some initializations.

\subsection{Stage 2: expert dynamics}
For $t \geq t_\theta$, $\theta$ is frozen.
In this second stage, we analyze the dynamics of $\lambda$.
We establish that the expert parameter assigned to samples from the Gaussian family moves toward zero, while uniform expert parameters grow toward endpoint estimation, and mixture expert parameters approach a certain limit.
In the respective regimes, the leading expert-parameter flows satisfy $\dot\lambda_1\asymp-\lambda_1^3/n$ and $\dot\lambda_2\asymp1/(n\lambda_2^2)$ assuming without loss of generality $\lambda_2 > 0$.
Controlling the residuals yields the following result.
\begin{lemma}[Expert dynamics; proof in \autoref{pf_convergence_of_first_layer}]
\label{convergence_of_first_layer}
Under the assumptions of \autoref{thm:fresh_context}, assume the first stage of \autoref{training_procedure} ends at a finite time $\ttheta$ satisfying
$$
\max\{\|\bar w_{\cdot\mathrm G}-e_1\|_1,\|\bar w_{\cdot\mathrm A}-e_2\|_1\}\le C/n,
\qquad
\max_q\|\theta_q\|_2\le C\log n,
$$
and $|\lambda_1(\ttheta)|\le1$, $C^{-1}\le|\lambda_2(\ttheta)|\le1$, for a constant $C<\infty$ independent of $n$.
Use the training procedure in \autoref{training_procedure}.
For sufficiently large $n$, with probability at least $1-\eta$, uniformly over the possible $(\theta,\lambda)$ at $\ttheta$ satisfying these bounds, $|\lambda_1|\le2$ until $|\lambda_2|$ first reaches the value below, and this first time $\tlambda$ satisfies
$$
\begin{cases}
|\lambda_2(\tlambda)|=n^{1-\epsilon},\qquad \tlambda-\ttheta\le C n^{4-3\epsilon},&\text{Uniform},\\
|\lambda_2(\tlambda)|=\dfrac{\log(1/\sigma_n^2)}{4\sqrt{1-\sigma_n^2}},\qquad \tlambda-\ttheta\le C n/\sigma_n^2,&\text{Gaussian mixture}.
\end{cases}
$$
Moreover, $|\lambda_1(\tlambda)|=o(1)$.
The tasks used for training $\lambda$ may also have been used for training $\theta$.
\end{lemma}

\subsection{Efficiency of the final estimator}
For every expert $i$, the error of the final estimator decomposes as
\begin{equation}
\widehat\mu-\mu
=\underbrace{g_i-\mu}_{\text{expert error}}
+\underbrace{\sum_{k\ne i}w_k(g_k-g_i)}_{\text{routing error}}
+\underbrace{\rbr{\sum_k w_k-1}g_i}_{\text{normalization error}}.
\label{eq:decomposition}
\end{equation}
The normalization error vanishes for the softmax MoE, since $\sum_k w_k=1$.
\autoref{thm:accuracy_final_estimator} bounds the routing and normalization errors in $L^2(\bbP_j)$ through $\|\bar w_{\cdot j}-e_i\|_1$.

\label{sec:statistics}
\begin{lemma}[Accuracy of the final estimator; proof in \autoref{app:accuracy}]
\label{thm:accuracy_final_estimator}
Under the data model of \autoref{adaptiveMeanEstimation}, fix $j\in\{\mathrm G,\mathrm A\}$.
Condition on the trained $(\theta,\lambda)$ and evaluate on an independent pair $(\mu,\bX)$, with
$\max_k\|\theta_k\|_2\le C\log n$.
Then, for sufficiently large $n$, every expert $i\in[2]$, and arbitrary $\lambda$,
\begin{align*}
\sqrt{\bbE_j\sbr{(\widehat\mu(\bX)-\mu)^2}}
&\le \sqrt{\bbE_j\sbr{(g_i(\bX)-\mu)^2}}
+C\|\bar w_{\cdot j}-e_i\|_1
\left(\sigma_\mu^2+
\sqrt{\bbE_j\sbr{\max_{\ell\in[n]}|X_\ell-\mu|^4}}\right)^{1/2}.
\end{align*}
\end{lemma}

\paragraph{Proof sketch of \autoref{thm:fresh_context}.}
Under the conclusions of \autoref{convergence_of_first_layer}, the moment estimates in \autoref{expert_moments_gradients}, \autoref{uniform_expert_moments_gradients}, and \autoref{mixture_expert_moments_gradients} give expert risks $(1+o(1))/n$, $\cO(n^{-2+\epsilon})$, and $\cO(\sigma_n^2/n)$, respectively.
Applying \autoref{thm:accuracy_final_estimator} and its proof with $i=1$ for $j=\mathrm G$ and $i=2$ for $j=\mathrm A$, the gate bound of \autoref{thm:softmax_dynamics}, and the sample-maximum moment bounds gives
\begin{equation}
\|\widehat\mu-g_i\|_{L^2(\bbP_j)}
=\cO\rbr{\frac{\sqrt{\sigma_\mu^2+\log n}}{n}}.
\label{eq:discrepancy}
\end{equation}
The triangle inequality then gives the stated risk bounds and preserves the Gaussian leading constant.
The fixed-location claims follow from the translation identities and normalization estimate in \autoref{pf_specialization_final_step}.

\section{Experiments}
\label{sec:experiments}

Figures~\ref{fig:joint_softmax} and~\ref{fig:joint} show joint training of the softmax MoE and GLU with the uniform alternative.
We use $m=8000$, $n=150$, $K=2$, $d=10$, and $\sigma_\mu^2=0.5$, with $600{,}000$ gradient-descent steps and constant step sizes $20$ for $\theta$ and $100$ for $\lambda$.
Both learn family-specific selection as $|\lambda_1|$ decreases toward zero and $|\lambda_2|$ grows, with greater variation in the onset of specialization across GLU runs.
Gaussian error stays near the sample-mean error, while uniform error falls below it toward the sample mid-range error.
Training details, Gaussian-mixture results, and stagewise experiments are in \autoref{app:experiments}.

\begin{figure}[htbp]
    \centering
    \includegraphics[width=\linewidth]{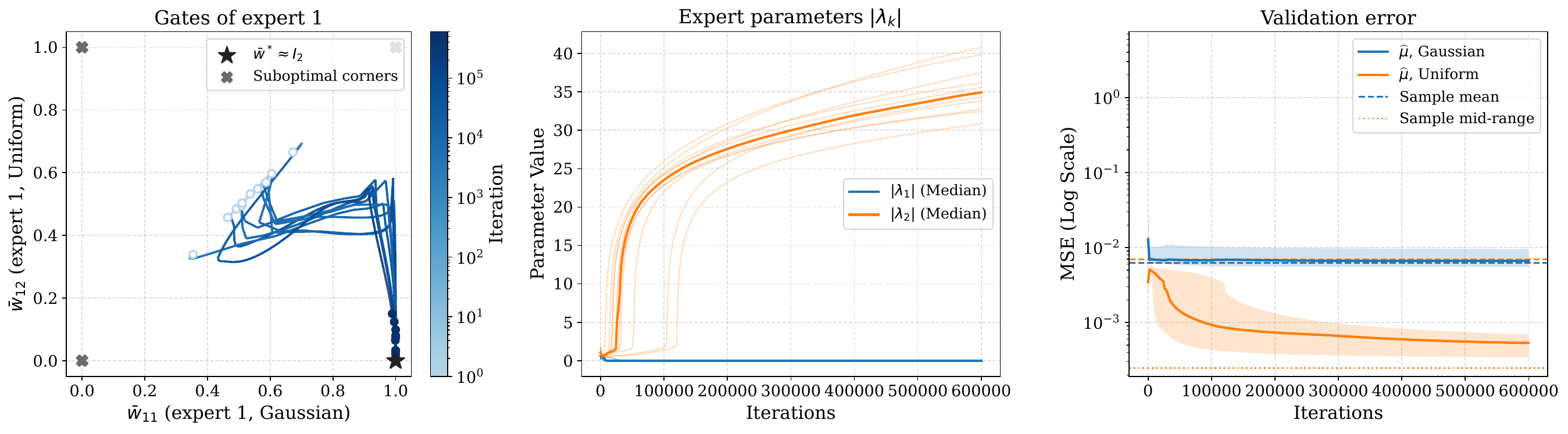}
    \caption{Softmax MoE, joint training with the uniform alternative.
Left: $(\bar w_{1\mathrm G},\bar w_{1\mathrm A})$ across runs, colored by iteration.
Middle: median $|\lambda_1|$ and $|\lambda_2|$, where thin lines are individual runs.
Right: median validation error of $\widehat\mu$ on Gaussian and uniform contexts, with the sample mean and sample mid-range as references.
Shading spans all ten runs.}
    \label{fig:joint_softmax}
\end{figure}

\begin{figure}[htbp]
    \centering
    \includegraphics[width=\linewidth]{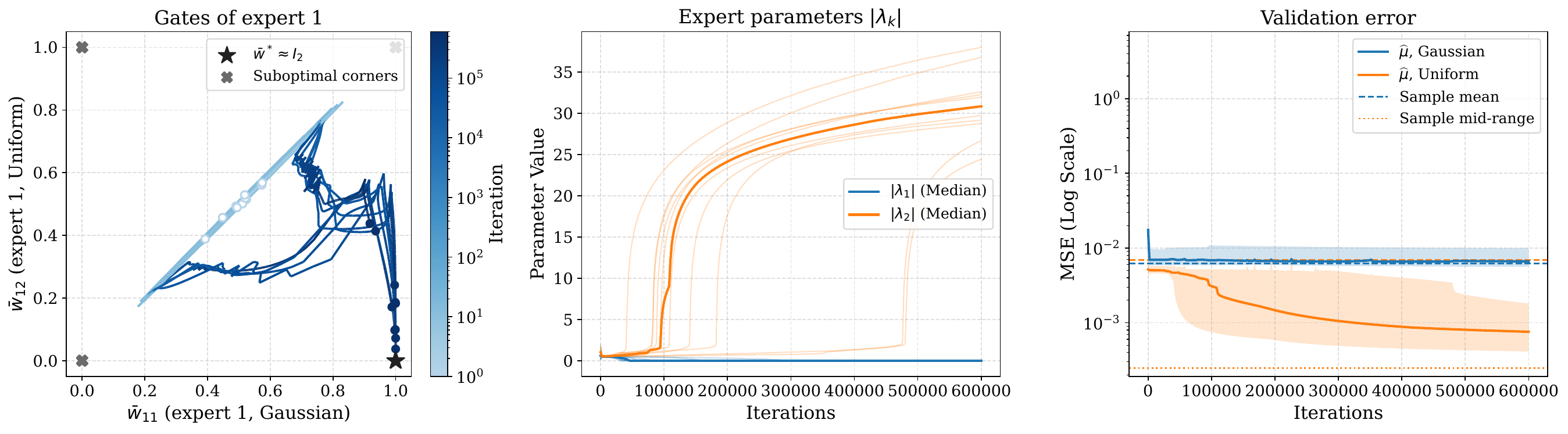}
    \caption{GLU, joint training with the uniform alternative, with panels and conventions as in Figure~\ref{fig:joint_softmax}.}
    \label{fig:joint}
\end{figure}

\section{Related work}
\label{sec:related}
\paragraph{In-context learning of statistical procedures.}
Prior-data fitted networks and tabular foundation models learn prediction procedures by pretraining on synthetic tasks~\citep{muller2024transformersbayesianinference,hollmann2025accuratepredictions,erickson2025tabarena}, and related pretrained models estimate treatment effects, survival curves, mutual information, and Poisson means~\citep{balazadeh2025causalpfn,qi2026survivalpfn,peyrard2025metastatisticallearningsupervisedlearning,teh2025empiricalbayes}.
Their focus is prediction or posterior approximation, not how pretraining attains estimation rates that depend on an unknown family.

\paragraph{Statistical theory of pretrained predictors.}
Transformers can implement in-context algorithm selection~\citep{bai2023transformersstatisticiansprovableincontext}, attain minimax rates over a known function class~\citep{kim2024transformersminimaxoptimalnonparametric}, and approximate Bayesian posterior predictors whose rates adapt to unknown structure~\citep{wakayama2025bayesian,ma2025adaptivity}; idealized pretrained Bayes estimators are similarly adaptive in empirical Bayes problems~\citep{cannella2026universalpriorssolvingempirical}.
These guarantees concern the empirical risk minimizer or an idealized predictor; we follow the training trajectory and identify the internal parameter that carries the adaptation.
\paragraph{Training dynamics of attention.}
Gradient-based analyses of attention for in-context learning~\citep{zhang2023trainedtransformerslearnlinear,huang2023incontextconvergencetransformers,kim2024transformerslearnnonlinearfeatures,oko2024pretrainedtransformer} include task allocation across softmax heads~\citep{chen2024trainingdynamicsmultiheadsoftmax}, where the task handled by each head is fixed by the task structure.
Here the family changes across contexts, so specialization must be paired with learned routing.

\paragraph{Adaptive estimation and architecture.}
Adaptivity is a classical objective in statistics~\citep{lepski1991adaptive,donoho1995adapting}, including adaptive estimation of a symmetric location~\citep{stone1975adaptive,bickel1993efficient}; the rate gap in our problem stems from the irregularity of the uniform family~\citep{ibragimov1981statistical}.
Our model is a small Set Transformer~\citep{lee2019settransformerframeworkattentionbased} whose second layer is a softmax mixture of experts~\citep{jacobs1991adaptive} or a GLU~\citep{shazeer2020glu}.
\autoref{app:related} discusses these connections in detail.

\section{Discussion}
Our current analysis is limited by making some simplifying assumptions on the training procedure as well as the architecture.
Without fixing parameters, we observe that the problem becomes much easier to learn in practice, but admits many possible optimal solutions making it much harder to study.
We also observe that separating training into two stages slows down training while making the problem much easier to study.
Future work should focus on addressing these issues.
Additionally, we believe transformers are able to do parameter estimation for much larger families of distributions.
Extending our analysis towards these would also be a natural next step.
Finally, we observe that there is an upper bound to the prior variance $\sigma_\mu^2$, which seems to be unavoidable in a simple architecture. 
Still, finding an architecture that is able to handle arbitrary $\sigma_\mu^2$ would be a clear next step.

\section{Conclusion}
We show that stagewise gradient flow on the empirical loss learns an adaptive location estimator without observing family labels.
The estimator achieves Gaussian asymptotic efficiency, nearly optimal uniform rates, and order-optimal rates for Gaussian mixtures with shrinking component variance.
Attention provides both distribution-sensitive features and candidate estimators, while the gating layer learns which expert to select.
Softmax enforces normalization and exact translation equivariance, whereas the GLU learns approximate normalization before selecting experts.

\newpage
\subsection*{Funding Statement}
Y.S. Tan was supported by NUS Startup Grant A-8000448-00-00 and the Singapore Ministry of Education (MOE) AcRF Tier 1 Grants A-8002498-00-00 and A-8004458-00-00.
Klusowski gratefully acknowledges support from the National Science Foundation through NSF CAREER award DMS-2239448 and from the Alfred P. Sloan Foundation through a Sloan Research Fellowship.
Subhro Ghosh was supported in part by the NUS Dean’s Chair Associate Professorship E-146-00-0037-01 and the
Singapore MOE grants A-8002014-00-00 and A-8003802-00-00.

\subsection*{AI Use Statement}
In this work, we used generative AI tools to assist in the writing of proofs, formulate mathematical claims and implement the experiment code.
Additionally, we used generative AI tools to polish the writing, find typesetting and grammar errors, suggest the structure of the main paper, and identify and summarize related work.
We have reviewed all AI-assisted work.
All proofs were verified by the authors.
The experiment code was reviewed and tested by the authors.
We take responsibility for the final content of this work, including text, claims or artifacts produced with the aid of generative AI.

\bibliography{references}

\begin{thebibliography}{48}
\providecommand{\natexlab}[1]{#1}
\providecommand{\url}[1]{\texttt{#1}}
\expandafter\ifx\csname urlstyle\endcsname\relax
  \providecommand{\doi}[1]{doi: #1}\else
  \providecommand{\doi}{doi: \begingroup \urlstyle{rm}\Url}\fi

\bibitem[Akyürek et~al.(2023)Akyürek, Schuurmans, Andreas, Ma, and
  Zhou]{akyurek2023learningalgorithmincontextlearning}
Ekin Akyürek, Dale Schuurmans, Jacob Andreas, Tengyu Ma, and Denny Zhou.
\newblock What learning algorithm is in-context learning? investigations with
  linear models, 2023.
\newblock URL \url{https://arxiv.org/abs/2211.15661}.

\bibitem[Bai et~al.(2023)Bai, Chen, Wang, Xiong, and
  Mei]{bai2023transformersstatisticiansprovableincontext}
Yu~Bai, Fan Chen, Huan Wang, Caiming Xiong, and Song Mei.
\newblock Transformers as statisticians: Provable in-context learning with
  in-context algorithm selection.
\newblock \emph{Advances in Neural Information Processing Systems}, 2023.

\bibitem[Balazadeh et~al.(2025)Balazadeh, Kamkari, Thomas, Li, Ma, Cresswell,
  and Krishnan]{balazadeh2025causalpfn}
Vahid Balazadeh, Hamidreza Kamkari, Valentin Thomas, Benson Li, Junwei Ma,
  Jesse~C. Cresswell, and Rahul~G. Krishnan.
\newblock {CausalPFN}: Amortized causal effect estimation via in-context
  learning, 2025.
\newblock URL \url{https://arxiv.org/abs/2506.07918}.

\bibitem[Beran(1974)]{beran1974asymptotically}
Rudolf Beran.
\newblock Asymptotically efficient adaptive rank estimates in location models.
\newblock \emph{The Annals of Statistics}, 2\penalty0 (1):\penalty0 63--74,
  1974.
\newblock \doi{10.1214/aos/1176342613}.

\bibitem[Berglund(2001)]{berglund2001perturbationtheorydynamicalsystems}
Nils Berglund.
\newblock Perturbation theory of dynamical systems, 2001.
\newblock URL \url{https://arxiv.org/abs/math/0111178}.

\bibitem[Bickel(1982)]{bickel1982adaptive}
P.~J. Bickel.
\newblock {On Adaptive Estimation}.
\newblock \emph{The Annals of Statistics}, 10\penalty0 (3):\penalty0 647 --
  671, 1982.
\newblock \doi{10.1214/aos/1176345863}.
\newblock URL \url{https://doi.org/10.1214/aos/1176345863}.

\bibitem[Bickel et~al.(1993)Bickel, Klaassen, Ritov, and
  Wellner]{bickel1993efficient}
Peter~J. Bickel, Chris A.~J. Klaassen, Ya'acov Ritov, and Jon~A. Wellner.
\newblock \emph{Efficient and Adaptive Estimation for Semiparametric Models}.
\newblock Johns Hopkins University Press, 1993.

\bibitem[Brown et~al.(2020)Brown, Mann, Ryder, Subbiah, Kaplan, Dhariwal,
  Neelakantan, Shyam, Sastry, Askell, Agarwal, Herbert-Voss, Krueger, Henighan,
  Child, Ramesh, Ziegler, Wu, Winter, Hesse, Chen, Sigler, Litwin, Gray, Chess,
  Clark, Berner, McCandlish, Radford, Sutskever, and
  Amodei]{brown2020languagemodelsfewshotlearners}
Tom~B. Brown, Benjamin Mann, Nick Ryder, Melanie Subbiah, Jared Kaplan,
  Prafulla Dhariwal, Arvind Neelakantan, Pranav Shyam, Girish Sastry, Amanda
  Askell, Sandhini Agarwal, Ariel Herbert-Voss, Gretchen Krueger, Tom Henighan,
  Rewon Child, Aditya Ramesh, Daniel~M. Ziegler, Jeffrey Wu, Clemens Winter,
  Christopher Hesse, Mark Chen, Eric Sigler, Mateusz Litwin, Scott Gray,
  Benjamin Chess, Jack Clark, Christopher Berner, Sam McCandlish, Alec Radford,
  Ilya Sutskever, and Dario Amodei.
\newblock Language models are few-shot learners, 2020.
\newblock URL \url{https://arxiv.org/abs/2005.14165}.

\bibitem[Bu et~al.(2025)Bu, Huang, Han, Nitanda, Zhang, Wong, and
  Suzuki]{bu2025provableincontextvectorarithmetic}
Dake Bu, Wei Huang, Andi Han, Atsushi Nitanda, Qingfu Zhang, Hau-San Wong, and
  Taiji Suzuki.
\newblock Provable in-context vector arithmetic via retrieving task concepts,
  2025.
\newblock URL \url{https://arxiv.org/abs/2508.09820}.

\bibitem[Cannella et~al.(2026)Cannella, Teh, Han, and
  Polyanskiy]{cannella2026universalpriorssolvingempirical}
Nick Cannella, Anzo Teh, Yanjun Han, and Yury Polyanskiy.
\newblock Universal priors: solving empirical {Bayes} via {Bayesian} inference
  and pretraining, 2026.
\newblock URL \url{https://arxiv.org/abs/2602.15136}.

\bibitem[Chen et~al.(2024{\natexlab{a}})Chen, Sheen, Wang, and
  Yang]{chen2024trainingdynamicsmultiheadsoftmax}
Siyu Chen, Heejune Sheen, Tianhao Wang, and Zhuoran Yang.
\newblock Training dynamics of multi-head softmax attention for in-context
  learning: Emergence, convergence, and optimality.
\newblock \emph{arXiv preprint arXiv:2402.19442}, 2024{\natexlab{a}}.

\bibitem[Chen et~al.(2024{\natexlab{b}})Chen, Zhao, and
  Zou]{chen2024transformersutilizemultiheadattention}
Xingwu Chen, Lei Zhao, and Difan Zou.
\newblock How transformers utilize multi-head attention in in-context learning?
  a case study on sparse linear regression.
\newblock \emph{arXiv preprint arXiv:2408.04532}, 2024{\natexlab{b}}.

\bibitem[Cranmer et~al.(2020)Cranmer, Brehmer, and Louppe]{Cranmer_2020}
Kyle Cranmer, Johann Brehmer, and Gilles Louppe.
\newblock The frontier of simulation-based inference.
\newblock \emph{Proceedings of the National Academy of Sciences}, 117\penalty0
  (48):\penalty0 30055--30062, 2020.
\newblock ISSN 1091-6490.
\newblock \doi{10.1073/pnas.1912789117}.
\newblock URL \url{http://dx.doi.org/10.1073/pnas.1912789117}.

\bibitem[Dauphin et~al.(2017)Dauphin, Fan, Auli, and
  Grangier]{dauphin2017languagemodelinggatedconvolutional}
Yann~N. Dauphin, Angela Fan, Michael Auli, and David Grangier.
\newblock Language modeling with gated convolutional networks, 2017.
\newblock URL \url{https://arxiv.org/abs/1612.08083}.

\bibitem[Donoho \& Johnstone(1995)Donoho and Johnstone]{donoho1995adapting}
David~L. Donoho and Iain~M. Johnstone.
\newblock Adapting to unknown smoothness via wavelet shrinkage.
\newblock \emph{Journal of the American Statistical Association}, 90\penalty0
  (432):\penalty0 1200--1224, 1995.
\newblock \doi{10.1080/01621459.1995.10476626}.

\bibitem[Erickson et~al.(2025)Erickson, Purucker, Tschalzev, Holzm{\"u}ller,
  Mutalik~Desai, Salinas, and Hutter]{erickson2025tabarena}
Nick Erickson, Lennart Purucker, Andrej Tschalzev, David Holzm{\"u}ller,
  Prateek Mutalik~Desai, David Salinas, and Frank Hutter.
\newblock {TabArena}: A living benchmark for machine learning on tabular data.
\newblock In \emph{Advances in Neural Information Processing Systems}, 2025.
\newblock URL \url{https://arxiv.org/abs/2506.16791}.

\bibitem[Garg et~al.(2023)Garg, Tsipras, Liang, and
  Valiant]{garg2023transformerslearnincontextcase}
Shivam Garg, Dimitris Tsipras, Percy Liang, and Gregory Valiant.
\newblock What can transformers learn in-context? a case study of simple
  function classes, 2023.
\newblock URL \url{https://arxiv.org/abs/2208.01066}.

\bibitem[Hollmann et~al.(2023)Hollmann, M{\"u}ller, Eggensperger, and
  Hutter]{hollmann2023tabpfntransformersolvessmall}
Noah Hollmann, Samuel M{\"u}ller, Katharina Eggensperger, and Frank Hutter.
\newblock {TabPFN}: A transformer that solves small tabular classification
  problems in a second.
\newblock \emph{International Conference on Learning Representations}, 2023.

\bibitem[Hollmann et~al.(2025)Hollmann, M{\"u}ller, Purucker, Krishnakumar,
  K{\"o}rfer, Hoo, Schirrmeister, and Hutter]{hollmann2025accuratepredictions}
Noah Hollmann, Samuel M{\"u}ller, Lennart Purucker, Arjun Krishnakumar, Max
  K{\"o}rfer, Shi~Bin Hoo, Robin~Tibor Schirrmeister, and Frank Hutter.
\newblock Accurate predictions on small data with a tabular foundation model.
\newblock \emph{Nature}, 637:\penalty0 319--326, 2025.
\newblock \doi{10.1038/s41586-024-08328-6}.

\bibitem[Huang et~al.(2023)Huang, Cheng, and
  Liang]{huang2023incontextconvergencetransformers}
Yu~Huang, Yuan Cheng, and Yingbin Liang.
\newblock In-context convergence of transformers, 2023.
\newblock URL \url{https://arxiv.org/abs/2310.05249}.

\bibitem[Ibragimov \& Has'minskii(1981)Ibragimov and
  Has'minskii]{ibragimov1981statistical}
I.~A. Ibragimov and R.~Z. Has'minskii.
\newblock \emph{Statistical Estimation: Asymptotic Theory}.
\newblock Springer New York, New York, NY, 1981.
\newblock ISBN 978-1-4899-0027-2.
\newblock \doi{10.1007/978-1-4899-0027-2}.

\bibitem[Ignatiadis \& Kankanala(2026)Ignatiadis and
  Kankanala]{ignatiadis2026compounddecisionsempiricalbayes}
Nikolaos Ignatiadis and Sid Kankanala.
\newblock Compound decisions and empirical {Bayes} via {Bayesian}
  nonparametrics, 2026.
\newblock URL \url{https://arxiv.org/abs/2602.20115}.

\bibitem[Jacobs et~al.(1991)Jacobs, Jordan, Nowlan, and
  Hinton]{jacobs1991adaptive}
Robert~A. Jacobs, Michael~I. Jordan, Steven~J. Nowlan, and Geoffrey~E. Hinton.
\newblock Adaptive mixtures of local experts.
\newblock \emph{Neural Computation}, 3\penalty0 (1):\penalty0 79--87, 1991.
\newblock \doi{10.1162/neco.1991.3.1.79}.

\bibitem[Kim \& Suzuki(2024)Kim and
  Suzuki]{kim2024transformerslearnnonlinearfeatures}
Juno Kim and Taiji Suzuki.
\newblock Transformers learn nonlinear features in context: Nonconvex
  mean-field dynamics on the attention landscape, 2024.
\newblock URL \url{https://arxiv.org/abs/2402.01258}.

\bibitem[Kim et~al.(2024)Kim, Nakamaki, and
  Suzuki]{kim2024transformersminimaxoptimalnonparametric}
Juno Kim, Tai Nakamaki, and Taiji Suzuki.
\newblock Transformers are minimax optimal nonparametric in-context learners,
  2024.
\newblock URL \url{https://arxiv.org/abs/2408.12186}.

\bibitem[Lee et~al.(2019)Lee, Lee, Kim, Kosiorek, Choi, and
  Teh]{lee2019settransformerframeworkattentionbased}
Juho Lee, Yoonho Lee, Jungtaek Kim, Adam~R. Kosiorek, Seungjin Choi, and
  Yee~Whye Teh.
\newblock Set transformer: A framework for attention-based
  permutation-invariant neural networks.
\newblock In \emph{Proceedings of the 36th International Conference on Machine
  Learning}, 2019.

\bibitem[Lehmann \& Casella(1998)Lehmann and Casella]{lehmann1998theory}
E.~L. Lehmann and George Casella.
\newblock \emph{Theory of Point Estimation}.
\newblock Springer Texts in Statistics. Springer New York, NY, 2 edition, 1998.
\newblock ISBN 978-0-387-98502-2.
\newblock \doi{10.1007/b98854}.

\bibitem[Lepskii(1991)]{lepski1991adaptive}
Oleg~V. Lepskii.
\newblock On a problem of adaptive estimation in {Gaussian} white noise.
\newblock \emph{Theory of Probability and Its Applications}, 35\penalty0
  (3):\penalty0 454--466, 1991.
\newblock \doi{10.1137/1135065}.

\bibitem[Li et~al.(2024)Li, Cao, Gao, He, Liu, Klusowski, Fan, and
  Wang]{li2024onelayertransformerprovablylearns}
Zihao Li, Yuan Cao, Cheng Gao, Yihan He, Han Liu, Jason~M. Klusowski, Jianqing
  Fan, and Mengdi Wang.
\newblock One-layer transformer provably learns one-nearest neighbor in
  context.
\newblock \emph{arXiv preprint arXiv:2411.10830}, 2024.

\bibitem[Ma et~al.(2025)Ma, Wang, and Samworth]{ma2025adaptivity}
Tianyi Ma, Tengyao Wang, and Richard~J. Samworth.
\newblock Optimal in-context adaptivity and distributional robustness of
  transformers, 2025.
\newblock URL \url{https://arxiv.org/abs/2510.23254}.

\bibitem[Müller et~al.(2024)Müller, Hollmann, Arango, Grabocka, and
  Hutter]{muller2024transformersbayesianinference}
Samuel Müller, Noah Hollmann, Sebastian~Pineda Arango, Josif Grabocka, and
  Frank Hutter.
\newblock Transformers can do {Bayesian} inference, 2024.
\newblock URL \url{https://arxiv.org/abs/2112.10510}.

\bibitem[Nagler(2023)]{nagler2023statisticalfoundationspriordatafitted}
Thomas Nagler.
\newblock Statistical foundations of prior-data fitted networks, 2023.
\newblock URL \url{https://arxiv.org/abs/2305.11097}.

\bibitem[Oko et~al.(2024)Oko, Song, Suzuki, and
  Wu]{oko2024pretrainedtransformer}
Kazusato Oko, Yujin Song, Taiji Suzuki, and Denny Wu.
\newblock Pretrained transformer efficiently learns low-dimensional target
  functions in-context.
\newblock In \emph{Advances in Neural Information Processing Systems}, 2024.
\newblock URL \url{https://arxiv.org/abs/2411.02544}.

\bibitem[Peyrard \& Cho(2025)Peyrard and
  Cho]{peyrard2025metastatisticallearningsupervisedlearning}
Maxime Peyrard and Kyunghyun Cho.
\newblock Meta-statistical learning: Supervised learning of statistical
  estimators.
\newblock \emph{arXiv preprint arXiv:2502.12088}, 2025.

\bibitem[Qi et~al.(2026)Qi, Balazadeh, Cooper, Greiner, and
  Krishnan]{qi2026survivalpfn}
Shi-ang Qi, Vahid Balazadeh, Michael Cooper, Russell Greiner, and Rahul~G
  Krishnan.
\newblock Survival{PFN}: Amortizing survival prediction via in-context
  {Bayesian} inference.
\newblock In \emph{2nd {ICML} Workshop on Foundation Models for Structured
  Data}, 2026.
\newblock URL \url{https://openreview.net/forum?id=PDik7bpFhE}.

\bibitem[Shao(2003)]{shao2003mathematical}
J.~Shao.
\newblock \emph{Mathematical Statistics}.
\newblock Springer Texts in Statistics. Springer, 2003.
\newblock ISBN 9780387953823.

\bibitem[Shazeer(2020)]{shazeer2020glu}
Noam Shazeer.
\newblock {GLU} variants improve transformer.
\newblock \emph{arXiv preprint arXiv:2002.05202}, 2020.

\bibitem[Shazeer et~al.(2017)Shazeer, Mirhoseini, Maziarz, Davis, Le, Hinton,
  and Dean]{shazeer2017outrageously}
Noam Shazeer, Azalia Mirhoseini, Krzysztof Maziarz, Andy Davis, Quoc Le,
  Geoffrey Hinton, and Jeff Dean.
\newblock Outrageously large neural networks: The sparsely-gated
  mixture-of-experts layer, 2017.
\newblock URL \url{https://arxiv.org/abs/1701.06538}.

\bibitem[Stein(1956)]{stein1956efficient}
Charles Stein.
\newblock Efficient nonparametric testing and estimation.
\newblock In \emph{Proceedings of the Third Berkeley Symposium on Mathematical
  Statistics and Probability, Volume 1: Contributions to the Theory of
  Statistics}, volume~1, pp.\  187--195. University of California Press, 1956.

\bibitem[Stone(1975)]{stone1975adaptive}
C.~J. Stone.
\newblock Adaptive maximum likelihood estimators of a location parameter.
\newblock \emph{Annals of Statistics}, 3\penalty0 (2):\penalty0 267--284, 1975.
\newblock \doi{10.1214/aos/1176343056}.

\bibitem[Teh et~al.(2025)Teh, Jabbour, and Polyanskiy]{teh2025empiricalbayes}
Anzo Teh, Mark Jabbour, and Yury Polyanskiy.
\newblock Solving empirical {Bayes} via transformers, 2025.
\newblock URL \url{https://arxiv.org/abs/2502.09844}.

\bibitem[Tsybakov(2009)]{tsybakov2009introduction}
Alexandre~B. Tsybakov.
\newblock \emph{Introduction to Nonparametric Estimation}.
\newblock Springer Series in Statistics. Springer New York, 2009.
\newblock \doi{10.1007/b13794}.

\bibitem[Wakayama \& Suzuki(2025)Wakayama and Suzuki]{wakayama2025bayesian}
Tomoya Wakayama and Taiji Suzuki.
\newblock In-context learning is provably {Bayesian} inference: A
  generalization theory for meta-learning, 2025.
\newblock URL \url{https://arxiv.org/abs/2510.10981}.

\bibitem[Ye et~al.(2025)Ye, Liu, Cai, Zhou, and Zhan]{ye2025closer}
Han-Jia Ye, Si-Yang Liu, Hao-Run Cai, Qi-Le Zhou, and De-Chuan Zhan.
\newblock A closer look at deep learning methods on tabular datasets, 2025.
\newblock URL \url{https://arxiv.org/abs/2407.00956}.

\bibitem[Zaheer et~al.(2017)Zaheer, Kottur, Ravanbakhsh, P{\'o}czos,
  Salakhutdinov, and Smola]{zaheer2018deepsets}
Manzil Zaheer, Satwik Kottur, Siamak Ravanbakhsh, Barnab{\'a}s P{\'o}czos,
  Ruslan Salakhutdinov, and Alexander~J. Smola.
\newblock Deep sets.
\newblock In \emph{Advances in Neural Information Processing Systems}, 2017.

\bibitem[Zhang et~al.(2025)Zhang, Tan, Tian, and Li]{zhang2025tabpfnonemodel}
Qiong Zhang, Yan~Shuo Tan, Qinglong Tian, and Pengfei Li.
\newblock {TabPFN}: One model to rule them all?
\newblock \emph{arXiv preprint arXiv:2505.20003}, 2025.
\newblock URL \url{https://arxiv.org/abs/2505.20003}.

\bibitem[Zhang et~al.(2024{\natexlab{a}})Zhang, Frei, and
  Bartlett]{zhang2023trainedtransformerslearnlinear}
Ruiqi Zhang, Spencer Frei, and Peter~L. Bartlett.
\newblock Trained transformers learn linear models in-context.
\newblock \emph{Journal of Machine Learning Research}, 25\penalty0
  (49):\penalty0 1--55, 2024{\natexlab{a}}.

\bibitem[Zhang et~al.(2024{\natexlab{b}})Zhang, Wu, and
  Bartlett]{zhang2024incontextlearninglineartransformer}
Ruiqi Zhang, Jingfeng Wu, and Peter~L. Bartlett.
\newblock In-context learning of a linear transformer block: Benefits of the
  {MLP} component and one-step {GD} initialization.
\newblock \emph{Advances in Neural Information Processing Systems},
  2024{\natexlab{b}}.

\end{thebibliography}
\bibliographystyle{iclr2027_conference}

\appendix

\section{Outline}
\label{app:assembly}
\paragraph{Proof map.}
\autoref{decomposition_of_gating_parameters} gives the reparameterization used in the dynamics, and \autoref{feature_geometry} establishes the geometry of $\bz$.
\autoref{app:mainproof} assembles the main theorem from the fresh-context risk bound in \autoref{app:accuracy}, the gating analysis in \autoref{app:routing}, and the expert-parameter analysis in \autoref{app:experts}.
\autoref{app:additional} collects auxiliary estimates, and \autoref{app:lowerbound} gives the Gaussian-mixture benchmark.
\autoref{app:experiments} reports additional experiments, and \autoref{role_of_layernorm} discusses the role of LayerNorm.
\autoref{app:related} gives an extended discussion of related work.
See \autoref{par:notation} for notation.

\section{Proof of \autoref{thm:fresh_context}}
\label{app:mainproof}

\subsection{Expert hierarchy}

\begin{definition}[Expert covariance and the better expert]
\label{betterExpert}
\label{expertHierarchy}
Let $\bSigma^{(j)}$ be the matrix with entries $\Sigma_{ik}^{(j)}=\Sigma_g(\lambda_i,\lambda_k)$ evaluated under the centered law $Q_j$.
The symmetrized covariance formula is given in \autoref{sym_expert_covariance}, with family-specific calculations in the expert-covariance appendix.
Expert $i\in\{1,2\}$ is \emph{better} for family $j\in\{\mathrm G,\mathrm A\}$ if, for the other expert $k\ne i$,
\begin{align}
    \Sigma_{ii}^{(j)} < \Sigma_{ik}^{(j)}
    < \Sigma_{kk}^{(j)}.
\end{align}
The experts satisfy the \emph{expert hierarchy} if each family has a better expert.
Its strict covariance gaps are the two positive differences in the ordering.
\end{definition}

\begin{lemma}[Gaussian distribution satisfies the expert hierarchy]
\label{expertHierarchy_normal}
Let $Y \sim \cN(0,1)$ and $\lambda_1, \lambda_2 \in \bbR$ with $0<|\lambda_1| < |\lambda_2|$.
Then
\begin{align}
    \Sigma_g(\lambda_1,\lambda_1) < \Sigma_g(\lambda_1,\lambda_2) < \Sigma_g(\lambda_2,\lambda_2).
\end{align}
\end{lemma}
\begin{proof}
    The statement follows from $\Sigma_g(\lambda_1,\lambda_2)=\cosh(\lambda_1\lambda_2)+\lambda_1\lambda_2\sinh(\lambda_1\lambda_2)$, which is strictly increasing in $|\lambda_1\lambda_2|$.
\end{proof}

\begin{lemma}[Gaussian mixture distribution satisfies the expert hierarchy]
\label{expertHierarchy_mixture}
For every fixed $0<\delta<1$ there exists $\sigma_\delta>0$ such that for all $0<\sigma_n<\sigma_\delta<1/\sqrt2$ and $0\le|\lambda_1|<|\lambda_2|\le L_{\max}=(1-\delta)\log(1/\sigma_n)$,
\begin{align}
    \Sigma_g(\lambda_1,\lambda_1) > \Sigma_g(\lambda_1,\lambda_2) > \Sigma_g(\lambda_2,\lambda_2).
\end{align}
\end{lemma}
\begin{proof}
See \autoref{gsr_hierarchy}.
\end{proof}

\begin{lemma}[Uniform distribution satisfies the expert hierarchy]
\label{expertHierarchy_uniform}
Let $Y \sim \mathrm{Unif}[-\sqrt3,\sqrt3]$ and $\lambda_1,\lambda_2 \in [-1,1]$ with $|\lambda_1| < |\lambda_2|$.
Then
\begin{align}
    \Sigma_g(\lambda_1,\lambda_1) > \Sigma_g(\lambda_1,\lambda_2) > \Sigma_g(\lambda_2,\lambda_2).
\end{align}
\end{lemma}
\begin{proof}
See \autoref{gsr_uniform_hierarchy_bound}.
\end{proof}
In the parameter ranges above, $0<|\lambda_1|<|\lambda_2|$ makes expert $1$ better on Gaussian contexts and expert $2$ better on alternative contexts.

\subsection{Gating-layer dynamics}
\begin{lemma}[Decomposition of the gating parameters]
\label{decomposition_of_gating_parameters}
For the GLU, let $Z_{jk}=\inner{\bar\bz_j,\bar\bz_k}$ and suppose $Z\succ0$.
Each gating parameter has the decomposition
$$
\theta_q=\theta_q^\parallel+\theta_q^\perp,
\qquad
\theta_q^\parallel
=\sum_{\ell,j}
(Z^{-1})_{\ell j}\sigma^{-1}(\bar w_{qj})\bar{\bz}_\ell,
\qquad
\theta_q^\perp\perp\operatorname{span}\{\bar{\bz}_j\}_{j\in\{\mathrm G,\mathrm A\}}.
$$
Thus $\theta$ is equivalently parameterized by $(\bar w,\theta^\perp)$.
Moreover,
$$
\|\theta_q\|_2
\le\|\theta_q^\perp\|_2+\sqrt{\frac{2}{\lambda_{\min}(Z)}}
\max_j\log\frac1{\bar w'_{qj}}.
$$
\end{lemma}
\begin{proof}
Taking the inner product of $\theta_q^\parallel$ with $\bar{\bz}_j$ gives $\sigma^{-1}(\bar w_{qj})=\inner{\theta_q,\bar{\bz}_j}$, so the stated decomposition is orthogonal.
Moreover,
$$
\|\theta_q^\parallel\|_2^2
=\sum_{i,j}
\sigma^{-1}(\bar w_{qi})(Z^{-1})_{ij}\sigma^{-1}(\bar w_{qj})
\le\frac1{\lambda_{\min}(Z)}
\sum_j|\sigma^{-1}(\bar w_{qj})|^2.
$$
The conclusion follows from the triangle inequality and $|\log(p/(1-p))|\le\log(1/(p(1-p)))$ for $0<p<1$.
\end{proof}
For the truncated flow, $\dot\theta_q$ lies in $\operatorname{span}\{\bar\bz_j\}_{j\in\{\mathrm G,\mathrm A\}}$, so $\theta_q^\perp(t)=\theta_q^\perp(0)$.
Under the assumptions of \autoref{population_gradient_perturbation}, the full flow also has $\|\theta_q^\perp(t)\|_2=\cO(1)$ for $\tcrit\le t\le \ttheta$.

\begin{definition}[Truncated loss]
\label{app:assembly_routing}
For fixed $\lambda$, define the bias loss
$$
\mathcal L^{\mathrm{bias}}(\theta)
=\frac{\sigma_\mu^2}{4}
\sum_j(\bar w_{1j}+\bar w_{2j}-1)^2
$$
and recall the truncated loss from \eqref{eq:shared_reduced_loss}
$$
\cL^{\mathrm{truncated}}(\theta)
=\frac1{4}\sum_j
\left[
\sigma_\mu^2\left(\sum_{k=1}^2\bar w_{kj}-1\right)^2
+\frac1n\bar w_{\cdot j}^{\top}\bSigma^{(j)}\bar w_{\cdot j}
\right].
$$
\end{definition}

\begin{lemma}[Stationary points of the truncated flow]
\label{list_of_stationary_points}
Fix $\lambda$ and $\sigma_\mu^2>0$, suppose the expert hierarchy of \autoref{expertHierarchy} holds, and let $k_j$ denote the better expert in family $j$.
Let $e_1=(1,0)^\top$, $e_2=(0,1)^\top$, and suppose $Z\succ0$.
The truncated flow is
\begin{align*}
\dot{\bar w}_{qk}
&=-\frac12\bar w'_{qk}\sum_jZ_{jk}\bar w'_{qj}S_{qj}(\bar w),\\
S_{qj}(\bar w)
&=\sigma_\mu^2\left(\sum_{i=1}^2\bar w_{ij}-1\right)
+\frac1n\sum_{i=1}^2\Sigma_{qi}^{(j)}\bar w_{ij}.
\end{align*}
Its stationary points form the set
$$
\cS=\left\{\bar w^\circ\in[0,1]^{2\times2}:
\bar w^\circ_{qj}(1-\bar w^\circ_{qj})S_{qj}(\bar w^\circ)=0
\text{ for all }q,j\right\},
$$
with
\begin{align*}
\bar w^\circ\in\cS
\quad\Longleftrightarrow\quad
&\bar w^\circ_{\cdot j}\in\Bigl\{
0,e_1,e_2,e_1+e_2,\\
&\qquad
\frac{n\sigma_\mu^2}{n\sigma_\mu^2+\Sigma_{11}^{(j)}}e_1,
\frac{n\sigma_\mu^2}{n\sigma_\mu^2+\Sigma_{22}^{(j)}}e_2
\Bigr\}
\quad(j\in\{\mathrm G,\mathrm A\}).
\end{align*}
The unique optimum is
\begin{align*}
\bar w_{\cdot j}^*
&=\frac{n\sigma_\mu^2}{n\sigma_\mu^2+\Sigma_{k_jk_j}^{(j)}}e_{k_j},\\
\operatorname*{argmin}_{\bar w\in[0,1]^{2\times2}}
\cL^{\mathrm{truncated}}(\bar w)&=\{\bar w^*\}.
\end{align*}
\end{lemma}
\begin{proof}
See \autoref{pf_list_of_stationary_points}.
\end{proof}

\begin{lemma}[Convergence to the critical manifold]
\label{convergence_to_critical_manifold}
Consider the unperturbed flow
$$
\dot\theta=-\nabla_\theta\mathcal L^{\mathrm{bias}}.
$$
Assume $Z\succ0$, fixed $\sigma_\mu^2>0$, and initialization $\theta(0)$ such that $\max_{q,j}|\bar w_{qj}(0)-1/2|\leq\delta$ for sufficiently small $\delta$.
Then all gates remain in $[1/4,3/4]$, and
$$
\max_j|\bar w_{1j}(t)+\bar w_{2j}(t)-1|
\le C\delta e^{-\alpha t},
\qquad
\alpha=\frac{9\sigma_\mu^2\lambda_{\min}(Z)}{256}.
$$
For the truncated flow at fixed $\lambda$, there is also a time $\tcrit=\cO(\log n)$ such that all gates remain in $[1/4,3/4]$ up to $\tcrit$ and
$$
\max_j|\bar w_{1j}(\tcrit)+\bar w_{2j}(\tcrit)-1|=\cO(n^{-1}).
$$
The same conclusion at $\tcrit$ holds for the full population flow at fixed $\lambda$ if $\max_q\|\theta_q(0)\|_2$ is bounded independently of $n$.
Under this initialization bound and the data conditions of \autoref{concentration_for_theta}, fix $\epsilon>0$ and sample $m=\widetilde\Omega(n^{1+\epsilon})$ independent tasks, with $m\in\operatorname{poly}(n)$.
With probability at least $1-1/n$, the empirical flow $\dot\theta=-\nabla_\theta\hat\cL_m$ has the same conclusion at $\tcrit=\cO(\log n)$, with all gates in $[1/4,3/4]$ and $\sup_{0\le t\le\tcrit}\|\theta(t)\|_2=\cO(1)$.
Constants are independent of $n$, but may depend on $\sigma_\mu^2$, $Z$, and the fixed $\lambda,\boldsymbol\omega$.
\end{lemma}
\begin{proof}
See \autoref{pf_convergence_to_critical_manifold}.
\end{proof}

\begin{lemma}[Convergence below loss barrier]
\label{application_of_thikonovs_thm}
Fix $\lambda$ independently of $n$ and suppose the strict expert hierarchy holds.
Assume $Z\succ0$ and $Z_{\mathrm{GA}}<\min\{Z_{\mathrm{GG}},Z_{\mathrm{AA}}\}$.
Suppose
$$
\bar w(\tcrit)\in[1/4,3/4]^{2\times2},
\qquad
n\sigma_\mu^2\max_j|\bar w_{1j}(\tcrit)+\bar w_{2j}(\tcrit)-1|\le C_0,
$$
with $C_0$ independent of $n$.
For the truncated flow and fixed $\sigma_\mu^2>0$, there exist $S,\kappa>0$, independent of $n$, such that, for sufficiently large $n$,
$$
\cL^{\mathrm{truncated}}(\bar w(\tcrit+nS))
<\cL^{\mathrm{truncated}}(\bar w^\circ)-\frac\kappa n
$$
for every
\begin{equation}
\bar w^\circ\in\left\{
\begin{pmatrix}1&1\\0&0\end{pmatrix},
\begin{pmatrix}0&0\\1&1\end{pmatrix},
\begin{pmatrix}0&1\\1&0\end{pmatrix}
\right\}.
\label{eq:suboptimal_corners}
\end{equation}
For the full population flow, assume also that $\max_q\|\theta_q^\perp(\tcrit)\|_2$ is bounded independently of $n$ and $\sigma_\mu^2$.
For sufficiently small fixed $\sigma_\mu^2>0$ and then sufficiently large $n$, both $\cL^{\mathrm{truncated}}(\bar w(\tcrit+nS))$ and $\cL(\theta(\tcrit+nS))$ satisfy the same upper bound.
Under these assumptions and the data and sample-size conditions of the empirical part of \autoref{convergence_to_critical_manifold}, the same conclusions hold for the empirical flow with probability at least $1-1/n$, uniformly over possibly data-dependent initial parameters satisfying the stated bounds.
For the full population and empirical flows, $\sup_{\tcrit\le t\le\tcrit+nS}\|\theta(t)\|_2=\cO(1)$, with a bound independent of $n$ and sufficiently small $\sigma_\mu^2$.
\end{lemma}
\begin{proof}
See \autoref{pf_application_of_thikonovs_thm}.
\end{proof}

\begin{lemma}[Flow at fixed $\lambda$]
\label{close_convergence_to_optimal_stationary_point}
Fix $\lambda$ independently of $n$ and suppose the strict inequalities in \autoref{expertHierarchy} hold.
Assume $Z\succ0$, $Z_{\mathrm{GA}}\ge0$, and fixed $\sigma_\mu^2>0$.
Use \autoref{list_of_stationary_points} with $k_{\mathrm G}=1$ and $k_{\mathrm A}=2$.
Suppose $\bar w(\tcrit)\in(0,1)^{2\times2}$, $\cL^{\mathrm{truncated}}(\bar w(\tcrit))=\cO(1/n)$, and, for some fixed $\kappa>0$,
$$
\cL^{\mathrm{truncated}}(\bar w(\tcrit))
\le \cL^{\mathrm{truncated}}(\bar w^\circ)
-\frac{\kappa}{n}
$$
for every $\bar w^\circ$ in \eqref{eq:suboptimal_corners}.
For sufficiently large $n$, the truncated flow reaches a finite time $\ttheta\ge\tcrit$ with
$$
\|\bar w_{\cdot j}(\ttheta)-e_{k_j}\|_2
\le\frac1n+\frac{\Sigma_{k_jk_j}^{(j)}}{n\sigma_\mu^2+\Sigma_{k_jk_j}^{(j)}}
=\cO(n^{-1}).
$$

For the full population flow, assume that $\|\theta(\tcrit)\|_2$ and $n\cL^{\mathrm{truncated}}(\bar w(\tcrit))$ are bounded independently of $n$ and $\sigma_\mu^2$.
For sufficiently small fixed $\sigma_\mu^2>0$ and then sufficiently large $n$, let $C_0$ be the constant from \autoref{population_gradient_perturbation} and set
$$
\rho=\frac{C_0}{n\sigma_\mu^2}.
$$
There is a time $\ttheta\le \tcrit+2C_0^2n^2\sigma_\mu^2$ with
$$
\|\bar w_{\cdot j}(\ttheta)-\bar w_{\cdot j}^*\|_2
\le\rho+\frac{\Sigma_{k_jk_j}^{(j)}}{n\sigma_\mu^2+\Sigma_{k_jk_j}^{(j)}},
\qquad
\cL^{\mathrm{truncated}}(\bar w(\ttheta))
\le\cL^{\mathrm{truncated}}(\bar w(\tcrit)).
$$
In particular, the gate error is $\cO(n^{-1})$, and $\sup_{\tcrit\le t\le \ttheta}\|\theta(t)\|_2=\cO(\log n)$.
\end{lemma}
\begin{proof}
See \autoref{pf_close_convergence_to_optimal_stationary_point}.
\end{proof}

\begin{lemma}[Empirical flow at fixed $\lambda$]
\label{initial_selection_empirical}
Assume the data conditions of \autoref{concentration_for_theta} and the full-flow assumptions of \autoref{close_convergence_to_optimal_stationary_point}, with $\|\theta(\tcrit)\|_2\le B$ and $n\cL^{\mathrm{truncated}}(\bar w(\tcrit))\le c$.
Let $C_0$ and $\rho$ be as there.
Fix $\epsilon>0$ and sample $m$ independent tasks according to \autoref{adaptiveMeanEstimation}, where $m\in\operatorname{poly}(n)$ and
$$
m\ge n^{1+\epsilon}.
$$
Consider $\dot\theta=-\nabla_\theta\hat\cL_m$ from possibly data-dependent parameters $\theta(\tcrit)$ satisfying these assumptions.
For sufficiently large $n$, with probability at least $1-1/n$, there is a time $\ttheta\le\tcrit+2C_0^2n^2\sigma_\mu^2$ with
$$
\|\bar w_{\cdot j}(\ttheta)-\bar w_{\cdot j}^*\|_2
\le\rho+\frac{\Sigma_{k_jk_j}^{(j)}}{n\sigma_\mu^2+\Sigma_{k_jk_j}^{(j)}},
\qquad
\cL^{\mathrm{truncated}}(\bar w(\ttheta))
\le\cL^{\mathrm{truncated}}(\bar w(\tcrit)),
$$
$$
\sup_{\tcrit\le t\le\ttheta}\|\theta(t)\|_2\le B+3C_0\log(1+cn\sigma_\mu^2)=\cO(\log n),
\qquad
\sup_{\tcrit\le t\le\ttheta}\|\theta_q^\perp(t)\|_2=\cO(1),
$$
and
$$
\cL(\theta(\ttheta))=\cO(n^{-1}),
\qquad \hat\cL_m(\theta(\ttheta))=\cO(n^{-1}).
$$
\end{lemma}
\begin{proof}
See \autoref{pf_initial_selection_empirical}.
\end{proof}

\subsection{Proof of the GLU part of \autoref{thm:glu_dynamics}}
\label{pf_glu_gating}
\begin{proof}[Proof of the GLU part of \autoref{thm:glu_dynamics}]
Since $\lambda$ is fixed in the first stage, $\lambda(\ttheta)=\lambda(0)$, and $|\lambda_k(0)|\le1$ lies in the ranges of \autoref{expertHierarchy_normal}, \autoref{expertHierarchy_mixture} and \autoref{expertHierarchy_uniform} for sufficiently large $n$.
By those lemmas, $0<|\lambda_1(0)|<|\lambda_2(0)|$ makes expert $1$ better in family $\mathrm G$ and expert $2$ better in family $\mathrm A$.
Under this hierarchy and $\bar w_{1j}+\bar w_{2j}=1$, selecting the better expert minimizes the variance in family $j$.
\autoref{feature_geometry} gives $Z\succ0$, $Z_{\mathrm{GA}}\ge0$ and $Z_{\mathrm{GA}}<\min\wbr{Z_{\mathrm{GG}},Z_{\mathrm{AA}}}$ for $d\ge3$, almost surely over $\boldsymbol\omega$.

Apply \autoref{convergence_to_critical_manifold} to obtain the gate and parameter bounds at $\tcrit$.
Then \autoref{application_of_thikonovs_thm} gives, at $\tcrit+nS$, the loss bound below every suboptimal corner and $\|\theta(\tcrit+nS)\|_2=\cO(1)$.
Since the loss at each such corner is $\cO(n^{-1})$, these bounds supply the entry assumptions of \autoref{initial_selection_empirical}, with its initial time shifted to $\tcrit+nS$.
Since $m=\widetilde\Omega(n^{1+\epsilon})\ge n^{1+\epsilon/2}$ for sufficiently large $n$, apply that lemma with $\epsilon/2$ in place of $\epsilon$ to obtain a finite time $\ttheta$ with $\|\bar w_{\cdot j}(\ttheta)-\bar w_{\cdot j}^*\|_2=\cO(n^{-1})$ and $\max_q\|\theta_q(\ttheta)\|_2=\cO(\log n)$.
\autoref{list_of_stationary_points} gives $\|\bar w_{\cdot j}^*-e_{k_j}\|_1=\cO(n^{-1})$, so
$$
\max\{\|\bar w_{\cdot\mathrm G}(\ttheta)-e_1\|_1,\|\bar w_{\cdot\mathrm A}(\ttheta)-e_2\|_1\}=\cO(n^{-1}).
$$
The probability that any of these first-stage conclusions fails is $\cO(n^{-1})$.
\end{proof}

\subsection{Attention-expert dynamics}

\begin{corollary}[Accuracy with fixed gating parameters]
\label{specialization_final_step}
Under the assumptions of \autoref{convergence_of_first_layer}, stop training $\lambda$ at $\tlambda$.
For sufficiently large $n$, with probability at least $1-\eta$, the mean-squared errors on independent test samples satisfy
$$
\bbE_{\mathrm G}[(\widehat\mu-\mu)^2]=\frac{1+o(1)}n,
\qquad
\bbE_{\mathrm A}[(\widehat\mu-\mu)^2]=\begin{cases}
\cO(n^{-2+\epsilon}),&\text{Uniform},\\
\cO(\sigma_n^2/n),&\text{Gaussian mixture}.
\end{cases}
$$
\end{corollary}
\begin{proof}
See \autoref{pf_specialization_final_step}.
\end{proof}

\subsection{Theorem assembly}
\label{pf_theorem_assembly}
\begin{proof}[Proof of \autoref{thm:fresh_context}]
Apply \autoref{thm:softmax_dynamics} for the respective architecture and freeze the gating parameters at $\ttheta$.
For the softmax MoE, represent the frozen weights by the sigmoid parameters $(\theta_1-\theta_2,\theta_2-\theta_1)$, which retain the $\cO(\log n)$ norm bound and give the same loss for every $\lambda$.
Since $\lambda(\ttheta)=\lambda(0)$, the expert parameters retain their initialization bounds, and $m=\widetilde\Omega(n^{1+\epsilon})$ satisfies the sample-size requirement of \autoref{convergence_of_first_layer}.
Its uniform conclusion applies to the data-dependent parameters at $\ttheta$ using the same training tasks and gives a finite target time $\tlambda$.
Finally, \autoref{specialization_final_step} gives the stated test risks at $\tlambda$, with joint probability at least $1-\eta$ by the failure bounds below.
The bounds at fixed $\mu$ follow from the translation identities in \autoref{pf_specialization_final_step}.

Under \autoref{TransformerSimplified}, after relabeling, $0<|\lambda_1(0)|<|\lambda_2(0)|$ almost surely and $\bbP(|\lambda_2(0)|<C^{-1})=C^{-2}$.
For sufficiently large $C>1$, independent of $n$, both $C^{-2}\le\eta/4$ and $\bbP(\max_q\|\theta_q(0)\|_2>C)\le\eta/4$.
For the GLU, choose $\sigma_\theta^2$ sufficiently small that $\bbP(\max_q\|\theta_q(0)\|_2>\min\{C,4\delta\})\le\eta/4$, where $\delta>0$ is the sufficiently small constant in \autoref{convergence_to_critical_manifold}.
On the complementary event, $\max_{q,j}|\bar w_{qj}(0)-1/2|\le\delta$, since $\|\bar\bz_j\|_2=1$ and $|\sigma(u)-1/2|\le|u|/4$.
Thus the required initialization bounds hold jointly with probability at least $1-\eta/2$.
For sufficiently large $n$, the remaining failure probabilities are at most $\eta/4$ per training stage, giving probability at least $1-\eta$ in both \autoref{thm:softmax_dynamics} and \autoref{thm:fresh_context}.
\end{proof}

\section{Accuracy of the final estimator}
\label{app:accuracy}

\begin{proof}[Proof of \autoref{thm:accuracy_final_estimator}]
It suffices to consider sigmoid gates, since for $K=2$ the softmax weights equal the sigmoid weights with parameters $(\theta_1-\theta_2,\theta_2-\theta_1)$, whose norms are still $\cO(\log n)$.
Since $h(\bX+\mu\mathbf1,a)=h(\bX,a)+\mu$, centering the feature coordinates removes the location exactly:
$$
\bz(\bX+\mu\mathbf1)=\bz(\bX).
$$
Thus $\bz$ and the gates depend only on the centered observations and are independent of $\mu$.
Each centered feature satisfies
$$
h(\bX,\omega_r)-\mu
=\frac{n^{-1}\sum_{\ell=1}^n e^{\omega_r(X_\ell-\mu)}(X_\ell-\mu)}
{n^{-1}\sum_{\ell=1}^n e^{\omega_r(X_\ell-\mu)}}.
$$
The summands have finite moments of every order, uniformly in $n$ also for the Gaussian mixture, since its laws are uniformly $1$-sub-Gaussian and $\boldsymbol\omega$ is fixed.
The population denominators are positive, and by \autoref{feature_geometry} the centered population feature norms are bounded below, so the normalization map is locally Lipschitz there.
Moment bounds for the empirical averages therefore give
$$
\bbE_j\sbr{\|\Delta\bz\|_2^{2p}}=\cO(n^{-p})
$$
for every fixed integer $p\ge1$: on a fixed neighborhood of the population averages this follows from local Lipschitz continuity, and outside that neighborhood it follows from $\|\Delta\bz\|_2\le2$ and Markov's inequality applied to the averages.
Using $\max_k\|\theta_k\|_2\le C\log n$, another application of Markov's inequality gives
$$
\bbP\sbr{\max_k|\inner{\theta_k,\Delta\bz}|>1\mid\tau=j}
=\cO\rbr{n^{-p}(\log n)^{2p}}.
$$
Since $\max_k|\inner{\theta_k,\Delta\bz}|\le2C\log n$, splitting at the event above and choosing $p>8C$ gives
$$
\bbE_j\sbr{\exp\rbr{4\max_k|\inner{\theta_k,\Delta\bz}|}}
\le e^4+\cO\rbr{n^{8C-p}(\log n)^{2p}}\le C.
$$

The sigmoid satisfies
$$
\sigma(u+v)\le e^{|v|}\sigma(u),
\qquad
1-\sigma(u+v)\le e^{|v|}(1-\sigma(u)).
$$
Applying these inequalities to each gate yields
$$
|1-w_i(\bX)|+\sum_{k\ne i}w_k(\bX)
\le
\exp\rbr{\max_k|\inner{\theta_k,\Delta\bz}|}
\|\bar w_{\cdot j}-e_i\|_1.
$$
The difference from the selected expert is
$$
\widehat\mu-g_i
=\mu\rbr{\sum_kw_k-1}
+\sum_k(w_k-\mathbf{1}_{\{k=i\}})(g_k-\mu).
$$
The gates and centered experts depend only on the centered observations and are independent of $\mu$.
Since $\bbE[\mu]=0$, averaging over $\mu$ eliminates the cross term and gives
\begin{align*}
\bbE_j\sbr{(\widehat\mu-g_i)^2}
&=\sigma_\mu^2\bbE_j\sbr{\rbr{\sum_kw_k-1}^2}
+\bbE_j\sbr{\rbr{\sum_k(w_k-\mathbf{1}_{\{k=i\}})(g_k-\mu)}^2} \\
&\le \|\bar w_{\cdot j}-e_i\|_1^2
\bbE_j\sbr{
\exp\rbr{2\max_k|\inner{\theta_k,\Delta\bz}|}
\rbr{\sigma_\mu^2+\max_{\ell\in[n]}|X_\ell-\mu|^2}} \\
&\le C\|\bar w_{\cdot j}-e_i\|_1^2
\left(\sigma_\mu^2+
\sqrt{\bbE_j\sbr{\max_{\ell\in[n]}|X_\ell-\mu|^4}}\right).
\end{align*}
The first inequality uses $|g_k-\mu|\le\max_\ell|X_\ell-\mu|$, which holds for every $\lambda_k$, and the second uses Cauchy--Schwarz and the exponential moment bound above.
The triangle inequality in $L^2$ proves the claim.
\end{proof}

The bound on $\bbE_j\sbr{(\widehat\mu-g_i)^2}$ is quadratic in $\|\bar w_{\cdot j}-e_i\|_1$ and uniform over $\lambda$.
It therefore applies throughout training of $\lambda$ after $\theta$ is fixed.

\section{Gating-layer dynamics}
\label{app:routing}
\subsection{Derivation of the population gradient flow}
\begin{remark}[Layer normalization]
Since the centered population feature norms are bounded below by \autoref{feature_geometry}, layer normalization has bounded derivatives on a fixed neighborhood of each such vector.
Composing the moment ratios with this map therefore preserves the orders in $n$ of the Taylor expansions and their remainder bounds in the following gradient and loss calculations, although their coefficients change.
Outside this neighborhood, the moment bounds for the empirical averages and $\|\bz\|_2=1$ control the contribution.
\end{remark}
\begin{lemma}[Population gradient expansion]
Let $\theta_q$ be the gating parameter vector for expert $q\in\{1,2\}$, and let the loss function $\mathcal{L}(\theta)$ be defined as:
\begin{align*}
\mathcal{L}(\theta) &= \frac{1}{2} \mathbb{E} \sbr{ \rbr{ \sum_{i=1}^2 w_i(\bX) g_i - \mu }^2 }
\end{align*}
Then
\begin{align*}
\nabla_{\theta_q} \mathcal{L}(\theta)
&= \frac{1}{2} \sum_j \Biggl\{ \sigma_\mu^2 \left( \sum_{i=1}^2 \bar{w}_{ij} - 1 \right) \bar{w}'_{qj} \bar{\bz}_j + \frac{\bar{w}'_{qj}}{n} \sum_{i=1}^2 \bar{w}_{ij} \Sigma_{qi}^{(j)} \bar{\bz}_j \\
&\quad + \sigma_\mu^2\bar w'_{qj}\bar\bz_j
\sum_{i=1}^2\bar w'_{ij}\theta_i^\top\bbE_j[\Delta\bz]\\
&\quad + \sigma_\mu^2\rbr{\sum_{i=1}^2\bar w_{ij}-1}
\left[\bar w''_{qj}\bar\bz_j\theta_q^\top\bbE_j[\Delta\bz]
+\bar w'_{qj}\bbE_j[\Delta\bz]\right]\\
&\quad + \frac{\sigma_\mu^2}{n} \sum_{i=1}^2 \bar{w}'_{ij}
\left[ \bar{w}''_{qj}
\left(\theta_q^\top \bSigma_Z^{(j)} \theta_i\right) \bar{\bz}_j
+ \bar{w}'_{qj} \bSigma_Z^{(j)} \theta_i  \right]
+ \frac{\sigma_\mu^2}{n}
\bar{w}''_{qj}
\rbr{\sum_{i=1}^2 \bar{w}_{ij} - 1}
\bSigma_Z^{(j)} \theta_q
\\
&\quad + \frac{\sigma_\mu^2}{2n}\bar{w}'_{qj}
\left[\sum_{i=1}^2 \bar{w}''_{ij}
\left(\theta_i^\top\bSigma_Z^{(j)}\theta_i\right)\right]\bar{\bz}_j\\
&\quad + \frac{\sigma_\mu^2}{2n}
\left( \sum_{i=1}^2 \bar{w}_{ij} - 1 \right)
\bar{w}'''_{qj}
\left( \theta_q^\top \bSigma_Z^{(j)} \theta_q \right) \bar{\bz}_j
+ \bar{w}'_{qj} \mathcal{O}(n^{-3/2})
\Biggr\}
\end{align*}
where
$\bar{w}_{ij} = \sigma(\inner{\theta_i,\bar\bz_j})$,
$\bar{\bz}_{j}$ is obtained by applying layer normalization to the population moment ratios,
$\Delta w_i = w_i(\bz) - \bar{w}_{ij}$,
and $\Delta w'_q = w'_q(\bz) - \bar{w}'_{qj}$.
Primes on gates denote derivatives of the sigmoid with respect to its scalar logit, so $w_q'=w_q(1-w_q)$, with the same convention for population gates and higher derivatives.
The matrix $\bSigma^{(j)}$ is defined in \autoref{betterExpert}, and $\bSigma_Z^{(j)}$ is the asymptotic covariance matrix of $\sqrt n(\bz-\bar\bz_j)$ under family $j$.
\end{lemma}
\begin{proof}
\begin{align*}
\nabla_{\theta_q} \mathcal{L}(\theta) &=
\mathbb{E} \sbr{
\rbr{ \sum_{i=1}^2 w_i(\bX) g_i - \mu } g_q w'_q(\bX) \bz
}
\\
&= \frac{1}{2} \sum_j
\mathbb{E}_j \sbr{
\rbr{
    \sum_{i=1}^2 w_i(\bX)
    \rbr{ \mu^2 + (g_i - \mu) (g_q - \mu)}%
    - \mu^2%
}  w'_q(\bz) \bz
}
\\
&= \frac{1}{2} \sum_j
\mathbb{E}_j \Biggl[
\left\{\sum_{i=1}^2(\bar w_{ij}+\Delta w_i)
\bigl(\mu^2+(g_i-\mu)(g_q-\mu)\bigr)-\mu^2\right\}\\
&\hspace{30mm}\times
(\bar w'_{qj}+\Delta w'_q)(\bar\bz_j+\Delta\bz)
\Biggr]
\\
&= \frac{1}{2} \sum_j \Bigg(
\sigma_\mu^2 \rbr{ \sum_{i=1}^2 \bar{w}_{ij} - 1 } \bar{w}'_{qj} \bar{\bz}_j
+ \rbr{ \sum_{i=1}^2 \bar{w}_{ij} \mathbb{E}_j \sbr{(g_i - \mu) (g_q - \mu) } } \bar{w}'_{qj} \bar{\bz}_j \\
&\quad + \sigma_\mu^2\bar w'_{qj}\bar\bz_j
\sum_{i=1}^2\bbE_j[\Delta w_i] + \sigma_\mu^2\rbr{\sum_{i=1}^2\bar w_{ij}-1}
\left[\bar\bz_j\bbE_j[\Delta w'_q]
+\bar w'_{qj}\bbE_j[\Delta\bz]\right]\\
&\quad + \sigma_\mu^2 \sum_{i=1}^2
\mathbb{E}_j \sbr{\Delta w_i \Delta w'_q } \bar{\bz}_{j}
 + \sigma_\mu^2 \sum_{i=1}^2
\mathbb{E}_j \sbr{\Delta w_i \Delta \bz } \bar{w}'_{qj}\\
&\quad +\sigma_\mu^2 \rbr{\sum_{i=1}^2 \bar{w}_{ij} - 1} \mathbb{E}_j \sbr{\Delta w'_q \Delta \bz }\\
&\quad + \rbr{ \sum_{i=1}^2 \mathbb{E}_j \sbr{ \Delta w_i (g_i - \mu)(g_q - \mu) } } \bar{w}'_{qj} \bar{\bz}_j\\
&\quad + \sum_{i=1}^2 \mathbb{E}_j \sbr{ (\bar{w}_{ij} + \Delta w_i)(g_i - \mu)(g_q - \mu) \Delta \bz  } \bar{w}'_{qj}
\\
&\quad + \sum_{i=1}^2 \mathbb{E}_j \sbr{ (\bar{w}_{ij} + \Delta w_i)(g_i - \mu)(g_q - \mu) \Delta w'_q  } \bar{\bz}_{j}
+ \sigma_\mu^2 \sum_{i=1}^2 \mathbb{E}_j \sbr{\Delta w_i \Delta w'_q \Delta \bz }
\\
&\quad + \sum_{i=1}^2 \mathbb{E}_j \sbr{ (\bar{w}_{ij} + \Delta w_i)(g_i - \mu)(g_q - \mu) \Delta w'_q \Delta \bz  } \Bigg)\\
&= \frac{1}{2} \sum_j \Bigg(
\sigma_\mu^2 \rbr{ \sum_{i=1}^2 \bar{w}_{ij} - 1 } \bar{w}'_{qj} \bar{\bz}_j
+ \rbr{ \sum_{i=1}^2 \bar{w}_{ij} \mathbb{E}_j \sbr{(g_i - \mu) (g_q - \mu) } } \bar{w}'_{qj} \bar{\bz}_j \\
&\quad + \sigma_\mu^2\bar w'_{qj}\bar\bz_j
\sum_{i=1}^2\bbE_j[\Delta w_i]\\
&\quad + \sigma_\mu^2\rbr{\sum_{i=1}^2\bar w_{ij}-1}
\left[\bar\bz_j\bbE_j[\Delta w'_q]
+\bar w'_{qj}\bbE_j[\Delta\bz]\right]\\
&\quad + \sigma_\mu^2 \sum_{i=1}^2 \mathbb{E}_j \sbr{\Delta w_i \Delta w'_q } \bar{\bz}_{j}
+ \sigma_\mu^2 \sum_{i=1}^2 \mathbb{E}_j \sbr{\Delta w_i \Delta \bz } \bar{w}'_{qj}\\
&\quad +\sigma_\mu^2 \rbr{\sum_{i=1}^2 \bar{w}_{ij} - 1} \mathbb{E}_j \sbr{\Delta w'_q \Delta \bz }
 + \bar{w}'_{qj}  \cO\rbr{n^{-3/2}}\Bigg)
\end{align*}
Taylor expansion of the gates and \autoref{SoftmaxO1overndeviation} give
\begin{align*}
\nabla_{\theta_q} \mathcal{L}(\theta)
&= \frac{1}{2} \sum_j \Biggl\{ \sigma_\mu^2 \left( \sum_{i=1}^2 \bar{w}_{ij} - 1 \right) \bar{w}'_{qj} \bar{\bz}_j + \frac{\bar{w}'_{qj}}{n} \sum_{i=1}^2 \bar{w}_{ij} \Sigma_{qi}^{(j)} \bar{\bz}_j \\
&\quad + \sigma_\mu^2\bar w'_{qj}\bar\bz_j
\sum_{i=1}^2\bar w'_{ij}\theta_i^\top\bbE_j[\Delta\bz]\\
&\quad + \sigma_\mu^2\rbr{\sum_{i=1}^2\bar w_{ij}-1}
\left[\bar w''_{qj}\bar\bz_j\theta_q^\top\bbE_j[\Delta\bz]
+\bar w'_{qj}\bbE_j[\Delta\bz]\right]\\
&\quad + \frac{\sigma_\mu^2}{n} \sum_{i=1}^2 \bar{w}'_{ij}
\left[ \bar{w}''_{qj}
\left(\theta_q^\top \bSigma_Z^{(j)} \theta_i\right) \bar{\bz}_j
+ \bar{w}'_{qj} \bSigma_Z^{(j)} \theta_i \right]
+ \frac{\sigma_\mu^2}{n}
\bar{w}''_{qj}
\rbr{\sum_{i=1}^2 \bar{w}_{ij} - 1}
\bSigma_Z^{(j)} \theta_q
\\
&\quad + \frac{\sigma_\mu^2}{2n}\bar{w}'_{qj}
\left[\sum_{i=1}^2 \bar{w}''_{ij}
\left(\theta_i^\top\bSigma_Z^{(j)}\theta_i\right)\right]\bar{\bz}_j\\
&\quad + \frac{\sigma_\mu^2}{2n}
\left( \sum_{i=1}^2 \bar{w}_{ij} - 1 \right)
\bar{w}'''_{qj}
\left( \theta_q^\top \bSigma_Z^{(j)} \theta_q \right) \bar{\bz}_j
+ \bar{w}'_{qj} \mathcal{O}(n^{-3/2})
\Biggr\}
\end{align*}
\end{proof}

\begin{lemma}[Population loss expansion]
Let $\theta_q$ be the gating parameter vector for expert $q\in\{1,2\}$, and let the loss function $\mathcal{L}(\theta)$ be defined as:
$$\mathcal{L}(\theta) = \frac{1}{2} \mathbb{E} \sbr{ \rbr{ \sum_{i=1}^2 w_i(\bX) g_i - \mu }^2 }$$
Then the expected loss can be expanded as:
\begin{align*}
\mathcal{L}(\theta) &= \frac{1}{4} \sum_j \Biggl\{
\sigma_\mu^2 \rbr{\sum_{i=1}^2 \bar{w}_{ij} - 1}^2 + \frac{1}{n} \sum_{i=1}^2 \sum_{q=1}^2 \bar{w}_{ij} \bar{w}_{qj} \Sigma_{iq}^{(j)}\\
& \quad+ \frac{\sigma_\mu^2}{n} \sum_{i=1}^2 \sum_{q=1}^2 \bar{w}'_{ij} \bar{w}'_{qj} \left(\theta_i^\top \bSigma_Z^{(j)} \theta_q\right) \\
& \quad+ 2\sigma_\mu^2\rbr{\sum_{i=1}^2\bar w_{ij}-1}
\sum_{i=1}^2\bar w'_{ij}\theta_i^\top\bbE_j[\Delta\bz]\\
& \quad+ \frac{\sigma_\mu^2}{n}\rbr{\sum_{i=1}^2 \bar{w}_{ij} - 1}
\sum_{i=1}^2\bar{w}''_{ij}\rbr{\theta_i^\top\bSigma_Z^{(j)}\theta_i}
\Biggr\}
+ \mathcal{O}(n^{-3/2})
\end{align*}
\end{lemma}
\begin{proof}
We can directly compute
\begin{align*}
\mathcal{L}(\theta)
&= \frac{1}{4} \sum_j \mathbb{E}_j \sbr{ \mu^2 \rbr{\sum_{i=1}^2 w_i(\bX) - 1}^2 + \rbr{\sum_{i=1}^2 w_i(\bX) (g_i - \mu)}^2  }
\\
&= \frac{1}{4} \sum_j \mathbb{E}_j \sbr{ \sigma_\mu^2 \rbr{\sum_{i=1}^2 w_i(\bX) - 1}^2 + \frac{1}{n} \sum_{i=1}^2 \sum_{q=1}^2 w_i(\bX) w_q(\bX) \Sigma_{iq}^{(j)}  }
+\mathcal O(n^{-3/2})
\\
&= \frac{1}{4} \sum_j \mathbb{E}_j \Biggl[ \sigma_\mu^2 \rbr{\sum_{i=1}^2 (\bar{w}_{ij} + \Delta w_i) - 1}^2\\
&\quad + \frac{1}{n} \sum_{i=1}^2 \sum_{q=1}^2 (\bar{w}_{ij} + \Delta w_i) (\bar{w}_{qj} + \Delta w_q) \Sigma_{iq}^{(j)}  \Biggr]
+\mathcal O(n^{-3/2})
\\
&= \frac{1}{4} \sum_j \Biggl\{ \sigma_\mu^2 \rbr{\sum_{i=1}^2 \bar{w}_{ij} - 1}^2 + \frac{1}{n} \sum_{i=1}^2 \sum_{q=1}^2 \bar{w}_{ij} \bar{w}_{qj} \Sigma_{iq}^{(j)} +
2\sigma_\mu^2 \rbr{\sum_{i=1}^2 \bar{w}_{ij} - 1} \sum_{i=1}^2 \mathbb{E}_j\sbr{\Delta w_i }
\\ &\quad + \sigma_\mu^2 \sum_{i=1}^2 \sum_{q=1}^2 \mathbb{E}_j \sbr{ \Delta w_i \Delta w_q  } + \frac{1}{n} \sum_{i=1}^2 \sum_{q=1}^2 \Sigma_{iq}^{(j)} \mathbb{E}_j \sbr{ \Delta w_i \Delta w_q  } \Biggr\}
+\mathcal O(n^{-3/2})
\\
&= \frac{1}{4} \sum_j \Biggl\{
\sigma_\mu^2 \rbr{\sum_{i=1}^2 \bar{w}_{ij} - 1}^2 + \frac{1}{n} \sum_{i=1}^2 \sum_{q=1}^2 \bar{w}_{ij} \bar{w}_{qj} \Sigma_{iq}^{(j)}\\
& \quad+ \frac{\sigma_\mu^2}{n} \sum_{i=1}^2 \sum_{q=1}^2 \bar{w}'_{ij} \bar{w}'_{qj} \left(\theta_i^\top \bSigma_Z^{(j)} \theta_q\right) \\
& \quad+ 2\sigma_\mu^2\rbr{\sum_{i=1}^2\bar w_{ij}-1}
\sum_{i=1}^2\bar w'_{ij}\theta_i^\top\bbE_j[\Delta\bz]\\
& \quad+ \frac{\sigma_\mu^2}{n}\rbr{\sum_{i=1}^2 \bar{w}_{ij} - 1}
\sum_{i=1}^2\bar{w}''_{ij}\rbr{\theta_i^\top\bSigma_Z^{(j)}\theta_i}
\Biggr\}
+ \mathcal{O}(n^{-3/2})
\end{align*}
\end{proof}

\subsubsection{Flow reparameterization}
Use $S_{qj}$ from \autoref{list_of_stationary_points}.
Recall that we are studying the gradient flow
\begin{align*}
\dot{\theta}_q
&= - \nabla_{\theta_q} \mathcal{L}(\theta)\\
&= - \frac{1}{2} \sum_j \Biggl\{
S_{qj}(\bar w)\bar{w}'_{qj}\bar{\bz}_j \\
&\quad + \sigma_\mu^2\bar w'_{qj}\bar\bz_j
\sum_{i=1}^2\bar w'_{ij}\theta_i^\top\bbE_j[\Delta\bz]\\
&\quad + \sigma_\mu^2\rbr{\sum_{i=1}^2\bar w_{ij}-1}
\left[\bar w''_{qj}\bar\bz_j\theta_q^\top\bbE_j[\Delta\bz]
+\bar w'_{qj}\bbE_j[\Delta\bz]\right]\\
&\quad + \frac{\sigma_\mu^2}{n} \sum_{i=1}^2 \bar{w}'_{ij}
\left[ \bar{w}''_{qj}
\left(\theta_q^\top \bSigma_Z^{(j)} \theta_i\right) \bar{\bz}_j
+ \bar{w}'_{qj} \bSigma_Z^{(j)} \theta_i \right]\\
&\quad + \frac{\sigma_\mu^2}{n}
\bar{w}''_{qj}
\rbr{\sum_{i=1}^2 \bar{w}_{ij} - 1}
\bSigma_Z^{(j)} \theta_q\\
&\quad + \frac{\sigma_\mu^2}{2n}\bar{w}'_{qj}
\left[\sum_{i=1}^2 \bar{w}''_{ij}
\left(\theta_i^\top\bSigma_Z^{(j)}\theta_i\right)\right]\bar{\bz}_j\\
&\quad + \frac{\sigma_\mu^2}{2n}
\left( \sum_{i=1}^2 \bar{w}_{ij} - 1 \right)
\bar{w}'''_{qj}
\left( \theta_q^\top \bSigma_Z^{(j)} \theta_q \right) \bar{\bz}_j
+ \bar{w}'_{qj} \mathcal{O}(n^{-3/2})
\Biggr\}
\end{align*}
We now re-parameterize the flow in terms of $\bar w=\rbr{\bar{w}_{ij}}_{i\in[2],j\in\{\mathrm G,\mathrm A\}}$, using $\bar w_{qk}=\sigma(\inner{\theta_q,\bar\bz_k})$ and the chain rule.
\begin{align*}
\frac{\partial \bar{w}_{qk}}{\partial t}
&= \frac{\partial \bar{w}_{qk}}{\partial \theta_q} \frac{\partial \theta_q}{\partial t} \\
& =
-
\Bigg(
\frac{1}{2} \sum_j \Biggl\{
S_{qj}(\bar w)
\bar{w}'_{qj}\bar{w}'_{qk}
\inner{\bar{\bz}_j,\bar{\bz}_k}
\\
&\quad + \sigma_\mu^2\bar w'_{qj}\bar w'_{qk}
\inner{\bar\bz_j,\bar\bz_k}
\sum_{i=1}^2\bar w'_{ij}\theta_i^\top\bbE_j[\Delta\bz]\\
&\quad + \sigma_\mu^2\bar w'_{qk}\rbr{\sum_{i=1}^2\bar w_{ij}-1}
\Bigl[\bar w''_{qj}\inner{\bar\bz_j,\bar\bz_k}\theta_q^\top\bbE_j[\Delta\bz]\\
&\hspace{38mm}
+\bar w'_{qj}\inner{\bbE_j[\Delta\bz],\bar\bz_k}\Bigr]\\
&\quad + \frac{\bar{w}'_{qk} \sigma_\mu^2}{n}
\sum_{i=1}^2 \bar{w}'_{ij}\Bigl[
\bar{w}''_{qj}
\left(\theta_q^\top \bSigma_Z^{(j)} \theta_i\right)
\inner{\bar{\bz}_j, \bar{\bz}_k}\\
&\hspace{38mm}
+ \bar{w}'_{qj}
\inner{\bSigma_Z^{(j)} \theta_i, \bar{\bz}_k}\Bigr]\\
&\quad+ \frac{\sigma_\mu^2 \bar{w}''_{qj} \bar{w}'_{qk} }{n}
\rbr{\sum_{i=1}^2 \bar{w}_{ij} - 1}
\inner{\bSigma_Z^{(j)}\theta_q, \bar{\bz}_k}
\\
&\quad + \frac{\sigma_\mu^2}{2n}\bar w'_{qj}\bar w'_{qk}
\left[\sum_{i=1}^2\bar w''_{ij}
\left(\theta_i^\top\bSigma_Z^{(j)}\theta_i\right)\right]
\inner{\bar\bz_j,\bar\bz_k}\\
&\quad + \frac{\sigma_\mu^2}{2n}\bar w'_{qk}
\rbr{\sum_{i=1}^2\bar w_{ij}-1}\bar w'''_{qj}
\left(\theta_q^\top\bSigma_Z^{(j)}\theta_q\right)
\inner{\bar\bz_j,\bar\bz_k}\\
&\quad + \bar{w}'_{qk} \bar{w}'_{qj} \mathcal{O}(n^{-3/2})
\Biggr\}\Bigg)
\end{align*}

\subsection{Concentration bounds}
\subsubsection{Uniform concentration bound for the empirical gradient}
\begin{lemma}[Concentration bound for the gradient of $\theta$]
\label{concentration_for_theta}
Let $B_d(r)$ denote the Euclidean ball of radius $r$ in $\bbR^d$.
Fix $j\in\{\mathrm G,\mathrm A\}$ and sample the $m$ tasks independently conditional on $\tau_b=j$, writing $\bar\bz=\bar\bz_j$ and $\bar w_k=\bar w_{kj}$.
Fix $\sigma_\mu^2$ and $\lambda$ independently of $n$, with either data law from \autoref{adaptiveMeanEstimation}.
Let $0<\beta<1/2$, $0<\delta<1$, $r=\cO(\log n)$, and suppose $m,\delta^{-1}\in\operatorname{poly}(n)$.
For sufficiently large $n$, with probability at least $1-\delta$, simultaneously for all $\theta_1,\theta_2\in B_d(r)$ and all $q\in\{1,2\}$ and $s\in[d]$,
\begin{align*}
&\Bigg|
    \frac{1}{m}
    \sum_{b=1}^m \sbr{\rbr{ \sum_{k=1}^2 w_k(\bX_b) g_k(\bX_b) - \mu_b } g_{q}(\bX_b) w'_{q}(\bX_b) z_s(\bX_b)}
\\
&\quad-
    \mathbb{E}_j \sbr{
    \rbr{ \sum_{k=1}^2 w_k(\bX) g_k - \mu } g_q w'_q(\bX) z_s
    }
\Bigg|
\\
&\le
\rbr{\left\vert{}\sum_{k=1}^2\bar w_k-1\right\vert{}+n^{-1/2+\beta}}
\bar w_q(1-\bar w_q)\,\cOs(m^{-1/2}).
\end{align*}
\end{lemma}
\begin{proof}
Set $M=n^{-1/2+\beta}$ and choose $R,C=\cO(\sqrt{\log n})$, $C\ge1$, with sufficiently large constants.
Apply \autoref{SMconcentration} with $k=1$ at the fixed $\pm\lambda_k$ and $\omega_s$, using deviations equal to a sufficiently small constant multiple of $M$.
Since $nM^2/(R^2e^{2R|\lambda|})=n^{2\beta-o(1)}$, symmetry of the centered observations, the fixed $d$ and $\boldsymbol\omega$, and the Gaussian tail of $\mu$ give
$$
\mathcal E=\left\{\max_k|g_k-\mu|\le M,\ |\mu|\le C,\ \|\bz-\bar\bz\|_2\le M\right\},
\qquad \bbP_j(\mathcal E^c)=\cO(n^{-A})
$$
for any prescribed fixed $A>0$.
Apply the deterministic covers in \autoref{concentration_for_grad_theta_helper} to the integrands multiplied by $\mathbf1_{\mathcal E}$.
The discarded normalized integrand is bounded by
$$
\frac{3e^{2r}}{M}\left(|\mu|+\max_{\ell\in[n]}|X_\ell-\mu|\right)^2,
$$
since $w'_q/\bar w'_q\le e^{2r}$ and $\Gamma_1\ge M$.
Its $L^2$ norm is polynomial in $n$, so Cauchy--Schwarz and sufficiently large $A$ make the discarded expectation $o(m^{-1/2})$ and the probability of discarding any training task at most $\delta/2$.
Use failure probability $\delta/2$ for the fixed number of coordinates in the helper bound, where $rM=o(1)$ and its remaining factors are logarithmic.
\end{proof}

\begin{lemma}[Concentration bound for the normalized gradient]
\label{concentration_for_grad_theta_helper}
Fix $j\in\{\mathrm G,\mathrm A\}$, a gate index $q\in[2]$, a coordinate $s\in[d]$, a radius $r>0$ and a failure probability $\delta\in(0,1)$, and sample $m$ tasks independently conditional on $\tau_b=j$,
writing $\bar\bz=\bar\bz_j$ and $\bar w_k=\bar w_{kj}$.
Suppose $M>0$, $C\ge1$, and, almost surely for a single task and every $k$,
\[
|g_k(\bX)-\mu|\le M,
\qquad |\mu|\le C,
\qquad \|\bz(\bX)-\bar\bz\|_2\le M,
\qquad \|\bar\bz\|_2\le C.
\]
Let
\[
w_k(\bX)=\sigma(\inner{\theta_k,\bz(\bX)}),
\qquad
\bar w_k=\sigma(\inner{\theta_k,\bar\bz}),
\]
and define
\[
\Gamma_1=\left|\sum_{k=1}^2\bar w_k-1\right|+M,
\qquad
\Gamma_2=\bar w_q(1-\bar w_q),
\]
\[
\Psi_{\theta,q,s}(\bX,\mu)
=\left(\sum_{k=1}^2w_k(\bX)g_k(\bX)-\mu\right)
 g_q(\bX)w_q'(\bX)z_s(\bX).
\]
Then, with probability at least $1-\delta$,
\begin{align*}
&\sup_{\theta_1,\theta_2\in B_d(r)}
\frac{1}{\Gamma_1\Gamma_2}
\left|\frac1m\sum_{b=1}^m\Psi_{\theta,q,s}(\bX_b,\mu_b)
      -\bbE_j[\Psi_{\theta,q,s}(\bX,\mu)]\right|\\
&\quad\lesssim
2(C+M+Cr)e^{rM}(M+C)^2\\[-2pt]
&\qquad\quad\times
\left[
\frac{2d\log\{2+2mr(C+M)(1+M^{-1})\}+\log(2/\delta)}{m}
\right]^{1/2}.
\end{align*}
\end{lemma}
\begin{proof}
Here $L_\infty$ denotes the supremum over the full support of a single task.
All covers below are deterministic.
Define
$$\mathcal{F}_a = \left\{ f_{\theta}:(\bz,\bg,\mu)\mapsto\frac{1}{\Gamma_1}\left(\sum_{k=1}^2\sigma(\theta_k^\top\bz)g_k-\mu\right):\theta_k\in B_d(r)\right\}.$$
For $f_{\theta}\in\mathcal F_a$ we can write
\begin{align*}
f_{\theta}
=\frac{1}{\Gamma_1}\left(
\sum_{k=1}^2w_k(g_k-\mu)
+\mu\sum_{k=1}^2(w_k-\bar w_k)
+\mu\left(\sum_{k=1}^2\bar w_k-1\right)\right).
\end{align*}
Since $|w_k-\bar w_k|\le rM/4$ and $C\ge1$, this gives
$$\|f_{\theta}\|_\infty\le 2(C+M+Cr)\equiv B_a.$$
The normalizing factor also depends on $\theta$, with
$$
|\Gamma_1(\theta)-\Gamma_1(\theta')|
\le\frac C2\max_k\|\theta_k-\theta'_k\|_2.
$$
Using $\Gamma_1\ge M$ and the sigmoid's $1/4$-Lipschitz bound, we obtain
$$
\|f_{\theta}-f_{\theta'}\|_\infty
\le\frac{(C+M)^2+B_aC}{2M}\max_k\|\theta_k-\theta'_k\|_2
\le\frac{2B_a(C+M)}{M}\max_k\|\theta_k-\theta'_k\|_2.
$$
A Euclidean cover of each $B_d(r)$ therefore gives
\begin{equation}
\log\mathcal N(u,\mathcal F_a,L_\infty)
\lesssim 2d\log\left(2+\frac{2rB_a(C+M)}{Mu}\right).
\label{refMetricEntropyOfScaledFunctionClass}
\end{equation}
Define
$$
\mathcal F_b=\left\{(\bz,\bg,\mu)\mapsto\frac{w'_q g_q z_s}{\Gamma_2}:\theta_q\in B_d(r)\right\}.
$$
Since $|\frac{d}{du}\log\sigma'(u)|\le1$,
$$
\frac{w'_q}{\Gamma_2}\le e^{rM},\qquad
\left\|\nabla_{\theta_q}\log\frac{w'_q}{\Gamma_2}\right\|_2\le M+2C.
$$
Thus the functions in $\mathcal F_b$ are bounded in absolute value by $e^{rM}(M+C)^2$ and have Lipschitz constant at most $e^{rM}(M+C)^2(M+2C)$ in $\theta_q$, including the dependence of $\Gamma_2$.
Let $\mathcal F$ consist of the products $f_af_b$ with the same parameters in both factors. These products are bounded in absolute value by $B_ae^{rM}(M+C)^2$.
A deterministic parameter cover with mesh $[4m(C+M)(1+M^{-1})]^{-1}$ gives
$$
\log\mathcal N\left(\frac{B_ae^{rM}(M+C)^2}{m},\mathcal F,L_\infty\right)
\lesssim 2d\log\left(2+2mr(C+M)(1+M^{-1})\right).
$$
Hoeffding's inequality and a union bound over this cover, followed by approximation in $L_\infty$, give, with probability at least $1-\delta$,
$$
\begin{aligned}
\sup_{f\in\mathcal F}\left|\frac1m\sum_{b=1}^m f(\bX_b)-\bbE_j[f]\right|
&\lesssim B_ae^{rM}(M+C)^2\\
&\quad\times\left[
\frac{2d\log\left(2+2mr(C+M)(1+M^{-1})\right)+\log(2/\delta)}{m}
\right]^{1/2}.
\end{aligned}
$$
The approximation contributes at most $2B_ae^{rM}(M+C)^2/m$, which is absorbed in this bound.
\end{proof}

\subsubsection{Uniform concentration for the loss}

\begin{lemma}[Local concentration for the loss]
\label{concentration_for_loss_local}
Under the assumptions of \autoref{concentration_for_theta}, with probability at least $1-\delta$, uniformly over $\theta_1,\theta_2\in B_d(r)$ satisfying $|\sum_k\bar w_k-1|=\cO(n^{-1})$ with a fixed implied constant,
$$
|\hat\cL^{(j)}_m(\theta)-\cL^{(j)}(\theta)|
=\cOs\left(\frac{n^{-1+\beta}}{\sqrt m}+\frac{n^{-1+2\beta}}{m}\right).
$$
\end{lemma}
\begin{proof}
Use $\mathcal E$, $M$, and $C$ from the proof of \autoref{concentration_for_theta}, and apply the covers in \autoref{concentration_for_loss} to $\mathcal F_{sq}\mathbf1_{\mathcal E}$.
The normalized squared residual has envelope
$$\frac{9}{M^2}\left(|\mu|+\max_{\ell\in[n]}|X_\ell-\mu|\right)^2.$$
Cauchy--Schwarz makes the discarded expectation $o(m^{-1})$, with the probability of discarding a training task absorbed in $\delta$.
The fixed-parameter moment-ratio argument in the proof of \autoref{thm:accuracy_final_estimator}, applied also to the paired experts, and independence of $\mu$ from the centered observations give
$$
\cL^{(j)}=\cO\left(\frac{(1+r)^2}{n}\right)=\cOs(n^{-1}),
\qquad \Gamma_1^2U=\cOs(n^{-1+2\beta}).
$$
Substitute into \autoref{concentration_for_loss}.
The discarded expectation is absorbed by its linear term.
\end{proof}

\begin{lemma}[Uniform concentration for the loss]
\label{concentration_for_loss}
Fix $j\in\{\mathrm G,\mathrm A\}$ and sample the $m$ tasks independently conditional on $\tau_b=j$, writing $\bar\bz=\bar\bz_j$ and $\bar w_k=\bar w_{kj}$.
Suppose $M>0$, $C\ge1$, and the following bounds hold almost surely for a single task and every $k$:
\begin{align*}
\vert{}g_k(\bX) - \mu\vert{} &\le M
\\
\vert{}\mu\vert{} &\le C
\\
\|\mathbf{z}(\bX) - \bar{\mathbf{z}}\|_2 &\le M
\\
\|\bar{\mathbf{z}}\|_2 &\le C
\\
w_k(\bX) &= \sigma(\langle \theta_k, \mathbf{z}(\bX) \rangle)
\\
\bar{w}_k &= \sigma(\langle \theta_k, \bar{\mathbf{z}} \rangle)
\\
\Gamma_1 &= \left\vert{} \sum_{k=1}^2 \bar{w}_k - 1 \right\vert{} + M
\\
U &= 4(C + M + Cr)^2
\end{align*}

Let
$$\hat\cL^{(j)}_m(\theta) = \frac{1}{2m} \sum_{b=1}^m \left( \sum_{k=1}^2 w_k(\bX_b) g_k(\bX_b) - \mu_b \right)^2, \quad \cL^{(j)}(\theta) = \frac{1}{2} \mathbb{E}_j \left[ \left( \sum_{k=1}^2 w_k(\bX) g_k(\bX) - \mu \right)^2 \right]$$
Then, with probability at least $1 - \delta$, the following holds for all $\theta_1,\theta_2\in B_d(r)$:
\begin{align*}
\left|\hat\cL^{(j)}_m(\theta)-\cL^{(j)}(\theta)\right|
&\lesssim\sqrt{\cL^{(j)}(\theta)\,\Gamma_1^2U\,
\frac{2d\log\left(2+2mr(C+M)(1+M^{-1})\right)+\log(2/\delta)}{m}}\\
&\quad+\Gamma_1^2U\,
\frac{2d\log\left(2+2mr(C+M)(1+M^{-1})\right)+\log(2/\delta)}{m}.
\end{align*}
\end{lemma}
\begin{proof}
Use $\mathcal F_a$ from the proof of \autoref{concentration_for_grad_theta_helper}, with $B_a=2(C+M+Cr)$, and let
$$\mathcal F_{sq}=\{f_{\theta}^2:f_{\theta}\in\mathcal F_a\}.$$
For every $u\in\mathcal F_{sq}$,
$$0\le u\le B_a^2=U,\qquad \operatorname{Var}(u)\le\bbE_j[u^2]\le U\bbE_j[u].$$
Since squaring is $2B_a$-Lipschitz on $[-B_a,B_a]$, \eqref{refMetricEntropyOfScaledFunctionClass} gives
$$
\log\mathcal N(u,\mathcal F_{sq},L_\infty)
\lesssim 2d\log\left(2+\frac{2rU(C+M)}{Mu}\right).
$$
Choose $u_0=U/m$ and a deterministic $u_0$-cover $\mathcal C_{u_0}\subset\mathcal F_{sq}$ in $L_\infty$ over the full support. Then
$$
\log|\mathcal C_{u_0}|
\lesssim 2d\log\left(2+2mr(C+M)(1+M^{-1})\right).
$$
Bernstein's inequality and a union bound give, simultaneously for every $v\in\mathcal C_{u_0}$, with probability at least $1-\delta$,
$$
\left|\frac1m\sum_{b=1}^m v(\bX_b)-\bbE_j[v]\right|
\le\sqrt{\frac{2U\bbE_j[v]\log(2|\mathcal C_{u_0}|/\delta)}{m}}
+\frac{2U\log(2|\mathcal C_{u_0}|/\delta)}{3m}.
$$
For any $u\in\mathcal F_{sq}$ choose $v$ in this cover with $\|u-v\|_\infty\le U/m$. This controls both the empirical average and the expectation, so
$$
\left|\frac1m\sum_{b=1}^m u(\bX_b)-\bbE_j[u]\right|
\le\left|\frac1m\sum_{b=1}^m v(\bX_b)-\bbE_j[v]\right|+\frac{2U}{m},
\qquad \bbE_j[v]\le\bbE_j[u]+\frac Um.
$$
The resulting approximation terms are absorbed by the linear term in Bernstein's bound. Finally, for $u=f_{\theta}^2$,
$$
\hat\cL^{(j)}_m(\theta)=\frac{\Gamma_1^2}{2m}\sum_{b=1}^m u(\bX_b),
\qquad \cL^{(j)}(\theta)=\frac{\Gamma_1^2}{2}\bbE_j[u].
$$
Multiplying by $\Gamma_1^2/2$ proves the claimed bound simultaneously for all parameters.
\end{proof}

\subsection{Analysis of individual stationary points: complete flow}

\subsubsection{Proof of \autoref{list_of_stationary_points}}
\label{pf_list_of_stationary_points}
\begin{proof}
Since $Z\succ0$,
\begin{align*}
\frac d{dt}\cL^{\mathrm{truncated}}(\bar w)
&=-\frac14\sum_{q=1}^2
\left\|\sum_j\bar w'_{qj}S_{qj}(\bar w)\bar\bz_j\right\|_2^2,\\
\dot{\bar w}=0
&\quad\Longleftrightarrow\quad
\bar w'_{qj}S_{qj}(\bar w)=0\quad\text{for all }q,j.
\end{align*}
For $k\ne k_j$, the strict hierarchy gives
$$
S_{kj}(\bar w)-S_{k_jj}(\bar w)
=\frac{\Sigma_{k k_j}^{(j)}-\Sigma_{k_jk_j}^{(j)}}n\bar w_{k_jj}
+\frac{\Sigma_{kk}^{(j)}-\Sigma_{k k_j}^{(j)}}n\bar w_{kj}>0
$$
whenever $\bar w_{k_jj}+\bar w_{kj}>0$, excluding two interior gates in the same context.
For $q\ne k$ and $0<\bar w^\circ_{qj}<1$, stationarity requires
\begin{align*}
\bar w^\circ_{kj}=1
&\quad\Longrightarrow\quad
0=S_{qj}(\bar w^\circ)
=\left(\sigma_\mu^2+\frac{\Sigma_{qq}^{(j)}}n\right)\bar w^\circ_{qj}
+\frac{\Sigma_{qk}^{(j)}}n>0,\\
\bar w^\circ_{kj}=0
&\quad\Longrightarrow\quad
\bar w^\circ_{qj}=\frac{n\sigma_\mu^2}{n\sigma_\mu^2+\Sigma_{qq}^{(j)}}.
\end{align*}
The four corners have $\bar w'_{qj}=0$, completing the classification.
The unique optimum follows from convexity, strict convexity on the selected coordinates, and
\begin{gather*}
\frac{\partial\cL^{\mathrm{truncated}}}{\partial\bar w_{qj}}=\frac12S_{qj},
\qquad
S_{k_jj}(\bar w^*)=0,\\
S_{kj}(\bar w^*)
=\frac{\bar w_{k_jj}^*}{n}
\bigl(\Sigma_{k k_j}^{(j)}-\Sigma_{k_jk_j}^{(j)}\bigr)>0
\quad(k\ne k_j).
\end{gather*}
\end{proof}

\subsubsection{Proof of \autoref{convergence_to_critical_manifold}}
\label{pf_convergence_to_critical_manifold}
\begin{proof}[Proof of \autoref{convergence_to_critical_manifold}]
Whenever \(\bar w_{qj}\in[1/4,3/4]\), we have \(\bar w'_{qj}\geq3/16\), and therefore
\begin{align*}
\|\nabla_\theta\mathcal L^{\mathrm{bias}}\|_2^2
&=\frac{\sigma_\mu^4}{4}
\sum_{q=1}^2
\left\|
\sum_j
(\bar w_{1j}+\bar w_{2j}-1)\bar w'_{qj}\bar{\mathbf z}_j
\right\|_2^2\\
&\geq
\frac{\sigma_\mu^4\lambda_{\min}(Z)}{4}
\sum_{q,j}
(\bar w_{1j}+\bar w_{2j}-1)^2(\bar w'_{qj})^2\\
&\geq 2\alpha\mathcal L^{\mathrm{bias}}.
\end{align*}
Consequently, locally the flow satisfies a PL inequality.
Moreover,
$$
-\frac{d}{dt}\sqrt{\mathcal L^{\mathrm{bias}}}
=\frac{\|\nabla_\theta\mathcal L^{\mathrm{bias}}\|_2^2}
{2\sqrt{\mathcal L^{\mathrm{bias}}}}
\geq\sqrt{\frac{\alpha}{2}}\,\|\dot\theta\|_2
$$
Thus,
$$
\|\theta(0) - \theta(t)\|_2
\leq
\int_0^t\|\dot\theta(s)\|_2\,ds
\leq
\sqrt{\frac{2\mathcal L^{\mathrm{bias}}(\theta(0))}{\alpha}}.
$$
Since the map $\theta\mapsto\bar w$ is globally
$\sqrt{\lambda_{\max}(Z)}/4$-Lipschitz, we obtain
$$
\int_0^t\|\dot{\bar w}(s)\|_{\mathrm F}\,ds
\leq
\frac{\sqrt{\lambda_{\max}(Z)}}4
\sqrt{\frac{2\mathcal L^{\mathrm{bias}}(\theta(0))}{\alpha}}.
$$
Since \(\max_{q,j}|\bar w_{qj}(0)-1/2|\leq\delta\), we have
\(\mathcal L^{\mathrm{bias}}(\theta(0))\leq2\sigma_\mu^2\delta^2\).
Thus we obtain a travel bound on $\bar{w}$ of
$$
\|\bar{w}(0) - \bar{w}(t)\|_{\mathrm F} \leq C\delta.
$$
With
$(1+C)\delta<1/4$, a first-exit argument shows that the flow remains in
\(\bar{w} \in [1/4,3/4]^{2 \times 2}\) for all time.
Moreover,
$$
\frac{d}{dt}\cL^{\mathrm{bias}}
=-\|\nabla_\theta\cL^{\mathrm{bias}}\|_2^2
\le-2\alpha\cL^{\mathrm{bias}},
$$
which gives the asserted exponential bound on the sum residuals.

The truncated gradient differs from the bias gradient by $\cO(n^{-1})$ at fixed $\lambda$.
The same bound holds for the full population gradient on bounded sets of $\theta$, by the loss and gradient expansions.
The PL inequality and
$\|\nabla_\theta\cL^{\mathrm{bias}}\|_2\le C\sqrt{\cL^{\mathrm{bias}}}$ give
$$
\frac{d}{dt}\sqrt{\cL^{\mathrm{bias}}}
\le-\alpha\sqrt{\cL^{\mathrm{bias}}}+\frac Cn,
\qquad
\sqrt{\cL^{\mathrm{bias}}(\theta(t))}
\le e^{-\alpha t}\sqrt{\cL^{\mathrm{bias}}(\theta(0))}
+\frac{C}{\alpha n}.
$$
At zero loss the inequality is interpreted using the upper right derivative.
Integrating $\|\dot\theta\|_2\le C\sqrt{\cL^{\mathrm{bias}}}+C/n$ now bounds the gate travel by
$C\delta+C(1+t)/n$.
Choose $\tcrit=2\alpha^{-1}\log n$ and then $n$ sufficiently large.
The same first-exit argument keeps the gates in $[1/4,3/4]$ up to $\tcrit$, and the displayed bound gives the claimed $\cO(n^{-1})$ residual.
For the full population flow, the same travel bound keeps $\theta$ in a fixed bounded set up to $\tcrit$.

For the empirical flow, apply \autoref{concentration_for_theta} with $\beta=\min\{\epsilon,1\}/4$ and failure probabilities of order $1/n$, conditional on the family labels.
Hoeffding's inequality bounds the deviations of the family proportions from $1/2$ by $\cOs(m^{-1/2})$.
On bounded sets of $\theta$, the conditional gradient expansion bounds each family's population gradient by $C\max_j|\bar w_{1j}+\bar w_{2j}-1|+C/n$.
Combining these bounds gives, with probability at least $1-1/n$, uniformly on these sets,
\begin{equation}
\|\nabla_\theta(\hat\cL_m-\cL)\|_2
\le\cOs(n^{-1/2-\epsilon/2})\max_j|\bar w_{1j}+\bar w_{2j}-1|
+\cOs(n^{-1-\epsilon/4}).
\label{eq:empirical_initial_gradient}
\end{equation}
For fixed $\sigma_\mu^2>0$, the first term is $o(1)\sqrt{\cL^{\mathrm{bias}}}$ and the second is $o(n^{-1})$.
Up to first exit from the bounded parameter set and $[1/4,3/4]^{2\times2}$, the preceding PL inequality therefore gives
$$
\frac{d}{dt}\sqrt{\cL^{\mathrm{bias}}}
\le-\frac\alpha2\sqrt{\cL^{\mathrm{bias}}}+\frac Cn,
\qquad
\sqrt{\cL^{\mathrm{bias}}(\theta(t))}
\le e^{-\alpha t/2}\sqrt{\cL^{\mathrm{bias}}(\theta(0))}
+\frac{2C}{\alpha n}.
$$
The same integration of $\|\dot\theta\|_2$ bounds the travel by $C\delta+C(1+t)/n$ and prevents exit up to $\tcrit=2\alpha^{-1}\log n$.
This proves the empirical residual and parameter bounds.
\end{proof}

\subsubsection{Proof of \autoref{application_of_thikonovs_thm}}
\label{pf_application_of_thikonovs_thm}
\begin{proof}[Proof of \autoref{application_of_thikonovs_thm}]
First, we check the assumptions of \autoref{avoidance_of_suboptimal_stationary_points}.
By \autoref{feature_geometry}, almost surely over $\boldsymbol\omega$ we have $Z\succ0$, $Z_{\mathrm{GA}}\ge0$ and $Z_{\mathrm{GA}}<\min\wbr{Z_{\mathrm{GG}},Z_{\mathrm{AA}}}$, with $\lambda_{\min}(Z)$ bounded below independently of $n$.

Set $s=(t-\tcrit)/n$ and
$$
r_j=n\sigma_\mu^2(\bar w_{1j}+\bar w_{2j}-1).
$$
The truncated flow satisfies
\begin{align*}
\frac1{n\sigma_\mu^2}\frac{dr_k}{ds}
&=-\frac12\bar w'_{1k}\sum_j Z_{jk}\bar w'_{1j}
\left(2r_j+\sum_{i=1}^2\bar w_{ij}(\Sigma_{1i}^{(j)}+\Sigma_{2i}^{(j)})\right)
+\cO\left(\frac1{n\sigma_\mu^2}\right),\\
\frac d{ds}\left[\frac{\sigma^{-1}(\bar w_{1k})-\sigma^{-1}(\bar w_{2k})}{2}\right]
&=-\frac1{4}\sum_j Z_{jk}\bar w'_{1j}n(S_{1j}-S_{2j})(\bar w)
+\cO\left(\frac1{n\sigma_\mu^2}\right),
\end{align*}
where $S_{qj}$ is defined in \autoref{list_of_stationary_points}.
On bounded coordinate sets, the vector fields are bounded through two derivatives, and the remainder bounds hold through two derivatives.

Initialize the reduced flow in \autoref{avoidance_of_suboptimal_stationary_points} by
$$
\bar w^{\mathrm{reduced}}_{1j}(0)
=\sigma\left(\frac{\sigma^{-1}(\bar w_{1j}(\tcrit))-\sigma^{-1}(\bar w_{2j}(\tcrit))}{2}\right),
\qquad
\bar w^{\mathrm{reduced}}_{2j}(0)=1-\bar w^{\mathrm{reduced}}_{1j}(0).
$$
Then
\begin{align*}
\frac d{ds}\left(n\cL^{\mathrm{truncated}}(\bar w^{\mathrm{reduced}})\right)
&=-\frac1{8}\sum_{j,k}
\bar w_{1j}^{\mathrm{reduced}\,\prime}n(S_{1j}-S_{2j})(\bar w^{\mathrm{reduced}})\\
&\qquad\times Z_{jk}\bar w_{1k}^{\mathrm{reduced}\,\prime}n(S_{1k}-S_{2k})(\bar w^{\mathrm{reduced}})\\
&\le-\frac{\lambda_{\min}(Z)}{8}\sum_j(\bar w_{1j}^{\mathrm{reduced}\,\prime}n(S_{1j}-S_{2j})(\bar w^{\mathrm{reduced}}))^2\le0.
\end{align*}
The sum is integrable and uniformly continuous, hence tends to zero.
Since $n(S_{1\mathrm G}-S_{2\mathrm G})<0<n(S_{1\mathrm A}-S_{2\mathrm A})$, connectedness of the accumulation set gives convergence to a corner.
\autoref{avoidance_of_suboptimal_stationary_points}, also with indices interchanged, excludes the two corners selecting the same expert.
Loss monotonicity excludes the remaining corner, which selects the wrong expert in both contexts.
Indeed, on the critical manifold
$$
\cL^{\mathrm{truncated}}(\bar w)=\frac1{4n}\sum_jq_j(\bar w_{1j}),
\qquad
q_j(\pi)=\pi^2\Sigma_{11}^{(j)}+2\pi(1-\pi)\Sigma_{12}^{(j)}+(1-\pi)^2\Sigma_{22}^{(j)},
$$
and $q_j''=2\,(1,-1)\bSigma^{(j)}(1,-1)^\top\ge0$ because $\bSigma^{(j)}$ is a covariance matrix.
Each $q_j$ is therefore convex on $[0,1]$ and attains its maximum at an endpoint.
As $q_j(1)=\Sigma_{11}^{(j)}$ and $q_j(0)=\Sigma_{22}^{(j)}$, which the strict hierarchy separates, the maximum is attained at the endpoint selecting the wrong expert in family $j$, and strictly so over $\bar w_{1j}\in[c,1-c]$ for any $c>0$.
Summing over $j$, the corner selecting the wrong expert in both contexts is the strict maximizer of $\cL^{\mathrm{truncated}}$ on the manifold.
Since $\bar w(\tcrit)\in[1/4,3/4]^{2\times2}$ gives bounded logits, $\bar w^{\mathrm{reduced}}(0)$ lies in $[c,1-c]^{2}$ for some absolute $c>0$, so the nonincreasing reduced flow cannot converge to that corner.
Thus $\bar w^{\mathrm{reduced}}(s)\to I_2$, and compactness, continuous dependence, and loss monotonicity give common $S,\kappa>0$ with
$$
n\cL^{\mathrm{truncated}}(\bar w^{\mathrm{reduced}}(S))
\le \min_{\bar w^\circ\text{ in }\eqref{eq:suboptimal_corners}}
n\cL^{\mathrm{truncated}}(\bar w^\circ)-4\kappa.
$$
Bounded logit derivatives give $\bar w^{\mathrm{reduced}}_{1j}(s)\in[c,1-c]$ on $[0,S]$, uniformly for some $0<c<1/2$.

Following \autoref{thm:tikhonov}, we set
\begin{align*}
y&=\left(\frac{\sigma^{-1}(\bar w_{1j})-\sigma^{-1}(\bar w_{2j})}{2}\right)_{j\in\{\mathrm G,\mathrm A\}},
\qquad
\varepsilon=\frac1{n\sigma_\mu^2},\\
y^0(s)&=\left(\sigma^{-1}(\bar w^{\mathrm{reduced}}_{1j}(s))\right)_{j\in\{\mathrm G,\mathrm A\}}.
\end{align*}
On $\bar w_{2j}=1-\bar w_{1j}$, the limiting normal equilibrium and stability bound are
\begin{align*}
r_j^\star(y)&=-\frac12\sum_{i=1}^2\bar w_{ij}(\Sigma_{1i}^{(j)}+\Sigma_{2i}^{(j)}),\\
\lambda_{\max}\left(\left[-\bar w'_{1k}Z_{jk}\bar w'_{1j}\right]_{k,j\in\{\mathrm G,\mathrm A\}}\right)
&\le-c^2(1-c)^2\lambda_{\min}(Z)<0.
\end{align*}
The affine normal equation gives, up to first exit from a bounded interior neighborhood,
$$
\left\|\left(r_j(s)+\frac12\sum_{i=1}^2\bar w_{ij}(\tcrit+ns)(\Sigma_{1i}^{(j)}+\Sigma_{2i}^{(j)})\right)_{j\in\{\mathrm G,\mathrm A\}}\right\|_2
\le Ce^{-\kappa_0n\sigma_\mu^2s}+\frac{C}{n\sigma_\mu^2}.
$$
This bounds $r$ uniformly up to first exit.
The limiting slow equation is independent of $r$, so its difference from the actual slow equation is uniformly $\cO(\varepsilon)$ on this region.
Since $y(0)=y^0(0)$, the Lipschitz bound and Gronwall's inequality give, up to first exit,
\begin{align*}
\|y(s)-y^0(s)\|_2
&\le C\int_0^s\|y(u)-y^0(u)\|_2\,du+C\varepsilon s,\\
\sup_{0\le s\le S}\|y(s)-y^0(s)\|_2
&\le C_S\varepsilon.
\end{align*}
Since $\bar w_{1j}+\bar w_{2j}-1=\varepsilon r_j$, smooth inversion of the coordinates on the interior region gives
$$
\sup_{0\le s\le S}
\|\bar w(\tcrit+ns)-\bar w^{\mathrm{reduced}}(s)\|_{\mathrm F}
\le\frac{C_S}{n\sigma_\mu^2},
\qquad \sup_{0\le s\le S}\|r(s)\|_2\le C_S.
$$
These bounds prevent exit before $S$, and
\begin{align*}
n\cL^{\mathrm{truncated}}(\bar w(\tcrit+nS))
&=n\cL^{\mathrm{truncated}}(\bar w^{\mathrm{reduced}}(S))
+\cO\left(\frac1{n\sigma_\mu^2}\right)\\
&<\min_{\bar w^\circ\text{ in }\eqref{eq:suboptimal_corners}}n\cL^{\mathrm{truncated}}(\bar w^\circ)-\kappa
\end{align*}
for sufficiently large $n$.

For the full population flow, the loss and gradient expansions give, on bounded sets of $\theta$,
\begin{align*}
|\cL-\cL^{\mathrm{truncated}}|
+\|\nabla_\theta(\cL-\cL^{\mathrm{truncated}})\|_2
&\le C\left(\frac{\sigma_\mu^2}{n}+n^{-3/2}\right),\\
\left\|\frac{d\theta_q^\perp}{ds}\right\|_2
&\le C(\sigma_\mu^2+n^{-1/2}),
\end{align*}
uniformly in $n$ and $0<\sigma_\mu^2\le1$, by \autoref{SoftmaxO1overndeviation}.
By \autoref{decomposition_of_gating_parameters}, bounded logits and perpendicular components give bounded $\theta$.
The coordinate errors are $\cO(\sigma_\mu^2+n^{-1/2})$, so normal stability and Gronwall's inequality give, up to first exit,
$$
\sup_{0\le s\le S}
\|\bar w(\tcrit+ns)-\bar w^{\mathrm{reduced}}(s)\|_{\mathrm F}
\le C_S\left(\sigma_\mu^2+n^{-1/2}+\frac1{n\sigma_\mu^2}\right),
\qquad \sup_{0\le s\le S}\|r(s)\|_2\le C_S.
$$
For sufficiently small $\sigma_\mu^2$ and then sufficiently large $n$, these bounds prevent exit before $S$, and
\begin{align*}
&\max\left\{n\cL^{\mathrm{truncated}}(\bar w(\tcrit+nS)),\ n\cL(\theta(\tcrit+nS))\right\}\\
&\quad\le n\cL^{\mathrm{truncated}}(\bar w^{\mathrm{reduced}}(S))
+C_S\left(\sigma_\mu^2+n^{-1/2}+\frac1{n\sigma_\mu^2}\right)\\
&\quad<\min_{\bar w^\circ\text{ in }\eqref{eq:suboptimal_corners}}n\cL^{\mathrm{truncated}}(\bar w^\circ)-\kappa.
\end{align*}

For the empirical flow, use the uniform concentration event in \eqref{eq:empirical_initial_gradient}.
Up to first exit from the same bounded region in $(r,y,\theta^\perp)$, bounded $r$ gives $\max_j|\bar w_{1j}+\bar w_{2j}-1|=\cO(n^{-1})$ for fixed $\sigma_\mu^2>0$, and hence
$$
n\|\nabla_\theta(\hat\cL_m-\cL)\|_2=o(1).
$$
In the rescaled equations this adds $o(1)$ to the right-hand sides of the normal equation, the slow equation, and the equation for $\theta_q^\perp$.
The normal stability estimate and Gronwall's inequality for the slow equation consequently give
\begin{gather*}
\sup_{0\le s\le S}
\|\bar w(\tcrit+ns)-\bar w^{\mathrm{reduced}}(s)\|_{\mathrm F}
\le C_S\left(\sigma_\mu^2+n^{-1/2}+\frac1{n\sigma_\mu^2}\right)+o(1),
\\
\sup_{0\le s\le S}\|r(s)\|_2\le C_S+o(1).
\end{gather*}
Also, $\sup_{0\le s\le S}\|\theta_q^\perp(\tcrit+ns)-\theta_q^\perp(\tcrit)\|_2\le C_S\sigma_\mu^2+o(1)$.
Choose the bounds on $r$ and $\theta^\perp$ strictly larger than these estimates and the logit bounds strictly larger than those of the reduced flow.
For sufficiently small fixed $\sigma_\mu^2$ and then sufficiently large $n$, the estimates prevent exit before $S$.
The preceding loss comparison acquires only an additional $o(1)$ term after multiplication by $n$, so the strict loss margin is preserved.
Bounded logits and perpendicular components give the stated parameter bound by \autoref{decomposition_of_gating_parameters}.
\end{proof}

\begin{lemma}[Exit from a neighborhood of a suboptimal stationary point]
\label{avoidance_of_suboptimal_stationary_points}
Let $Z_{ij} = \inner{\bar{\bz}_i, \bar{\bz}_j}$ and assume
\[
Z_{\mathrm{GA}}< \min\wbr{Z_{\mathrm{GG}}, Z_{\mathrm{AA}}}.
\]
In the time variable $s=(t-\tcrit)/n$, consider the reduced flow on the critical manifold $\bar{w}_{1j} + \bar{w}_{2j} = 1$ given by
\begin{align*}
\dot{\bar{w}}_{1q}
&=-\frac{1}{4}\bar{w}'_{1q}
  \sum_j Z_{jq}\bar{w}'_{1j}n(S_{1j}-S_{2j})(\bar{w}),
\end{align*}
where $S_{qj}$ is defined in \autoref{list_of_stationary_points}, so that
\begin{align*}
n(S_{1j}-S_{2j})(\bar{w}) =(\Sigma_{11}^{(j)}-\Sigma_{12}^{(j)})\bar{w}_{1j}  +(\Sigma_{12}^{(j)}-\Sigma_{22}^{(j)})\bar{w}_{2j}.
\end{align*}
Suppose the expert hierarchy of \autoref{expertHierarchy} holds and consider the suboptimal stationary point $\bar{w}^\circ_{1\mathrm G}=\bar{w}^\circ_{1\mathrm A}=1$, $\bar{w}^\circ_{2\mathrm G}=\bar{w}^\circ_{2\mathrm A}=0$.
The corner selecting the other expert in both contexts is covered by exchanging the expert indices.
There exist $\delta\in(0,1/2)$, $r_0\in(0,1)$, $p>0$ and $\gamma>0$, depending only on $Z$ and $\{\bSigma^{(j)}\}_{j\in\{\mathrm G,\mathrm A\}}$, such that the following holds.
Define
$$
R_\delta
=\left\{0<\bar{w}_{2\mathrm G},\bar{w}_{2\mathrm A}<\delta\right\}
$$
$$
W(\bar{w})=\frac{1}{\bar{w}_{2\mathrm A}}
\max\left\{1,\left(\frac{\bar{w}_{2\mathrm G}}{r_0\bar{w}_{2\mathrm A}}\right)^p\right\}.
$$
For every initialization $\bar{w}(0)\in R_\delta$, the first exit time satisfies
\[
\inf\{s\ge0:\bar{w}(s)\notin R_\delta\} \le \frac{W(\bar{w}(0))}{\gamma} < \infty.
\]

\end{lemma}

\begin{proof}
Note that $\bar{w}'_{1q}=\bar{w}'_{2q}$.
Also
$$
n(S_{1\mathrm G}-S_{2\mathrm G})(\bar{w})\le0,
\qquad
n(S_{1\mathrm A}-S_{2\mathrm A})(\bar{w}^\circ)=\Sigma_{11}^{(\mathrm A)}-\Sigma_{12}^{(\mathrm A)}>0.
$$
Thus, for sufficiently small $\delta>0$,
\[
n(S_{1\mathrm A}-S_{2\mathrm A})(\bar{w})\ge\frac12n(S_{1\mathrm A}-S_{2\mathrm A})(\bar{w}^\circ)>0
\]
for all $\bar{w} \in R_\delta$.
Choose $p>0$ such that
\begin{align*}
pZ_{\mathrm{GG}}-(p+1)Z_{\mathrm{GA}}&>0,\\
(p+1)Z_{\mathrm{AA}}-pZ_{\mathrm{GA}}&>0.
\end{align*}
Such a $p$ exists: since $Z_{\mathrm{AA}}>0$ and $Z_{\mathrm{GA}}<Z_{\mathrm{AA}}$, the second inequality holds for every $p>0$, and since $Z_{\mathrm{GA}}<Z_{\mathrm{GG}}$, the first holds for every $p>\max\{0,Z_{\mathrm{GA}}\}/(Z_{\mathrm{GG}}-Z_{\mathrm{GA}})$.
Choose $\delta$ sufficiently small that
\begin{align*}
p(1-\bar{w}_{2\mathrm G})Z_{\mathrm{GG}}-(p+1)(1-\bar{w}_{2\mathrm A})Z_{\mathrm{GA}}&>0,\\
(p+1)(1-\bar{w}_{2\mathrm A})Z_{\mathrm{AA}}-p(1-\bar{w}_{2\mathrm G})Z_{\mathrm{GA}}&>0,
\end{align*}
uniformly on $R_\delta$.

We show $\dot W\le-\gamma<0$ for $\bar{w} \in R_\delta$.

If $\bar{w}_{2\mathrm G}\le r_0\bar{w}_{2\mathrm A}$, then $W(\bar{w})=1/\bar{w}_{2\mathrm A}$ and
\begin{align*}
\dot W
&=-\frac{\dot{\bar{w}}_{2\mathrm A}}{\bar{w}_{2\mathrm A}^2}\\
&=-\frac{(1-\bar{w}_{2\mathrm A})^2}{4}
\left[
Z_{\mathrm{GA}}\frac{\bar{w}'_{1\mathrm G}}{\bar{w}'_{1\mathrm A}}n(S_{1\mathrm G}-S_{2\mathrm G})(\bar{w})
+Z_{\mathrm{AA}}n(S_{1\mathrm A}-S_{2\mathrm A})(\bar{w})
\right].
\end{align*}
Moreover,
$$
\frac{\bar{w}'_{1\mathrm G}}{\bar{w}'_{1\mathrm A}}
=\frac{\bar{w}_{2\mathrm G}(1-\bar{w}_{2\mathrm G})}{\bar{w}_{2\mathrm A}(1-\bar{w}_{2\mathrm A})}
\le\frac{r_0}{1-\delta}.
$$
Since $n(S_{1\mathrm G}-S_{2\mathrm G})$ is bounded, choose $r_0\in(0,1)$ such that
$$
\max\{Z_{\mathrm{GA}},0\}\frac{r_0}{1-\delta}
\sup_{\bar{w}\in\overline{R_\delta}}|n(S_{1\mathrm G}-S_{2\mathrm G})(\bar{w})|
\le\frac14 Z_{\mathrm{AA}}n(S_{1\mathrm A}-S_{2\mathrm A})(\bar{w}^\circ).
$$
Then
\begin{align}
\dot W
\le-\frac{(1-\delta)^2Z_{\mathrm{AA}}n(S_{1\mathrm A}-S_{2\mathrm A})(\bar{w}^\circ)}{16}<0.
\label{eq:escape_first_branch}
\end{align}
Suppose next that $\bar{w}_{2\mathrm G}\ge r_0\bar{w}_{2\mathrm A}$.
For the branch $W=\bar{w}_{2\mathrm G}^p/(r_0^p\bar{w}_{2\mathrm A}^{p+1})$, the uniform bounds above give a constant $c>0$, independent of $\bar{w}\in R_\delta$, such that
\begin{align*}
4\frac{\dot W}{W}
&= 4\left(p\frac{\dot{\bar{w}}_{2\mathrm G}}{\bar{w}_{2\mathrm G}} -(p+1)\frac{\dot{\bar{w}}_{2\mathrm A}}{\bar{w}_{2\mathrm A}}\right)
\\
&=p\frac{\bar{w}'_{1\mathrm G}}{\bar{w}_{2\mathrm G}}
\sum_j Z_{j\mathrm G}\bar{w}'_{1j}n(S_{1j}-S_{2j})(\bar{w})
-(p+1)\frac{\bar{w}'_{1\mathrm A}}{\bar{w}_{2\mathrm A}}
\sum_j Z_{j\mathrm A}\bar{w}'_{1j}n(S_{1j}-S_{2j})(\bar{w})
\\
&=p(1-\bar{w}_{2\mathrm G})
\left[Z_{\mathrm{GG}}\bar{w}'_{1\mathrm G}n(S_{1\mathrm G}-S_{2\mathrm G})(\bar{w})+Z_{\mathrm{GA}}\bar{w}'_{1\mathrm A}n(S_{1\mathrm A}-S_{2\mathrm A})(\bar{w})\right]
\\
&\quad-(p+1)(1-\bar{w}_{2\mathrm A})
\left[Z_{\mathrm{GA}}\bar{w}'_{1\mathrm G}n(S_{1\mathrm G}-S_{2\mathrm G})(\bar{w})+Z_{\mathrm{AA}}\bar{w}'_{1\mathrm A}n(S_{1\mathrm A}-S_{2\mathrm A})(\bar{w})\right]
\\
&= \left[p(1-\bar{w}_{2\mathrm G})Z_{\mathrm{GG}}
-(p+1)(1-\bar{w}_{2\mathrm A})Z_{\mathrm{GA}}\right]
\bar{w}'_{1\mathrm G}n(S_{1\mathrm G}-S_{2\mathrm G})(\bar{w})\\
&\quad+\left[p(1-\bar{w}_{2\mathrm G})Z_{\mathrm{GA}}
-(p+1)(1-\bar{w}_{2\mathrm A})Z_{\mathrm{AA}}\right]
\bar{w}'_{1\mathrm A}n(S_{1\mathrm A}-S_{2\mathrm A})(\bar{w})
\\
&\le-\left[(p+1)(1-\bar{w}_{2\mathrm A})Z_{\mathrm{AA}}-p(1-\bar{w}_{2\mathrm G})Z_{\mathrm{GA}}\right]
\bar{w}'_{1\mathrm A}n(S_{1\mathrm A}-S_{2\mathrm A})(\bar{w})
\\
&\le-\frac12\left[(p+1)(1-\bar{w}_{2\mathrm A})Z_{\mathrm{AA}}-p(1-\bar{w}_{2\mathrm G})Z_{\mathrm{GA}}\right]
\bar{w}'_{1\mathrm A}n(S_{1\mathrm A}-S_{2\mathrm A})(\bar{w}^\circ)
\\
& \le -4c\bar{w}'_{1\mathrm A}.
\end{align*}
Thus
\begin{align*}
\dot W
& \le-c\bar{w}'_{1\mathrm A}W
\\
& = - c (1-\bar{w}_{2\mathrm A})\left(\frac{\bar{w}_{2\mathrm G}}{r_0\bar{w}_{2\mathrm A}}\right)^p
\\
& \leq - c (1-\delta) < 0.
\end{align*}
Both branch derivatives satisfy these bounds at $\bar w_{2\mathrm G}=r_0\bar w_{2\mathrm A}$ as well.
Since $W$ is locally Lipschitz, $\dot W\le-\gamma$ almost everywhere along the trajectory for a common $\gamma>0$.
Apply \autoref{LyapunovExitTime} with lower bound $C=0$.
\end{proof}

\begin{lemma}[Gradient bound away from stationary points]
\label{gradient_bound_away_from_stationary_points}
Fix $\lambda$, suppose the expert hierarchy of \autoref{expertHierarchy} holds and $Z\succ0$, and let $\sigma_\mu^2,c,\rho>0$.
Use $\cS$ from \autoref{list_of_stationary_points} and define
$$
R_{c,\rho}
=\left\{\bar w\in[0,1]^{2\times2}:
\cL^{\mathrm{truncated}}(\bar w)
-\min_{\bar w^\circ\in\cS}\cL^{\mathrm{truncated}}(\bar w^\circ)\le\frac cn,
\quad
\min_{\bar w^\circ\in\cS}\|\bar w-\bar w^\circ\|_{\mathrm F}\ge\rho
\right\},
$$
$$
C_\Sigma=\max_{\substack{j\in\{\mathrm G,\mathrm A\}\\q\ne k_j}}
\left\{
\frac1{\Sigma_{k_jk_j}^{(j)}},
\frac1{\Sigma_{qk_j}^{(j)}-\Sigma_{k_jk_j}^{(j)}},
\frac1{\Sigma_{qq}^{(j)}-\Sigma_{qk_j}^{(j)}}
\right\}.
$$
If
$$
n\sigma_\mu^2\ge
\max\left\{16c+4\sum_j\Sigma_{k_jk_j}^{(j)},\ \max_{i,k,j}\Sigma_{ik}^{(j)}\right\},
$$
then every $\bar w\in R_{c,\rho}$ satisfies
$$
\|\nabla_\theta\cL^{\mathrm{truncated}}(\bar w)\|_2
\ge\frac{\sqrt{\lambda_{\min}(Z)}}{160C_\Sigma}\frac\rho n.
$$
\end{lemma}
\begin{proof}
The loss bound and $\min_{\bar w^\circ\in\cS}\cL^{\mathrm{truncated}}(\bar w^\circ)\le\cL^{\mathrm{truncated}}(I_2)$ give
$$
\left|\bar w_{1j}+\bar w_{2j}-1\right|^2
\le\frac{4c+\sum_\ell\Sigma_{\ell\ell}^{(\ell)}}{n\sigma_\mu^2}
\le\frac14.
$$
For each $j$, choose $q$ with $\bar w_{qj}\ge\bar w_{kj}$, $k\ne q$, so $\bar w_{qj}\ge1/4$ and $\bar w_{kj}\le3/4$.
The hierarchy gives $|S_{qj}-S_{kj}|\ge1/(2C_\Sigma n)$.
If $1-\bar w_{qj}\ge\bar w_{kj}/2$, then
$$
\frac{\bar w_{kj}}{2C_\Sigma n}
\le\bar w_{kj}|S_{qj}-S_{kj}|
\le12\max_i|\bar w'_{ij}S_{ij}|.
$$
Otherwise $S_{kj}\ge\Sigma_{kq}^{(j)}\bar w_{qj}/n\ge1/(4C_\Sigma n)$, so $\bar w'_{kj}S_{kj}\ge\bar w_{kj}/(16C_\Sigma n)$.
Thus in both cases
$$
\bar w_{kj}\le24C_\Sigma n\max_i|\bar w'_{ij}S_{ij}|.
$$
As a polynomial in $\bar w_{qj}$, $(1-\bar w_{qj})S_{qj}$ has roots
$$
1,
\qquad
\frac{n\sigma_\mu^2-(n\sigma_\mu^2+\Sigma_{qk}^{(j)})\bar w_{kj}}
{n\sigma_\mu^2+\Sigma_{qq}^{(j)}},
$$
separated by at least $\Sigma_{qq}^{(j)}/(n\sigma_\mu^2+\Sigma_{qq}^{(j)})$.
Using $\bar w_{qj}\ge1/4$ and $(n\sigma_\mu^2+\Sigma_{qk}^{(j)})/(n\sigma_\mu^2+\Sigma_{qq}^{(j)})\le2$ gives
$$
\min\left\{1-\bar w_{qj},
\left|\bar w_{qj}-\frac{n\sigma_\mu^2}{n\sigma_\mu^2+\Sigma_{qq}^{(j)}}\right|\right\}
\le8C_\Sigma n|\bar w'_{qj}S_{qj}|+2\bar w_{kj}.
$$
Setting $\bar w_{kj}=0$ and choosing the nearer of these two values in each family produces a stationary point, hence
$$
\rho\le\min_{\bar w^\circ\in\cS}\|\bar w-\bar w^\circ\|_{\mathrm F}
\le80C_\Sigma n\left(\sum_{i,j}(\bar w'_{ij}S_{ij})^2\right)^{1/2}
\le\frac{160C_\Sigma n}{\sqrt{\lambda_{\min}(Z)}}
\|\nabla_\theta\cL^{\mathrm{truncated}}\|_2.
$$
\end{proof}

\begin{lemma}[Population gradient perturbation]
\label{population_gradient_perturbation}
Assume the initial loss and covariance conditions of \autoref{close_convergence_to_optimal_stationary_point}, with its constant $\kappa$, at a point $\theta_0$.
Fix $B,c>0$ independently of $n$ and $\sigma_\mu^2$, and assume $\|\theta_0\|_2\le B$ and $\cL^{\mathrm{truncated}}(\bar w(\theta_0))-\cL^{\mathrm{truncated}}(\bar w^*)\le c/n$.
Here $\bar w^*$ denotes the optimum, and $\cL^{\mathrm{truncated}}$ is composed with $\bar w(\theta)$ when differentiating.
For a constant $C_0$ and $\rho=C_0/(n\sigma_\mu^2)$, set
$$
E(\theta)=\cL^{\mathrm{truncated}}(\bar w(\theta))-\cL^{\mathrm{truncated}}(\bar w^*)+\frac1{n^2\sigma_\mu^2},
\qquad
E_0=\frac{c+\kappa}n+\frac1{n^2\sigma_\mu^2},
$$
and let $\cD$ be the closed set of $\theta$ satisfying
\begin{gather*}
\cL^{\mathrm{truncated}}(\bar w(\theta))\le\cL^{\mathrm{truncated}}(\bar w(\theta_0))+\frac{\kappa}{2n},
\qquad
\min_{\bar w^\circ\in\cS}\|\bar w(\theta)-\bar w^\circ\|_{\mathrm F}\ge\rho,
\\
\|\theta\|_2\le B+1+4C_0\log\frac{E_0}{E(\theta)}.
\end{gather*}
There is such a $C_0$, independent of $n$ and $\sigma_\mu^2$, for which the following hold for sufficiently small fixed $\sigma_\mu^2>0$ and then sufficiently large $n$.
\begin{enumerate}
\item[(a)] $\cD$ is contained in the ball of radius $r=B+1+4C_0\log(1+(c+\kappa)n\sigma_\mu^2)=\cO(\log n)$, and uniformly over $\theta\in\cD$,
\begin{align*}
\|\nabla_\theta\cL^{\mathrm{truncated}}\|_2
&\ge\frac{E(\theta)}{C_0},\\
\sum_{q,j}\bar w'_{qj}
&\le C_0n\|\nabla_\theta\cL^{\mathrm{truncated}}\|_2,\\
\|\nabla_\theta(\cL-\cL^{\mathrm{truncated}})\|_2
&\le\left[C_0\sigma_\mu^2nE(\theta)(1+\|\theta\|_2)^2+\cOs(n^{-1})\right]
\|\nabla_\theta\cL^{\mathrm{truncated}}\|_2\\
&\le\frac14\|\nabla_\theta\cL^{\mathrm{truncated}}\|_2.
\end{align*}
\item[(b)] Let $\theta(t)$, $t\ge\tcrit$, be absolutely continuous with $\theta(\tcrit)=\theta_0$ and
\begin{gather*}
\dot\theta=-\nabla_\theta\cL(\theta)-\zeta(t),
\\
\|\zeta(t)\|_2\le\min\left\{\frac14,\frac1{\log n}\right\}\|\nabla_\theta\cL^{\mathrm{truncated}}(\theta(t))\|_2
\quad\text{whenever }\theta(t)\in\cD,
\end{gather*}
Let $t_\rho$ be the first time $t\ge\tcrit$ with $\min_{\bar w^\circ\in\cS}\|\bar w(\theta(t))-\bar w^\circ\|_{\mathrm F}\le\rho$.
Then $t_\rho\le\tcrit+2C_0^2n^2\sigma_\mu^2$, $\theta(t)\in\cD$ for $\tcrit\le t<t_\rho$, and for $\tcrit\le t\le t_\rho$,
\begin{align*}
\frac d{dt}\cL^{\mathrm{truncated}}(\bar w(\theta(t)))
&\le-\frac12\|\nabla_\theta\cL^{\mathrm{truncated}}\|_2^2,
\qquad
\|\dot\theta\|_2\le\frac32\|\nabla_\theta\cL^{\mathrm{truncated}}\|_2,\\
\|\theta(t)\|_2
&\le B+3C_0\log\frac{E(\theta_0)}{E(\theta(t))}
\le B+3C_0\log(1+cn\sigma_\mu^2),\\
\|\theta_q^\perp(t)-\theta_q^\perp(\tcrit)\|_2
&\le\cO(\sigma_\mu^2)+\cOs(n^{-1})+\frac{2C_0}{\log n}\log(1+cn\sigma_\mu^2).
\end{align*}
\end{enumerate}
The full population flow is the case $\zeta\equiv0$.
\end{lemma}
\begin{proof}
\emph{Part (a).}
On $\cD$, the truncated loss is $\cO(n^{-1})$, lies at least $\kappa/(2n)$ below the loss of every point in \eqref{eq:suboptimal_corners}, and has excess at most $(c+\kappa)/n$.
In particular, $1/(n^2\sigma_\mu^2)\le E(\theta)\le E_0$, which gives the radius $r$.
The excess-loss bound \eqref{eq:excess_loss_gate_bound} of \autoref{loss_barrier_near_optimal_stationary_point} gives
$$
\max_j\left[\sum_q\bar w'_{qj}
+\left|\sum_q\bar w_{qj}-1\right|\right]
\lesssim n[\cL^{\mathrm{truncated}}(\bar w)-\cL^{\mathrm{truncated}}(\bar w^*)]
+\frac1{n\sigma_\mu^2}
=nE(\theta).
$$
Minimizing the quadratic loss over the other gates shows that the loss gap below each wrong selection bounds $\bar w_{k_jj}$ and $1-\bar w_{qj}$ away from zero for $q\ne k_j$.
The same two-case calculation as in \autoref{gradient_bound_away_from_stationary_points} then gives
$$
\frac{\bar w_{qj}}n
+\sigma_\mu^2\left(\sum_i\bar w_{ij}-1\right)^2
\lesssim\|\nabla_\theta\cL^{\mathrm{truncated}}\|_2
+\frac1{n^2\sigma_\mu^2}.
$$
Expanding the quadratic loss and applying \autoref{gradient_bound_away_from_stationary_points} with $\rho=C_0/(n\sigma_\mu^2)$ proves the first estimate for sufficiently large $C_0$.
Combined with the first display, it gives the second estimate after enlarging $C_0$.
These two estimates depend only on $\bar w$, so $C_0$ is fixed before $\sigma_\mu^2$ is chosen.

Taylor-expand the conditional gradient through cubic order in the centered empirical moments.
Their signed cubic moments and fourth absolute moments are $\cO(n^{-2})$, so retaining the sigmoid derivatives gives
$$
\|\nabla_\theta(\cL-\cL^{\mathrm{truncated}})\|_2
\lesssim\sum_{q,j}\bar w'_{qj}
\left[
\frac{\sigma_\mu^2(1+\|\theta\|_2)^2}{n}
\left(\sum_i\bar w'_{ij}+\left|\sum_i\bar w_{ij}-1\right|\right)
+\frac{(1+\|\theta\|_2)^4}{n^2}
\right].
$$
This bound is uniform for $\|\theta\|_2\le r$ and $0<\sigma_\mu^2\le1$, using the moment bounds and truncation from \autoref{concentration_for_theta}.
The first display and the second estimate turn it into the first inequality of the third estimate.
On $\cD$, $1+\|\theta\|_2\le2+B+4C_0\log(E_0/E(\theta))$, so with $u=E(\theta)/E_0\in(0,1]$ and $nE_0=\cO(1)$,
$$
C_0\sigma_\mu^2nE(\theta)(1+\|\theta\|_2)^2
\le C_0\sigma_\mu^2nE_0\sup_{0<u\le1}u\left(2+B+4C_0\log\frac1u\right)^2
=\cO(\sigma_\mu^2).
$$
Choose $\sigma_\mu^2$ small and then $n$ large to make the bracket at most $1/4$.

\emph{Part (b).}
Write $\dot\theta=-\nabla_\theta\cL^{\mathrm{truncated}}-p$ with $p=\nabla_\theta(\cL-\cL^{\mathrm{truncated}})+\zeta$.
While $\theta(t)\in\cD$, part (a) gives $\|p\|_2\le\frac12\|\nabla_\theta\cL^{\mathrm{truncated}}\|_2$, hence the displayed descent and speed bounds, and, by the first estimate,
$$
\frac d{dt}\log E(\theta(t))
\le-\frac{\|\nabla_\theta\cL^{\mathrm{truncated}}\|_2^2}{2E(\theta(t))}
\le-\frac{\|\nabla_\theta\cL^{\mathrm{truncated}}\|_2}{2C_0}.
$$
Integrating,
\begin{equation}
\int_{\tcrit}^t\|\nabla_\theta\cL^{\mathrm{truncated}}\|_2\,ds
\le2C_0\log\frac{E(\theta_0)}{E(\theta(t))},
\qquad
\int_{\tcrit}^t\|\dot\theta(s)\|_2\,ds
\le3C_0\log\frac{E(\theta_0)}{E(\theta(t))}.
\label{eq:perturbed_flow_length}
\end{equation}

We first show that the trajectory cannot leave $\cD$ before $t_\rho$.
Let $T$ be the supremum of times $t\le t_\rho$ with $\theta([\tcrit,t])\subset\cD$, and suppose $T<t_\rho$.
Since $\cD$ is closed, \eqref{eq:perturbed_flow_length} holds up to $T$, so
$$
\cL^{\mathrm{truncated}}(\bar w(\theta(T)))\le\cL^{\mathrm{truncated}}(\bar w(\theta_0))
<\cL^{\mathrm{truncated}}(\bar w(\theta_0))+\frac\kappa{2n},
$$
$$
\|\theta(T)\|_2\le B+3C_0\log\frac{E(\theta_0)}{E(\theta(T))}
<B+1+4C_0\log\frac{E_0}{E(\theta(T))},
$$
using $E(\theta_0)\le E_0$, while the distance to $\cS$ exceeds $\rho$ because $T<t_\rho$.
All three conditions defining $\cD$ hold strictly at $\theta(T)$, so by continuity they hold on a neighborhood of $T$, contradicting the definition of $T$.

For finite entry, set $W=-1/E$.
On $[\tcrit,t_\rho)$,
$$
\frac{d}{dt}W(\theta(t))
=\frac{\dot E}{E^2}
\le-\frac{\|\nabla_\theta\cL^{\mathrm{truncated}}\|_2^2}{2E^2}
\le-\frac1{2C_0^2},
\qquad
-n^2\sigma_\mu^2\le W<0,
$$
so, as in \autoref{LyapunovExitTime}, $t_\rho\le\tcrit+2C_0^2n^2\sigma_\mu^2$.
The parameter bound follows from \eqref{eq:perturbed_flow_length}, $\|\theta_0\|_2\le B$, and $n^2\sigma_\mu^2E(\theta_0)\le1+cn\sigma_\mu^2\le n^2\sigma_\mu^2E_0$.

Since the truncated gradient has zero perpendicular component,
$\|\dot\theta_q^\perp\|_2\le\|\nabla_\theta(\cL-\cL^{\mathrm{truncated}})\|_2+\|\zeta\|_2$.
Along the trajectory, $E$ decreases and $\|\nabla_\theta\cL^{\mathrm{truncated}}\|_2\,ds\le-2C_0\,dE/E$.
Substituting the third estimate of part (a) and the parameter bound of $\cD$,
\begin{align*}
\int_{\tcrit}^t\|\nabla_\theta(\cL-\cL^{\mathrm{truncated}})\|_2\,ds
&\le2C_0^2\sigma_\mu^2nE_0\int_0^1\left(2+B+4C_0\log\frac1u\right)^2du\\
&\quad+\cOs(n^{-1})\log(1+cn\sigma_\mu^2)\\
&=\cO(\sigma_\mu^2)+\cOs(n^{-1}).
\end{align*}
Finally, $\int_{\tcrit}^t\|\zeta\|_2\,ds\le\frac{2C_0}{\log n}\log(1+cn\sigma_\mu^2)$ by \eqref{eq:perturbed_flow_length}.
\end{proof}

\subsubsection{Proof of \autoref{close_convergence_to_optimal_stationary_point}}
\label{pf_close_convergence_to_optimal_stationary_point}
\begin{proof}
Use $\cS$ from \autoref{list_of_stationary_points} and $C_\Sigma$ from \autoref{gradient_bound_away_from_stationary_points}.
Until entry within distance $1/n$ of $\cS$,
\begin{align*}
\|\nabla_\theta\cL^{\mathrm{truncated}}\|_2
&\ge\frac{\sqrt{\lambda_{\min}(Z)}}{160C_\Sigma n^2},\\
\frac d{dt}\cL^{\mathrm{truncated}}(\bar w(t))
&=-\|\nabla_\theta\cL^{\mathrm{truncated}}\|_2^2
\le-\frac{\lambda_{\min}(Z)}{(160C_\Sigma)^2n^4}.
\end{align*}
Since $\cL^{\mathrm{truncated}}\ge0$, entry occurs by
$$
\ttheta\le\tcrit+
\frac{(160C_\Sigma)^2n^4}{\lambda_{\min}(Z)}
\cL^{\mathrm{truncated}}(\bar w(\tcrit))
=\tcrit+\cO(n^3).
$$
For $\rho=\cO(n^{-1})$, $\rho$-neighborhoods of stationary points with both gates $0$ or $1$ in a context have
$$
\cL^{\mathrm{truncated}}(\bar w)
\ge\frac{\sigma_\mu^2}{4}-\cO(n^{-1})
>\cL^{\mathrm{truncated}}(\bar w(\tcrit)).
$$
For every $\bar w^\circ$ in \eqref{eq:suboptimal_corners} and every $\bar w$ with $\|\bar w-\bar w^\circ\|_{\mathrm F}=\cO(n^{-1})$, the quadratic expansion gives
\begin{align*}
\cL^{\mathrm{truncated}}(\bar w)
&\ge\cL^{\mathrm{truncated}}(\bar w^\circ)-\cO(n^{-2})\\
&>\cL^{\mathrm{truncated}}(\bar w(\tcrit)).
\end{align*}
Thus, by \autoref{list_of_stationary_points}, loss descent permits entry only near stationary points with
$$
\bar w^\circ_{\cdot j}\in\{e_{k_j},\bar w_{\cdot j}^*\},
\qquad
\|e_{k_j}-\bar w_{\cdot j}^*\|_2
=\frac{\Sigma_{k_jk_j}^{(j)}}{n\sigma_\mu^2+\Sigma_{k_jk_j}^{(j)}}.
$$
Adding the entry radius $1/n$ bounds $\|\bar w_{\cdot j}(\ttheta)-e_{k_j}\|_2$ as claimed.

For the full flow, take $B$ and $c$ as the assumed bounds on $\|\theta(\tcrit)\|_2$ and $n\cL^{\mathrm{truncated}}(\bar w(\tcrit))$, and apply part (b) of \autoref{population_gradient_perturbation} with $\theta_0=\theta(\tcrit)$ and $\zeta\equiv0$.
Set $\ttheta=t_\rho\le\tcrit+2C_0^2n^2\sigma_\mu^2$.
It gives $\cL^{\mathrm{truncated}}(\bar w(\ttheta))\le\cL^{\mathrm{truncated}}(\bar w(\tcrit))$ and $\|\theta(t)\|_2\le B+3C_0\log(1+cn\sigma_\mu^2)=\cO(\log n)$ for $\tcrit\le t\le\ttheta$.
Since $\rho=\cO(n^{-1})$, the exclusion above applies to a point of $\cS$ within distance $\rho$ of $\bar w(\ttheta)$, and the triangle inequality gives the full-flow gate bound.
\end{proof}

\subsubsection{Proof of \autoref{initial_selection_empirical}}
\label{pf_initial_selection_empirical}
\begin{proof}
Let $\cD$ and $r=\cO(\log n)$ be as in \autoref{population_gradient_perturbation} with $\theta_0=\theta(\tcrit)$.
The radius $r$ depends only on $B$, $c$, $\kappa$, and $C_0$, so the concentration events below do not depend on the data-dependent initial point.
Write $m^{(j)}=\#\{b:\tau_b=j\}$ and $\cL^{(j)}=\frac12\bbE_j[(\widehat\mu-\mu)^2]$, with $\hat\cL^{(j)}_{m^{(j)}}$ its empirical average over the tasks of family $j$, so that
$$
\cL=\frac12\sum_j\cL^{(j)},
\qquad
\hat\cL_m=\sum_j\frac{m^{(j)}}m\hat\cL^{(j)}_{m^{(j)}}.
$$

\emph{Concentration.}
Hoeffding's inequality gives, with failure probability at most $1/(3n)$,
$$
\max_j\left|\frac{m^{(j)}}m-\frac12\right|
=\cOs(m^{-1/2}),
$$
so $m^{(j)}\ge m/4$ for large $n$.
Conditioning on the family labels, apply \autoref{concentration_for_theta} and \autoref{concentration_for_loss_local} with radius $r$ and $\beta=\min\{\epsilon,1\}/4$ to the $m^{(j)}$ tasks of each family, each with failure probability $1/(6n)$.
The total failure probability is at most $1/n$, and all bounds below hold on this event, uniformly over the ball of radius $r$.

\emph{Relative gradient error.}
On $\cD$, since $\bSigma^{(j)}\succeq0$,
$$
\max_j\left|\sum_k\bar w_{kj}-1\right|^2
\le\frac{4}{\sigma_\mu^2}\cL^{\mathrm{truncated}}(\bar w)
\le\frac{4}{\sigma_\mu^2}\left(\cL^{\mathrm{truncated}}(\bar w(\tcrit))+\frac\kappa{2n}\right)
=\cO(n^{-1}),
$$
so the prefactor in \autoref{concentration_for_theta} is at most $2n^{-1/2+\beta}$.
Summing its coordinatewise bound over the fixed number of coordinates and weighting by $m^{(j)}/m$,
$$
\left\|\sum_j\frac{m^{(j)}}m\nabla_\theta\left(\hat\cL^{(j)}_{m^{(j)}}-\cL^{(j)}\right)\right\|_2
\le\cOs\left(\frac{n^{-1/2+\beta}}{\sqrt m}\right)\sum_{q,j}\bar w'_{qj}.
$$
The truncation bounds in the proof of \autoref{concentration_for_theta} also give
$$
\left\|
\sum_j\left(\frac{m^{(j)}}m-\frac12\right)
\nabla_\theta\cL^{(j)}
\right\|_2
\le
\cOs\left(\frac{n^{-1/2+\beta}}{\sqrt m}\right)
\sum_{q,j}\bar w'_{qj}.
$$
By the second estimate of \autoref{population_gradient_perturbation}(a), uniformly over $\theta\in\cD$,
$$
\frac{\|\nabla_\theta(\hat\cL_m-\cL)\|_2}
{\|\nabla_\theta\cL^{\mathrm{truncated}}\|_2}
\le C_0n\,\cOs\left(\frac{n^{-1/2+\beta}}{\sqrt m}\right)
=\cOs\left(\frac{n^{1/2+\beta}}{\sqrt m}\right)
\le\min\left\{\frac14,\frac1{\log n}\right\}.
$$
For the last inequality, write the preceding bound as $C(\log n)^an^{1/2+\beta}/\sqrt m$ with $C$ and $a$ fixed.
Since $\beta=\min\{\epsilon,1\}/4$ and $m\ge n^{1+\epsilon}$, it is at most $C(\log n)^an^{-\min\{\epsilon,1\}/4}$, hence at most $\min\{1/4,1/\log n\}$ for sufficiently large $n$.

\emph{Entry, parameter bound, and gate accuracy.}
The empirical flow is $\dot\theta=-\nabla_\theta\cL-\zeta(t)$ with $\zeta(t)=\nabla_\theta(\hat\cL_m-\cL)(\theta(t))$, which satisfies $\|\zeta(t)\|_2\le\min\{1/4,1/\log n\}\|\nabla_\theta\cL^{\mathrm{truncated}}\|_2$ whenever $\theta(t)\in\cD$.
Part (b) of \autoref{population_gradient_perturbation} applies.
Combined with part (a), it gives, before entry,
\begin{align*}
\|\nabla_\theta(\hat\cL_m-\cL^{\mathrm{truncated}})\|_2\le\frac12\|\nabla_\theta\cL^{\mathrm{truncated}}\|_2,
\\
\frac d{dt}\cL^{\mathrm{truncated}}\le-\frac12\|\nabla_\theta\cL^{\mathrm{truncated}}\|_2^2,
\\
\|\dot\theta\|_2\le\frac32\|\nabla_\theta\cL^{\mathrm{truncated}}\|_2.
\end{align*}
Set $\ttheta=t_\rho\le\tcrit+2C_0^2n^2\sigma_\mu^2$.
Part (b) gives loss descent and the stated parameter bound on $[\tcrit,\ttheta]$, and
$$
\|\theta_q^\perp(t)\|_2
\le B+\cO(\sigma_\mu^2)+\cOs(n^{-1})+\frac{2C_0}{\log n}\log(1+cn\sigma_\mu^2)
=\cO(1).
$$
Since $\rho=\cO(n^{-1})$ and $\cL^{\mathrm{truncated}}(\bar w(\ttheta))\le\cL^{\mathrm{truncated}}(\bar w(\tcrit))$, the exclusion in the proof of \autoref{close_convergence_to_optimal_stationary_point} applies to a point of $\cS$ within distance $\rho$ of $\bar w(\ttheta)$.
Its columns are $e_{k_j}$ or $\bar w^*_{\cdot j}$, and the triangle inequality gives the gate bound.
Since $\|e_{k_j}-\bar w^*_{\cdot j}\|_2=\Sigma_{k_jk_j}^{(j)}/(n\sigma_\mu^2+\Sigma_{k_jk_j}^{(j)})$,
$$
\|\bar w_{\cdot j}(\ttheta)-e_{k_j}\|_1
\le\sqrt2\left(\rho+\frac{2\Sigma_{k_jk_j}^{(j)}}{n\sigma_\mu^2+\Sigma_{k_jk_j}^{(j)}}\right)
=\cO(n^{-1}).
$$

\emph{Risk.}
By \autoref{SoftmaxO1overndeviation}, $\bbE_j[(g_{k_j}-\mu)^2]=\cO(n^{-1})$.
For the $1$-sub-Gaussian centered observations, $\bbE_j[\max_{\ell\in[n]}|X_\ell-\mu|^4]=\cO(\log^2n)$.
Apply \autoref{thm:accuracy_final_estimator} with $i=k_j$ and $\max_k\|\theta_k(\ttheta)\|_2=\cO(\log n)$:
$$
\sqrt{\bbE_j[(\widehat\mu-\mu)^2]}
\le\cO(n^{-1/2})+\cO(n^{-1})\left(\sigma_\mu^2+\cO(\log n)\right)^{1/2}
=\cO(n^{-1/2}).
$$
Hence $\cL^{(j)}(\theta(\ttheta))=\cO(n^{-1})$ for each $j$, and $\cL(\theta(\ttheta))=\cO(n^{-1})$.
Moreover, $|\sum_k\bar w_{kj}(\ttheta)-1|\le\|\bar w_{\cdot j}(\ttheta)-e_{k_j}\|_1=\cO(n^{-1})$ with a fixed implied constant, and $\|\theta(\ttheta)\|_2\le r$.
Hence \autoref{concentration_for_loss_local} and $m^{(j)}\ge m/4$ give
$$
\left|\hat\cL^{(j)}_{m^{(j)}}(\theta(\ttheta))-\cL^{(j)}(\theta(\ttheta))\right|
=\cOs\left(\frac{n^{-1+\beta}}{\sqrt m}+\frac{n^{-1+2\beta}}{m}\right)
=\cO(n^{-3/2}),
$$
so $\hat\cL_m(\theta(\ttheta))\le\max_j\hat\cL^{(j)}_{m^{(j)}}(\theta(\ttheta))=\cO(n^{-1})$.
All concentration bounds are uniform over the ball of radius $r$, so they apply at the data-dependent time $\ttheta$.
\end{proof}

\begin{lemma}[Loss barrier near the optimal stationary point]
\label{loss_barrier_near_optimal_stationary_point}
Fix $\lambda$ and let $\sigma_\mu^2>0$.
Suppose the expert hierarchy of \autoref{expertHierarchy} holds and use \autoref{list_of_stationary_points} and the truncated loss from \autoref{app:assembly_routing}.
For each family $j$, choose the positive threshold
$$
\delta_j
=\frac1{16n^2}
\min\left\{
\sigma_\mu^2\bigl(\Sigma_{k_jk_j}^{(j)}\bigr)^2,
\ \bar w_{k_jj}^*\Sigma_{k_jk_j}^{(j)}
\min_{k\ne k_j}\bigl(\Sigma_{k k_j}^{(j)}-\Sigma_{k_jk_j}^{(j)}\bigr)
\right\},
$$
and let $\delta=\min_{j\in\{\mathrm G,\mathrm A\}}\delta_j$.

Define the region
$$
N_\delta
=\left\{\bar w\in[0,1]^{2\times2}:
\cL^{\mathrm{truncated}}(\bar w)-\cL^{\mathrm{truncated}}(\bar w^*)\le\delta\right\}.
$$
Every point in $N_\delta$ satisfies, for each $j$,
$$
\|\bar w_{\cdot j}-\bar w_{\cdot j}^*\|_2
\le\frac{\Sigma_{k_jk_j}^{(j)}}n.
$$
Moreover, for every $\bar w\in[0,1]^{2\times2}$ and all sufficiently large $n$,
\begin{equation}
\max_j\left[\sum_q\bar w'_{qj}
+\left|\sum_q\bar w_{qj}-1\right|\right]
\lesssim n\bigl[\cL^{\mathrm{truncated}}(\bar w)-\cL^{\mathrm{truncated}}(\bar w^*)\bigr]
+\frac1{n\sigma_\mu^2},
\label{eq:excess_loss_gate_bound}
\end{equation}
with an implied constant depending only on $\bSigma^{(\mathrm G)}$ and $\bSigma^{(\mathrm A)}$.
The truncated gradient flow
$$
\dot{\bar w}_{qk}
=-\frac12\bar w'_{qk}
\sum_j Z_{jk}\bar w'_{qj}S_{qj}(\bar w),
\qquad
Z_{jk}=\inner{\bar{\bz}_j,\bar{\bz}_k},
$$
cannot exit $N_\delta$ once it enters it.

For the full population gradient flow $\dot\theta=-\nabla_\theta\mathcal L$ at fixed $\lambda$, the same region cannot be exited after time $\tcrit$ provided
$$
\cL^{\mathrm{truncated}}(\bar w(\tcrit))-\cL^{\mathrm{truncated}}(\bar w^*)\le\delta/4
$$
and, for every $\theta$ whose gates $\bar w$ lie in $N_\delta$,
$$
\left|
\mathcal L(\theta)-\mathcal L(\theta(\tcrit))
-\cL^{\mathrm{truncated}}(\bar w)+\cL^{\mathrm{truncated}}(\bar w(\tcrit))
\right|\le\delta/4.
$$
The same statement holds for the empirical gradient flow with $\mathcal L$ replaced by $\hat{\mathcal L}_m$ in both the flow and this bound.
\end{lemma}
\begin{proof}
Since $\cL^{\mathrm{truncated}}$ is quadratic and $S_{k_jj}(\bar w^*)=0$, its expansion at $\bar w^*$ is exact:
\begin{align*}
2\bigl(\cL^{\mathrm{truncated}}(\bar w)-\cL^{\mathrm{truncated}}(\bar w^*)\bigr)
=\sum_j\Bigg[
&\frac{\bar w_{k_jj}^*}{n}
\sum_{k\ne k_j}
\bigl(\Sigma_{k k_j}^{(j)}-\Sigma_{k_jk_j}^{(j)}\bigr)\bar w_{kj}
\\
&+\frac{\sigma_\mu^2}{2}
\left(\sum_{k=1}^2(\bar w_{kj}-\bar w_{kj}^*)\right)^2
\\
&+\frac1{2n}
(\bar w_{\cdot j}-\bar w_{\cdot j}^*)^{\top}
\bSigma^{(j)}
(\bar w_{\cdot j}-\bar w_{\cdot j}^*)
\Bigg].
\end{align*}
Every term on the right is nonnegative by the expert hierarchy and positive semidefiniteness of the covariance matrices.

We first deduce \eqref{eq:excess_loss_gate_bound}.
Bound each group of terms in the expansion by $2[\cL^{\mathrm{truncated}}(\bar w)-\cL^{\mathrm{truncated}}(\bar w^*)]$.
Since $\bar w_{k_jj}^*\ge1/2$ for large $n$ and the hierarchy gaps are positive constants, the first group gives
$\sum_{k\ne k_j}\bar w_{kj}\lesssim n[\cL^{\mathrm{truncated}}(\bar w)-\cL^{\mathrm{truncated}}(\bar w^*)]$.
The second gives
\begin{align*}
\Bigl|\sum_k(\bar w_{kj}-\bar w_{kj}^*)\Bigr|
&\le2\sqrt{[\cL^{\mathrm{truncated}}(\bar w)-\cL^{\mathrm{truncated}}(\bar w^*)]/\sigma_\mu^2}\\
&=2\sqrt{n[\cL^{\mathrm{truncated}}(\bar w)-\cL^{\mathrm{truncated}}(\bar w^*)]\cdot\tfrac1{n\sigma_\mu^2}}\\
&\le n[\cL^{\mathrm{truncated}}(\bar w)-\cL^{\mathrm{truncated}}(\bar w^*)]+\frac1{n\sigma_\mu^2}
\end{align*}
by the arithmetic--geometric mean inequality.
As $\sum_k\bar w_{kj}^*=\bar w_{k_jj}^*$ and
$1-\bar w_{k_jj}^*=\Sigma_{k_jk_j}^{(j)}/(n\sigma_\mu^2+\Sigma_{k_jk_j}^{(j)})\lesssim(n\sigma_\mu^2)^{-1}$,
the triangle inequality bounds $|\sum_k\bar w_{kj}-1|$ by the right-hand side of \eqref{eq:excess_loss_gate_bound}.
Combining the same two bounds controls $|\bar w_{k_jj}-\bar w_{k_jj}^*|$ and hence $1-\bar w_{k_jj}$, so
$\bar w'_{k_jj}\le1-\bar w_{k_jj}$ and $\bar w'_{qj}\le\bar w_{qj}$ for $q\ne k_j$ are bounded by the same quantity.
This proves \eqref{eq:excess_loss_gate_bound}.

For each family $j$, the bound $\delta\le\delta_j$ gives, on $N_\delta$,
$$
\sum_{k\ne k_j}\bar w_{kj}
\le
\frac{2n\delta}
{\bar w_{k_jj}^*\min_{k\ne k_j}(\Sigma_{k k_j}^{(j)}-\Sigma_{k_jk_j}^{(j)})}
\le\frac{\Sigma_{k_jk_j}^{(j)}}{8n},
$$
and
$$
\left|\sum_{k=1}^2(\bar w_{kj}-\bar w_{kj}^*)\right|
\le\sqrt{\frac{4\delta}{\sigma_\mu^2}}
\le\frac{\Sigma_{k_jk_j}^{(j)}}{2n}.
$$
Consequently,
\begin{align*}
\|\bar w_{\cdot j}-\bar w_{\cdot j}^*\|_2
&\le |\bar w_{k_jj}-\bar w_{k_jj}^*|
+\sum_{k\ne k_j}\bar w_{kj}\\
&\le\left|\sum_{k=1}^2(\bar w_{kj}-\bar w_{kj}^*)\right|
+2\sum_{k\ne k_j}\bar w_{kj}\\
&\le\sqrt{\frac{4\delta}{\sigma_\mu^2}}
+\frac{4n\delta}
{\bar w_{k_jj}^*\min_{k\ne k_j}(\Sigma_{k k_j}^{(j)}-\Sigma_{k_jk_j}^{(j)})}
\le\frac{3\Sigma_{k_jk_j}^{(j)}}{4n}.
\end{align*}

The truncated flow preserves $[0,1]^{2\times2}$ because $\bar w'_{qk}=0$ on each coordinate boundary.
Moreover, $\frac{\partial\cL^{\mathrm{truncated}}}{\partial\bar w_{qj}}=S_{qj}/2$, so
$$
\frac{d}{dt}\cL^{\mathrm{truncated}}(\bar w(t))
=-\frac1{4}\sum_{q=1}^2
\left\|\sum_j
S_{qj}(\bar w)\bar w'_{qj}\bar{\bz}_j\right\|_2^2
\le0.
$$
Since $\lambda$ is fixed, $\cL^{\mathrm{truncated}}(\bar w^*)$ and $\delta$ are constant. Hence the truncated flow cannot exit $N_\delta$.

For the full flow, a first exit would require a time $t$ with
$\cL^{\mathrm{truncated}}(\bar w(t))-\cL^{\mathrm{truncated}}(\bar w^*)=\delta$.
The assumed bound would then imply
$$
\mathcal L(\theta(t))-\mathcal L(\theta(\tcrit))
\ge\cL^{\mathrm{truncated}}(\bar w(t))-\cL^{\mathrm{truncated}}(\bar w(\tcrit))-\delta/4
\ge\delta/2>0,
$$
contradicting $d\mathcal L/dt=-\|\nabla_\theta\mathcal L\|_2^2\le0$.
Replacing $\mathcal L$ by $\hat{\mathcal L}_m$ gives the empirical statement.
\end{proof}

\subsubsection{Proof of \autoref{specialization_final_step}}
\label{pf_specialization_final_step}
\begin{proof}[Proof of \autoref{specialization_final_step}]
Since $g_1$ is the sample mean at $\lambda_1=0$, integrating the Gaussian gradient expansion in \autoref{expert_moments_gradients} gives, uniformly for $|\lambda_1|\le2$,
\[
\bbE_{\mathrm G}[(g_1-\mu)^2]
=\frac{\cosh(\lambda_1^2)+\lambda_1^2\sinh(\lambda_1^2)}n+\cO(n^{-2}).
\]
At $\tlambda$, \autoref{expert_moments_gradients}, \autoref{uniform_expert_moments_gradients}, and \autoref{mixture_expert_moments_gradients} give
$$
\bbE_{\mathrm G}[(g_1-\mu)^2]=\frac{1+o(1)}n,
\qquad
\bbE_{\mathrm A}[(g_2-\mu)^2]=\begin{cases}
\cO((n|\lambda_2|)^{-1}),&\text{Uniform},\\
\cO(\sigma_n^2/n),&\text{Gaussian mixture},
\end{cases}
$$
since $|\lambda_1(\tlambda)|=o(1)$ and $e^{-4\sqrt{1-\sigma_n^2}|\lambda_2|}=\sigma_n^2$ for the mixture.
The sample-maximum fourth moment is $\cO(1)$ for Uniform data and $\cO((\log n)^2)$ for Gaussian and Gaussian-mixture data.
Apply \autoref{thm:accuracy_final_estimator} with $i=1$ for $j=\mathrm G$ and $i=2$ for $j=\mathrm A$, and $\|\bar w_{\cdot j}-e_i\|_1=\cO(n^{-1})$.
Its proof gives $\|\widehat\mu-g_1\|_{L^2(\bbP_{\mathrm G})}=\cO(\sqrt{\log n}/n)=o(n^{-1/2})$, so the triangle inequality in both directions preserves the Gaussian leading constant.
In the Gaussian-mixture case, this discrepancy is of lower order than the expert risk, since
\[
\frac{n^{-2}\log n}{\sigma_n^2/n}
\le n^{-\epsilon}\log n=o(1).
\]
For the bounds at fixed $\mu$, the translation identities and the gate bound give
\begin{align*}
w_k(\bX+\mu\mathbf1)&=w_k(\bX),
\qquad g_k(\bX+\mu\mathbf1)=g_k(\bX)+\mu,\\
\widehat\mu(\bX+\mu\mathbf1)-\mu
&=\widehat\mu(\bX)+\mu\rbr{\sum_k w_k(\bX)-1},\\
\bbE_{\bX\sim P_{0,j}^{\otimes n}}
\sbr{\rbr{\sum_k w_k(\bX)-1}^2}
&\le\bbE_j\sbr{\|w(\bX)-e_i\|_1^2}\\
&\le C\|\bar w_{\cdot j}-e_i\|_1^2
=\cO(n^{-2}),\\
R_{j,n}(\widehat\mu)
&=\bbE_{\bX\sim P_{0,j}^{\otimes n}}[\widehat\mu(\bX)^2]\\
&\quad+\sigma_\mu^2\bbE_{\bX\sim P_{0,j}^{\otimes n}}
\sbr{\rbr{\sum_k w_k(\bX)-1}^2},\\
\bbE_{\bX\sim P_{\mu,j}^{\otimes n}}[(\widehat\mu(\bX)-\mu)^2]
&\le2\bbE_{\bX\sim P_{0,j}^{\otimes n}}[\widehat\mu(\bX)^2]\\
&\quad+2\mu^2\bbE_{\bX\sim P_{0,j}^{\otimes n}}
\sbr{\rbr{\sum_k w_k(\bX)-1}^2}\\
&\le2R_{j,n}(\widehat\mu)+\frac{C\mu^2}{n^2},\\
\sum_k w_k=1
&\quad\Longrightarrow\quad
\bbE_{\bX\sim P_{\mu,j}^{\otimes n}}[(\widehat\mu(\bX)-\mu)^2]
=R_{j,n}(\widehat\mu).
\end{align*}
\end{proof}

\subsubsection{Proof of \autoref{cor:length_transfer}}
\label{pf_length_transfer}
\begin{proof}[Proof of \autoref{cor:length_transfer}]
The population gates $\bar w_{\cdot j}$ depend on the data only through $\bar\bz_j$, which does not depend on the context length.
Hence $\max\{\|\bar w_{\cdot\mathrm G}-e_1\|_1,\|\bar w_{\cdot\mathrm A}-e_2\|_1\}=\cO(n^{-1})$ and $\max_q\|\theta_q\|_2=\cO(\log n)=\cO(\log n')$ at the trained parameters, and \autoref{thm:accuracy_final_estimator} applies with $n'$ in place of $n$.
The hypotheses of \autoref{uniform_expert_moments_gradients} hold at length $n'$, since $|\lambda_2(\tlambda)|=n^{1-\epsilon}\le n'^{1-\epsilon}$, and those of \autoref{mixture_expert_moments_gradients} hold since $n'^{-1+\epsilon}\le n^{-1+\epsilon}\le\sigma_n^2$.
The proof of \autoref{specialization_final_step} with $n'$ test observations then gives the claim, with sample-maximum fourth moment $\cO(1)$ for Uniform data and $\cO((\log n')^2)$ otherwise.
\end{proof}

\section{Proof of the softmax MoE part of \autoref{thm:softmax_dynamics}}
\label{softmax_router_convergence}

For the proof, write $\vartheta=\theta_1-\theta_2$.
The softmax MoE weights in \eqref{eq:softmax_moe} satisfy $w_1(\bX)=\sigma(\inner{\bz(\bX),\vartheta})$ and $w_2(\bX)=1-w_1(\bX)$, with the same identities at $\bar\bz_j$.
Writing the loss as a function of $\vartheta$, the original population flow satisfies $\dot\vartheta=-2\nabla_\vartheta\cL$ and $\frac{d}{dt}(\theta_1+\theta_2)=0$.
The same identities hold for the empirical loss and for $\cL^{\mathrm{sm}}$.
Since $\lambda$ is fixed in the first stage, $|\lambda_k(0)|\le1$ and $0<|\lambda_1(0)|<|\lambda_2(0)|$ make expert $1$ better in family $\mathrm G$ and expert $2$ better in family $\mathrm A$ by \autoref{expertHierarchy_normal}, \autoref{expertHierarchy_mixture} and \autoref{expertHierarchy_uniform} for sufficiently large $n$, as in \autoref{pf_glu_gating}.
The strict covariance gaps do not depend on $n$ for the Gaussian and uniform families; for the Gaussian mixture, \eqref{eq:gsr_covariance} converges as $\sigma_n\to0$ to $\sech^2(\lambda_i)\sech^2(\lambda_k)$, whose gaps are positive, so they are bounded below independently of $n$.

\begin{proof}[Proof of the softmax MoE part of \autoref{thm:softmax_dynamics}]
The gate and parameter bounds follow from the empirical part of \autoref{lem:softmax_population_convergence} with $B=1$.
Since $\theta_1+\theta_2$ is constant and bounded at initialization, the bound on $\vartheta$ gives $\max_k\|\theta_k(t)\|_2=\cO(\log n)$ throughout training of the router.
\end{proof}

\begin{lemma}[Taylor expansion of the loss]
\label{lem:softmax_loss_expansion}
Under the assumptions of \autoref{thm:softmax_dynamics}, for every fixed $C_0>0$ and all sufficiently large $n$,
\begin{align}
\left|\cL(\vartheta)-\cL^{\mathrm{sm}}(\vartheta)\right|
&\le\frac{C(1+\|\vartheta\|_2)^4}{n^{3/2}},
\label{eq:softmax_loss_remainder}\\
\left\|\nabla_\vartheta\cL-\nabla_\vartheta\cL^{\mathrm{sm}}\right\|_2
&\le\frac{C(1+\|\vartheta\|_2)^4}{n^{3/2}}\sum_j\bar w'_{1j},
\label{eq:softmax_gradient_remainder}
\end{align}
uniformly on $\|\vartheta\|_2\le C_0\log n$, where $\cL^{\mathrm{sm}}$ is defined in \eqref{eq:softmax_truncated_loss}.
\end{lemma}

\begin{proof}
$w_1+w_2=1$ gives
\[
\widehat\mu-\mu=\sum_{i=1}^2w_i(g_i-\mu).
\]
Write $\Delta\bz=\bz-\bar\bz_j$ as in \autoref{app:routing}.
The empirical-moment expansions in \autoref{SoftmaxO1overndeviation} and \autoref{expert_moments_gradients}, together with the layer-normalization argument in \autoref{app:routing}, give
\begin{align*}
\bbE_j[(g_i-\mu)(g_k-\mu)]
&=\frac{\Sigma_{ik}^{(j)}}n+\cO(n^{-3/2}),\\
\bbE_j[|g_i-\mu|^4]&=\cO(n^{-2}),
\qquad
\bbE_j[\|\Delta\bz\|_2^8]=\cO(n^{-4}),\\
\left\|\bbE_j[(g_i-\mu)(g_k-\mu)\Delta\bz]\right\|_2
&\le\left(\bbE_j[|g_i-\mu|^4]\bbE_j[|g_k-\mu|^4]\right)^{1/4}
\bbE_j[\|\Delta\bz\|_2^2]^{1/2}\\
&=\cO(n^{-3/2}).
\end{align*}
The exponential-moment bound in the proof of \autoref{thm:accuracy_final_estimator} and H\"older's inequality give
\begin{align*}
\bbE_j\!\left[e^{8|\inner{\vartheta,\Delta\bz}|}\right]&\le C,\\
\bbE_j\!\left[|g_i-\mu||g_k-\mu|\|\Delta\bz\|_2^2
 e^{2|\inner{\vartheta,\Delta\bz}|}\right]
&\le
\left(\bbE_j[|g_i-\mu|^4]\bbE_j[|g_k-\mu|^4]\right)^{1/4}\\
&\quad\cdot
\left(\bbE_j[\|\Delta\bz\|_2^8]
\bbE_j[e^{8|\inner{\vartheta,\Delta\bz}|}]\right)^{1/4}
=\cO(n^{-2}).
\end{align*}
Taylor expansion of the gate factors at $\bar\bz_j$ gives
\begin{align*}
w_iw_k
&=\bar w_{ij}\bar w_{kj}
+\left.\nabla_\bz(w_iw_k)\right|_{\bar\bz_j}^{\top}\Delta\bz
+\cO\!\left((1+\|\vartheta\|_2)^2\|\Delta\bz\|_2^2\right),\\
w_iw'_1\bz
&=\bar w_{ij}\bar w'_{1j}\bar\bz_j
+\bar w_{ij}\bar w'_{1j}\Delta\bz
+\bar\bz_j\left.\nabla_\bz(w_iw'_1)\right|_{\bar\bz_j}^{\top}\Delta\bz\\
&\quad+\bar w'_{1j}\cO\!\left((1+\|\vartheta\|_2)^2
 e^{2|\inner{\vartheta,\Delta\bz}|}\|\Delta\bz\|_2^2\right),
\end{align*}
where
\[
\left\|\left.\nabla_\bz(w_iw_k)\right|_{\bar\bz_j}\right\|_2\le C\|\vartheta\|_2,
\qquad
\left\|\left.\nabla_\bz(w_iw'_1)\right|_{\bar\bz_j}\right\|_2
\le C\bar w'_{1j}\|\vartheta\|_2.
\]
Substituting these expansions yields
\begin{align}
\nabla_\vartheta\cL
&=\frac12\sum_j\bbE_j\left[
\left(\sum_{i=1}^2w_i(g_i-\mu)\right)(g_1-g_2)w'_1\bz
\right]
\label{eq:softmax_exact_gradient}\\
&=\frac1{2n}\sum_j\bar w'_{1j}\bar\bz_j
\sum_{i=1}^2\bar w_{ij}(\Sigma_{i1}^{(j)}-\Sigma_{i2}^{(j)})
+\cO\!\left(\frac{(1+\|\vartheta\|_2)^4}{n^{3/2}}\right)\sum_j\bar w'_{1j}
\notag\\
&=\nabla_\vartheta\cL^{\mathrm{sm}}(\vartheta)
+\cO\!\left(\frac{(1+\|\vartheta\|_2)^4}{n^{3/2}}\right)\sum_j\bar w'_{1j},
\label{eq:softmax_leading_gradient}\\
\cL(\vartheta)
&=\frac14\sum_j\bbE_j\left[\left(\sum_{i=1}^2w_i(g_i-\mu)\right)^2\right]
\notag\\
&=\frac1{4n}\sum_j\sum_{i,k}
\bar w_{ij}\bar w_{kj}\Sigma_{ik}^{(j)}
+\cO\!\left(\frac{(1+\|\vartheta\|_2)^4}{n^{3/2}}\right)
\notag\\
&=\cL^{\mathrm{sm}}(\vartheta)
+\cO\!\left(\frac{(1+\|\vartheta\|_2)^4}{n^{3/2}}\right).
\notag
\end{align}
\end{proof}

\begin{lemma}[Convergence of the softmax router]
\label{lem:softmax_population_convergence}
Under the assumptions of \autoref{thm:softmax_dynamics}, consider the population flow $\dot\vartheta=-2\nabla_\vartheta\cL$ with the expert parameters fixed.
For every fixed $B>0$ and all sufficiently large $n$, there is a finite time $\ttheta=\cO(n^{B+1})$ such that
\[
\max\{\|\bar w_{\cdot\mathrm G}(\ttheta)-e_1\|_1,\|\bar w_{\cdot\mathrm A}(\ttheta)-e_2\|_1\}=\cO(n^{-B}),
\qquad
\sup_{0\le t\le\ttheta}\|\vartheta(t)\|_2=\cO(\log n).
\]
The reduced flow $\dot\vartheta=-2\nabla_\vartheta\cL^{\mathrm{sm}}$ satisfies $\bar w(t)\to I_2$ as $t\to\infty$.
Under the sample-size conditions of \autoref{thm:softmax_dynamics}, the same finite-time bounds hold for $\dot\vartheta=-2\nabla_\vartheta\hat\cL_m$ with probability at least $1-1/n$.
\end{lemma}

\begin{proof}
In time $s=t/n$, the reduced gate flow is the flow of \autoref{avoidance_of_suboptimal_stationary_points}, with time multiplied by four.
Its hypotheses hold since $Z_{\mathrm{GG}}=Z_{\mathrm{AA}}=1$ and $0\le Z_{\mathrm{GA}}<1$ by \autoref{feature_geometry}, which also bounds $\lambda_{\min}(Z)$ below by a positive constant independent of $n$.
The reduced convergence argument there, using avoidance and loss monotonicity, gives $\bar w^{\mathrm{reduced}}(s)\to I_2$ and fixed $S,\kappa>0$ such that
\begin{equation}
n\cL^{\mathrm{sm}}(\vartheta^{\mathrm{reduced}}(S))
\le\min_{\bar w^{\mathrm{suboptimal}}}
n\cL^{\mathrm{sm}}(\bar w^{\mathrm{suboptimal}})-3\kappa,
\label{eq:softmax_reduced_barrier}
\end{equation}
where the minimum is over the three incorrect corners and $\vartheta^{\mathrm{reduced}}(0)=\vartheta(0)$.

The rescaled vector field is uniformly Lipschitz on a bounded neighborhood of $\vartheta^{\mathrm{reduced}}([0,S])$, so \autoref{lem:softmax_loss_expansion} and Gronwall's inequality give
\begin{equation}
\sup_{0\le s\le S}\|\vartheta(ns)-\vartheta^{\mathrm{reduced}}(s)\|_2
\le\frac{C_S}{\sqrt n}.
\label{eq:softmax_gronwall}
\end{equation}
This prevents exit from that neighborhood for sufficiently large $n$, gives $\|\vartheta(nS)\|_2=\cO(1)$, and transfers \eqref{eq:softmax_reduced_barrier} to
\begin{equation}
n\cL^{\mathrm{sm}}(\vartheta(nS))
\le\min_{\bar w^{\mathrm{suboptimal}}}
n\cL^{\mathrm{sm}}(\bar w^{\mathrm{suboptimal}})-2\kappa.
\label{eq:softmax_full_barrier}
\end{equation}

On this sublevel set, $\bar w_{1\mathrm G}$ and $\bar w_{2\mathrm A}$ are bounded below by a positive constant, since the minimum on each face $\bar w_{1\mathrm G}=0$ and $\bar w_{2\mathrm A}=0$ is attained at an incorrect corner.
The strict hierarchy and $Z\succ0$ therefore give
\begin{align}
\|\nabla_\vartheta\cL^{\mathrm{sm}}\|_2
&\asymp\frac1n\sum_j\bar w'_{1j}
\asymp\cL^{\mathrm{sm}}-\cL^{\mathrm{sm}}(I_2),
\label{eq:softmax_gradient_excess}\\
\max\{\|\bar w_{\cdot\mathrm G}-e_1\|_1,\|\bar w_{\cdot\mathrm A}-e_2\|_1\}
&\le Cn\left[\cL^{\mathrm{sm}}-\cL^{\mathrm{sm}}(I_2)\right].
\label{eq:softmax_excess_loss}
\end{align}
By \autoref{lem:softmax_loss_expansion}, while $\|\vartheta\|_2\le C_0\log n$,
\begin{equation}
\frac{\|\nabla_\vartheta(\cL-\cL^{\mathrm{sm}})\|_2}
{\|\nabla_\vartheta\cL^{\mathrm{sm}}\|_2}
\le\frac{C(1+\|\vartheta\|_2)^4}{\sqrt n}\le\frac12
\label{eq:softmax_relative_error}
\end{equation}
for sufficiently large $n$, hence
\begin{align}
\frac d{dt}\cL^{\mathrm{sm}}
&\le-\|\nabla_\vartheta\cL^{\mathrm{sm}}\|_2^2
\le-c\left[\cL^{\mathrm{sm}}-\cL^{\mathrm{sm}}(I_2)\right]^2,
\label{eq:softmax_loss_descent}\\
\|\dot\vartheta\|_2&\le3\|\nabla_\vartheta\cL^{\mathrm{sm}}\|_2.
\notag
\end{align}
These bounds preserve \eqref{eq:softmax_full_barrier} and, by the parameter-travel calculation in \autoref{population_gradient_perturbation}, give
\begin{equation}
\|\vartheta(t)-\vartheta(nS)\|_2
\le C\log\frac{
\cL^{\mathrm{sm}}(\vartheta(nS))-\cL^{\mathrm{sm}}(I_2)}{
\cL^{\mathrm{sm}}(\vartheta(t))-\cL^{\mathrm{sm}}(I_2)}
\le CB\log n+C
\label{eq:softmax_parameter_travel}
\end{equation}
until the excess loss first reaches $n^{-B-1}$.
Choosing $C_0$ larger than this coefficient prevents exit, and integrating \eqref{eq:softmax_loss_descent} gives attainment by $\ttheta=nS+\cO(n^{B+1})$; if the target already holds at $nS$, take $\ttheta=nS$.
The gate bound follows from \eqref{eq:softmax_excess_loss}.

For empirical flow, apply \autoref{concentration_for_theta} with $\beta=\min\{\epsilon,1\}/4$ to the sigmoid parameters $(\vartheta,-\vartheta)$ and subtract the two gradient blocks.
Since $\bar w_{1j}+\bar w_{2j}=1$, the normalization-error term vanishes.
Conditioning on the family counts as in \autoref{pf_initial_selection_empirical}, and using the classwise bound $\|\nabla_\vartheta\cL^{(j)}\|_2\le C\bar w'_{1j}/n$ from \eqref{eq:softmax_exact_gradient}, gives, with probability at least $1-1/n$, uniformly on $\|\vartheta\|_2\le C_0\log n$,
\[
\|\nabla_\vartheta(\hat\cL_m-\cL)\|_2
\le\cOs\!\left(\frac{n^{-1/2+\beta}}{\sqrt m}\right)\sum_j\bar w'_{1j}.
\]
On $[0,S]$, \eqref{eq:softmax_gronwall} consequently holds with right-hand side $\cO(n^{-1/2})+\cOs(n^{1/2+\beta}/\sqrt m)=o(1)$, as in \autoref{pf_initial_selection_empirical}, which preserves \eqref{eq:softmax_full_barrier}.
On that sublevel set, \eqref{eq:softmax_gradient_excess} gives
\[
\frac{\|\nabla_\vartheta(\hat\cL_m-\cL^{\mathrm{sm}})\|_2}
{\|\nabla_\vartheta\cL^{\mathrm{sm}}\|_2}
\le\frac{C(1+\|\vartheta\|_2)^4}{\sqrt n}
+\cOs\!\left(\frac{n^{1/2+\beta}}{\sqrt m}\right)
\le\frac12.
\]
Thus \eqref{eq:softmax_loss_descent} and \eqref{eq:softmax_parameter_travel} also hold for $\dot\vartheta=-2\nabla_\vartheta\hat\cL_m$, giving the same stopping-time and parameter bounds.
\end{proof}

\section{Attention-expert dynamics}
\label{app:experts}
\subsection{Proof of \autoref{convergence_of_first_layer}}
\label{pf_convergence_of_first_layer}
\begin{proof}[Proof of \autoref{convergence_of_first_layer}]
Since $g_k$ is even in $\lambda_k$, it suffices to consider $\lambda_2(\ttheta)>0$.
For sufficiently large $n$, the uniform event in \autoref{first_layer_empirical_gradient} has probability at least $1-\eta$.
On this event, the vector field points inward at $\lambda_1=\pm2$ and $\lambda_2$ increases, so the flow stays in the stated ranges of $\lambda_1,\lambda_2$ until the first target time.
Throughout this interval,
$$
\dot\lambda_2\ge\begin{cases}
1/(C n\lambda_2^2),&\text{Uniform},\\
e^{-4\sqrt{1-\sigma_n^2}\lambda_2}/(C n),&\text{Gaussian mixture}.
\end{cases}
$$
Integration of $\lambda_2^2\,d\lambda_2$ or $e^{4\sqrt{1-\sigma_n^2}\lambda_2}\,d\lambda_2$, respectively, gives the asserted finite times.
The initial loss bound in \autoref{first_layer_empirical_gradient} and $d\hat\cL_m/dt=-\|\dot{\lambda}\|_2^2$ give
\[
|\lambda_2(\tlambda)-\lambda_2(\ttheta)|^2
\le(\tlambda-\ttheta)\int_{\ttheta}^{\tlambda}\|\dot{\lambda}(t)\|_2^2\,dt
\le\frac{C(\tlambda-\ttheta)}n.
\]
Both target values diverge, so $(\tlambda-\ttheta)/n\to\infty$.
For every fixed $\delta\in(0,1)$, the uniform $\lambda_1$ gradient expansion in \autoref{first_layer_empirical_gradient} gives $\operatorname{sgn}(\lambda_1)\partial_{\lambda_1}\hat\cL_m\ge\delta^3/n$ whenever $\delta\le|\lambda_1|\le2$, for sufficiently large $n$.
Thus $|\lambda_1|$ enters $[0,\delta]$ within $2n/\delta^3$ time after $\ttheta$ and cannot leave before $\tlambda$.
Since $\delta$ is arbitrary, $|\lambda_1(\tlambda)|=o(1)$.
\end{proof}

\subsection{Expert-parameter gradients}
Throughout this subsection the data follow \autoref{adaptiveMeanEstimation}, and in the Gaussian-mixture case we assume
$$
n^{-1+\epsilon}\le\sigma_n^2,
\qquad \sigma_n^2=o(1)
$$
for a constant $\epsilon>0$ independent of $n$.
Fix $C<\infty$, $p\ge2$ and, for Uniform data, $\epsilon\in(0,1)$ independently of $n$.
In the next three lemmas, the covariance $\Sigma_g$ is evaluated under the indicated distribution.
Constants depend on $\epsilon$, $C,p$, and the gradient inequalities hold for sufficiently large $n$.

\begin{lemma}[Expert moments and normal gradients]
\label{expert_moments_gradients}
For $q\in[2]$ and $j\in\{\mathrm G,\mathrm A\}$, uniformly for $|\lambda_q|\le2$,
$$
\|g_q-\mu\|_{L^p(\bbP_j)}+
\|\frac{\partial g_q}{\partial\lambda_q}\|_{L^p(\bbP_j)}\le C_p n^{-1/2}.
$$
Uniformly for $|\lambda_1|\le2$,
$$
\bbE_{\mathrm G}[(g_1-\mu)\frac{\partial g_1}{\partial\lambda_1}]
=\frac1{2n}\frac{d}{d\lambda_1}\Sigma_g(\lambda_1,\lambda_1)+\cO(n^{-2}).
$$
At $\lambda_1=\pm2$, this gives
$$
\operatorname{sgn}(\lambda_1)\bbE_{\mathrm G}[(g_1-\mu)\frac{\partial g_1}{\partial\lambda_1}]\ge C^{-1}/n.
$$
\end{lemma}

\begin{proof}
With $\xi_i=X_i-\mu$ from \autoref{adaptiveMeanEstimation} and the population moment ratios from \autoref{lemAsympNormality},
the linear terms for $h(\bX,\lambda)-\mu-m_j(\lambda)$ and $\frac{\partial h(\bX,\lambda)}{\partial\lambda}-m_j'(\lambda)$ are the sample averages of
\begin{equation}
\frac{e^{\lambda \xi_i}}{M_Y(\lambda)}(\xi_i-m_j(\lambda)),
\qquad
\frac{e^{\lambda \xi_i}}{M_Y(\lambda)}
\bigl((\xi_i-m_j(\lambda))^2-m_j'(\lambda)\bigr),
\label{eq:ratio-derivatives}
\end{equation}
respectively. The second expression is the derivative of the first with respect to $\lambda$. Pairing at $\pm\lambda_q$ cancels the population means, since $m_j(-\lambda)=-m_j(\lambda)$ and $m_j'(-\lambda)=m_j'(\lambda)$. The covariance of the paired linear terms is therefore $(2n)^{-1}d\Sigma_g(\lambda_q,\lambda_q)/d\lambda_q$.
For $|\lambda_q|\le2$, the moment summands have uniformly bounded fixed-order moments and population denominators bounded below.
Expansion of the ratios gives the stated $L^p$ bounds in both contexts.
Outside a fixed neighborhood of the population moments, higher moment bounds and the sample maximum control the contribution.

For Gaussian data, let
$$
V_i=\left(\frac{e^{\nu \xi_i}}{C_\nu},\frac{\xi_i e^{\nu \xi_i}}{C_\nu},\frac{\xi_i^2 e^{\nu \xi_i}}{C_\nu}\right)_{\nu=\lambda_q,-\lambda_q},
\qquad \mathbf V_n=\frac1n\sum_iV_i,\quad \mu_V=\bbE_j[V_i],
$$
where $C_\nu=e^{\nu^2/2}$. The normalized population denominators are $1$, and all fixed-order moments of $V_i$ are uniformly bounded.
Write $f(V)$ for the product $(g_q-\mu)\frac{\partial g_q}{\partial\lambda_q}$ expressed in these moments and $\Delta\mathbf V_n=\mathbf V_n-\mu_V$. Since $f(\mu_V)=0$ and $\nabla f(\mu_V)=0$, Taylor expansion gives
$$
f(\mathbf V_n)=\frac12\Delta\mathbf V_n^\top\nabla^2 f(\mu_V)\Delta\mathbf V_n
+\frac16\nabla^3 f(\mu_V)\Delta\mathbf V_n^{\otimes3}+R_{4,n}.
$$
Independence and centering give
\begin{align*}
\bbE_j\left[\frac12\Delta\mathbf V_n^\top\nabla^2 f(\mu_V)\Delta\mathbf V_n\right]
&=\frac1{2n}\frac{d}{d\lambda_q}\Sigma_g(\lambda_q,\lambda_q),\\
\bbE_j[\Delta\mathbf V_n^{\otimes3}]&=\cO(n^{-2}),
\qquad \bbE_j[\|\Delta\mathbf V_n\|_2^4]=\cO(n^{-2}).
\end{align*}
Bounded derivatives near $\mu_V$ give $\bbE_j[|R_{4,n}|]=\cO(n^{-2})$ there. Outside that neighborhood, use higher moments and
$$
|g_q-\mu|\le\max_i|\xi_i|,\qquad
|\frac{\partial g_q}{\partial\lambda_q}|\le C\max_i|\xi_i|^2.
$$

For Gaussian data, the covariance formula gives
$$
\frac1{2n}\frac{d}{d\lambda_1}\Sigma_g(\lambda_1,\lambda_1)
=\frac{\lambda_1}{n}\left[2\sinh(\lambda_1^2)+\lambda_1^2\cosh(\lambda_1^2)\right].
$$
The claimed sign at $\lambda_1=\pm2$ follows.
\end{proof}

\begin{lemma}[Uniform expert moments and gradients]
\label{uniform_expert_moments_gradients}
Uniformly for $C^{-1}\le\lambda_2\le n^{1-\epsilon}$,
\begin{align*}
\|g_2-\mu\|_{L^p(\bbP_{\mathrm U})}&\le\frac{C_p}{\sqrt{n\lambda_2}},&
\|\frac{\partial g_2}{\partial\lambda_2}\|_{L^p(\bbP_{\mathrm U})}&\le\frac{C_p}{\sqrt{n\lambda_2^3}},\\
\bbE_{\mathrm U}[(g_2-\mu)\frac{\partial g_2}{\partial\lambda_2}]
&=\frac1{2n}\frac{d}{d\lambda_2}\Sigma_g(\lambda_2,\lambda_2)
+\cO\left(\frac1{n^{3/2}\lambda_2^{3/2}}\right)
\le-\frac1{C n\lambda_2^2}.
\end{align*}
\end{lemma}

\begin{proof}
With $\xi_i=X_i-\mu$ as in the proof of \autoref{expert_moments_gradients},
\begin{align*}
\bbE_{\mathrm U}[e^{\lambda_2(\xi_i-\sqrt3)}]
&=\frac{1-e^{-2\sqrt3\lambda_2}}{2\sqrt3\lambda_2}=\Theta(\lambda_2^{-1}),\\
\left\|\frac{\frac1n\sum_i e^{\lambda_2(\xi_i-\sqrt3)}}{\bbE_{\mathrm U}[e^{\lambda_2(\xi_i-\sqrt3)}]}-1\right\|_{L^p(\bbP_{\mathrm U})}
&\le C_p\sqrt{\lambda_2/n}.
\end{align*}
The linear terms for the centered ratio and its derivative, given by \eqref{eq:ratio-derivatives}, have $L^p$ norms $\cO((n\lambda_2)^{-1/2})$ and $\cO((n\lambda_2^3)^{-1/2})$, respectively. These follow by integrating the bounded functions $e^{-\lambda_2 u}$ times powers of $u$ over $0\le u\le2\sqrt3$, and applying the moment bounds for independent averages.
Dividing by one plus the relative denominator fluctuation leaves remainders $\cO(n^{-1})$ and $\cO((n\lambda_2)^{-1})$, respectively.
For the derivative, the additional squared linear numerator has the latter order as well.
The denominator lies outside a fixed relative neighborhood of its expectation with probability at most $C\exp(-n/(C\lambda_2))$, by Bernstein's inequality. Since $n/\lambda_2\ge n^{\epsilon}$, this event contributes less than any inverse power of $n$ to the stated moments.
The product remainders are bounded by
$$
C\left(\frac1{n\sqrt{n\lambda_2^3}}
+\frac1{n\lambda_2\sqrt{n\lambda_2}}
+\frac1{n^2\lambda_2}\right)
\le\frac{C}{n^{3/2}\lambda_2^{3/2}}.
$$

The Uniform covariance and its derivative in \autoref{uniform_self_variance_monotone} satisfy
$$
-\frac{d}{d\lambda_2}\Sigma_g(\lambda_2,\lambda_2)
=\frac{\sqrt3}{4\lambda_2^2}
+\cO\left(\lambda_2^2e^{-2\sqrt3\lambda_2}\right)
\qquad (\lambda_2\to\infty).
$$
Positivity and continuity on $[C^{-1},\infty)$ therefore give
$$
\inf_{\lambda_2\ge C^{-1}}\left(-\lambda_2^2\frac{d}{d\lambda_2}\Sigma_g(\lambda_2,\lambda_2)\right)>0,
\qquad
\frac{n^{-3/2}\lambda_2^{-3/2}}{1/(n\lambda_2^2)}
=\sqrt{\lambda_2/n}=o(1).
$$
\end{proof}

\begin{lemma}[Gaussian-mixture expert moments and gradients]
\label{mixture_expert_moments_gradients}
Assume $n^{-1+\epsilon}\le\sigma_n^2$ and $\sigma_n^2=o(1)$ for a fixed $\epsilon>0$.
Uniformly for $C^{-1}\le\lambda_2\le\log(1/\sigma_n^2)/(4\sqrt{1-\sigma_n^2})$,
\begin{align*}
\bbE_{\mathrm M}[(g_2-\mu)^2]&=\frac{\Sigma_g(\lambda_2,\lambda_2)}n+\cO(n^{-2}),\\
\bbE_{\mathrm M}[(g_2-\mu)\frac{\partial g_2}{\partial\lambda_2}]
&=\frac1{2n}\frac{d}{d\lambda_2}\Sigma_g(\lambda_2,\lambda_2)+\cO(n^{-2})
\le-\frac{e^{-4\sqrt{1-\sigma_n^2}\lambda_2}}{C n},
\end{align*}
with
$$
\|g_2-\mu\|_{L^p(\bbP_{\mathrm M})}+
\|\frac{\partial g_2}{\partial\lambda_2}\|_{L^p(\bbP_{\mathrm M})}
\le C_p(1+\lambda_2)^C
\left(\frac{\sigma_n+e^{-2\sqrt{1-\sigma_n^2}\lambda_2}}{\sqrt n}+\frac1n\right).
$$
\end{lemma}

\begin{proof}
Use $V_i$ from the proof of \autoref{expert_moments_gradients}, with $j=\mathrm M$ and $q=2$ and $C_\nu=e^{\sqrt{1-\sigma_n^2}|\nu|+\sigma_n^2\nu^2/2}$.
For $\xi_i=\sqrt{1-\sigma_n^2}S_i+\sigma_n G_i$,
\begin{align*}
\bbE_{\mathrm M}\left[\frac{e^{\nu \xi_i}}{C_\nu}\right]
&=e^{-\sqrt{1-\sigma_n^2}|\nu|}\cosh(\sqrt{1-\sigma_n^2}\nu)\in[1/2,1],\\
\bbE_{\mathrm M}\left[\left(\frac{e^{\nu \xi_i}}{C_\nu}\right)^p\right]
&=e^{-p\sqrt{1-\sigma_n^2}|\nu|}\cosh(p\sqrt{1-\sigma_n^2}\nu)e^{(p^2-p)\sigma_n^2\nu^2/2}
\le e^{(p^2-p)\sigma_n^2\nu^2/2}.
\end{align*}
Including the polynomial factors by Gaussian tilting preserves boundedness because $\sigma_n\lambda_2=\cO(\sigma_n\log(1/\sigma_n^2))=o(1)$.
The Taylor expansion in the proof of \autoref{expert_moments_gradients}, applied to $(g_2-\mu)\frac{\partial g_2}{\partial\lambda_2}$ and $(g_2-\mu)^2$, gives the two covariance expansions.

For the Gaussian mixture, \eqref{eq:gsr_covariance} gives
\begin{align*}
\Sigma_g(\lambda_2,\lambda_2)
&=\sigma_n^2+(1-\sigma_n^2)\sech^4(\sqrt{1-\sigma_n^2}\lambda_2)
+\cO\bigl(\sigma_n^4(1+\lambda_2)^4\bigr),\\
\frac{d}{d\lambda_2}\Sigma_g(\lambda_2,\lambda_2)
&=-4(1-\sigma_n^2)^{3/2}\sech^4(\sqrt{1-\sigma_n^2}\lambda_2)\tanh(\sqrt{1-\sigma_n^2}\lambda_2)
+\cO\bigl(\sigma_n^4(1+\lambda_2)^5\bigr).
\end{align*}
Since $e^{-4\sqrt{1-\sigma_n^2}\lambda_2}\ge\sigma_n^2$,
\begin{align*}
\frac{\sigma_n^4(1+\lambda_2)^5}{e^{-4\sqrt{1-\sigma_n^2}\lambda_2}}
&\le\sigma_n^2(1+\lambda_2)^5\\
&\le C\sigma_n^2(1+\log(1/\sigma_n^2))^5=o(1),\\
\frac{n^{-2}}{e^{-4\sqrt{1-\sigma_n^2}\lambda_2}/n}
&\le(n\sigma_n^2)^{-1}\le n^{-\epsilon}=o(1).
\end{align*}
For the refined moment bounds, substitute $\xi_i=\sqrt{1-\sigma_n^2}S_i+\sigma_n G_i$ in the linear terms of \eqref{eq:ratio-derivatives}.
At $\sigma_n=0$ their paired values are bounded by $C e^{-2\sqrt{1-\sigma_n^2}\lambda_2}$.
The difference for $\sigma_n>0$ has $L^p$ norm at most $C_p(1+\lambda_2)^C\sigma_n$, by Gaussian moments and $\sigma_n\lambda_2=o(1)$.
Averaging the centered linear terms gives the factor $n^{-1/2}$, and the ratio remainders are $\cO(n^{-1})$ in $L^p$.
\end{proof}

\begin{lemma}[Empirical gradients with fixed gating parameters]
\label{first_layer_empirical_gradient}
Assume the data model of \autoref{adaptiveMeanEstimation} and fix $C<\infty$ and $\eta\in(0,1)$ independently of $n$.
For Uniform data, also fix $\epsilon>0$.
For Gaussian-mixture data, fix $\epsilon>0$ and assume $n^{-1+\epsilon}\le\sigma_n^2$ and $\sigma_n^2=o(1)$.
Consider $(\theta,\lambda)$ satisfying
$$
\max\{\|\bar w_{\cdot\mathrm G}-e_1\|_1,\|\bar w_{\cdot\mathrm A}-e_2\|_1\}\le C/n,
\qquad \max_q\|\theta_q\|_2\le C\log n,
\qquad |\lambda_1|\le2,
$$
with $C^{-1}\le\lambda_2\le n^{1-\epsilon}$ for Uniform data and $C^{-1}\le\lambda_2\le\log(1/\sigma_n^2)/(4\sqrt{1-\sigma_n^2})$ for the mixture.
For $m=\widetilde\Omega(n^{1+\epsilon})$ independent training tasks and sufficiently large $n$, with probability at least $1-\eta$, uniformly over this class,
$$
\frac{\partial\hat\cL_m}{\partial\lambda_2}\le\begin{cases}
-1/(C n\lambda_2^2),&\text{Uniform},\\
-e^{-4\sqrt{1-\sigma_n^2}\lambda_2}/(C n),&\text{Gaussian mixture}.
\end{cases}
$$
On the same event, uniformly over all $(\theta,\lambda)$ satisfying the bounds above,
$$
\frac{\partial\hat\cL_m}{\partial\lambda_1}
=\frac{\lambda_1}{2n}\left[2\sinh(\lambda_1^2)+\lambda_1^2\cosh(\lambda_1^2)\right]+\cO(n^{-3/2+o(1)}).
$$
On the same event, $\hat\cL_m(\theta,\lambda)\le C/n$ whenever also $\lambda_2\le1$.
Constants may depend on $\epsilon$, $C,d$, and $\boldsymbol\omega$.
\end{lemma}

\begin{proof}
For every fixed $p\ge2$, the feature moment bounds and the sigmoid inequalities in the proof of \autoref{thm:accuracy_final_estimator} give
\begin{align*}
\|1-w_1\|_{L^p(\bbP_{\mathrm G})}+\|w_2\|_{L^p(\bbP_{\mathrm G})}+\|1-w_2\|_{L^p(\bbP_{\mathrm A})}+\|w_1\|_{L^p(\bbP_{\mathrm A})}&\le C_p/n,\\
\|w_k-\bar w_{kj}\|_{L^p(\bbP_j)}&\le C_p\log n/n^{3/2}.
\end{align*}
For $q,k\in[2]$, the fixed $\boldsymbol\omega$ give
\begin{align*}
\frac{\partial w_k(\bX)}{\partial\lambda_q}
&=w_k(\bX)(1-w_k(\bX))
\left\langle\theta_k,\frac{\partial\bz(\bX)}{\partial\lambda_q}\right\rangle=0,\\
\frac{\partial h(\bX,\lambda)}{\partial\lambda}
&=\frac{\sum_i X_i^2 e^{\lambda X_i}}{\sum_i e^{\lambda X_i}}-h(\bX,\lambda)^2,\\
\frac{\partial g_q(\bX)}{\partial\lambda_q}
&=\frac12\left(
\left.\frac{\partial h(\bX,\lambda)}{\partial\lambda}\right|_{\lambda=\lambda_q}
-\left.\frac{\partial h(\bX,\lambda)}{\partial\lambda}\right|_{\lambda=-\lambda_q}
\right),\\
\frac{\partial \widehat\mu(\bX)}{\partial\lambda_q}
&=w_q(\bX)\frac{\partial g_q(\bX)}{\partial\lambda_q}.
\end{align*}
Consequently,
\begin{align*}
\frac{\partial}{\partial\lambda_q}\frac12(\widehat\mu-\mu)^2
&=(\widehat\mu-\mu)w_q\frac{\partial g_q}{\partial\lambda_q}\\
&=\left[\sum_kw_k(g_k-\mu)+\mu\left(\sum_kw_k-1\right)\right]
w_q\frac{\partial g_q}{\partial\lambda_q},\\
\frac{\partial\hat\cL_m}{\partial\lambda_q}
&=\frac1m\sum_{b=1}^m(\widehat\mu(\bX_b)-\mu_b)
w_q(\bX_b)\frac{\partial g_q(\bX_b)}{\partial\lambda_q}.
\end{align*}
Since the gates and centered experts are independent of $\mu$,
\begin{align*}
\bbE_j\left[\mu\left(\sum_kw_k-1\right)w_q\frac{\partial g_q}{\partial\lambda_q}\right]
&=\bbE[\mu]\,\bbE_j\left[\left(\sum_kw_k-1\right)w_q\frac{\partial g_q}{\partial\lambda_q}\right]=0,\\
\frac{\partial\cL}{\partial\lambda_q}
&=\frac12\sum_j\sum_k
\bbE_j\left[w_qw_k(g_k-\mu)\frac{\partial g_q}{\partial\lambda_q}\right].
\end{align*}

For $\lambda>0$ and $u\ge0$,
\begin{align*}
\sum_{i:\,\max_\ell X_\ell-X_i\ge u}\frac{e^{\lambda X_i}}{\sum_\ell e^{\lambda X_\ell}}
&\le\min\{1,ne^{-\lambda u}\},\\
0\le\frac{\partial h(\bX,\lambda)}{\partial\lambda}
&=\sum_i\frac{e^{\lambda X_i}}{\sum_\ell e^{\lambda X_\ell}}(X_i-h(\bX,\lambda))^2\\
&\le\sum_i\frac{e^{\lambda X_i}}{\sum_\ell e^{\lambda X_\ell}}(\max_\ell X_\ell-X_i)^2\\
&\le\int_0^\infty 2u\min\{1,ne^{-\lambda u}\}\,du
\le\frac{C(1+\log n)^2}{\lambda^2}.
\end{align*}
Applying the same bound to $-\bX$ gives
$$
\left|\frac{\partial g_q(\bX)}{\partial\lambda_q}\right|
\le\frac{C(1+\log n)^2}{\lambda_q^2}.
$$

For Uniform data, the gate moment bounds, H\"older's inequality, and \autoref{uniform_expert_moments_gradients} give
\begin{align*}
&\left|\bbE_{\mathrm U}\left[\frac{\partial}{\partial\lambda_2}\frac12(\widehat\mu-\mu)^2\right]
-\bbE_{\mathrm U}\left[(g_2-\mu)\frac{\partial g_2}{\partial\lambda_2}\right]\right|\\
&\quad=\left|\bbE_{\mathrm U}\left[
\left((w_2^2-1)(g_2-\mu)+w_1w_2(g_1-\mu)\right)
\frac{\partial g_2}{\partial\lambda_2}\right]\right|\\
&\quad\le\frac Cn\left(
\frac1{\sqrt{n\lambda_2}}\frac1{\sqrt{n\lambda_2^3}}
+\frac1{\sqrt n}\frac1{\sqrt{n\lambda_2^3}}\right)
=C\left(\frac1{n^2\lambda_2^2}+\frac1{n^2\lambda_2^{3/2}}\right).
\end{align*}
In family $\mathrm G$, the same bounds give
$$
\left|\sum_k\bbE_{\mathrm G}\left[w_2w_k(g_k-\mu)\frac{\partial g_2}{\partial\lambda_2}\right]\right|
\le\frac{C(1+\log n)^C}{n^{3/2}\lambda_2^2}.
$$
Consequently,
$$\frac{\partial\cL}{\partial\lambda_2}\le-\frac1{C n\lambda_2^2}.$$

For the Gaussian mixture, \autoref{mixture_expert_moments_gradients} gives
$$
\left|\bbE_{\mathrm M}\left[
\left((w_2^2-1)(g_2-\mu)+w_1w_2(g_1-\mu)\right)
\frac{\partial g_2}{\partial\lambda_2}\right]\right|
\le\frac{C(1+\log n)^C}{n^2}.
$$
In family $\mathrm G$, symmetry gives
$$
\bbE_{\mathrm G}[g_1-\mu]=\bbE_{\mathrm G}\left[\frac{\partial g_2}{\partial\lambda_2}\right]=0.
$$
Integration by parts and Cauchy--Schwarz give
\begin{align*}
\left|\bbE_{\mathrm G}\left[(g_1-\mu)\frac{\partial g_2}{\partial\lambda_2}\right]\right|
&\le
\left\|\max_i\left|\frac{\partial g_1}{\partial X_i}\right|\right\|_{L^2(\bbP_{\mathrm G})}
\left\|\sum_i\left|\frac{\partial^2 g_2}{\partial X_i\partial\lambda_2}\right|\right\|_{L^2(\bbP_{\mathrm G})}.
\end{align*}
Differentiating
$$
\frac{\partial h(\bX,\lambda)}{\partial X_i}
=\frac{e^{\lambda X_i}}{\sum_\ell e^{\lambda X_\ell}}\bigl(1+\lambda(X_i-h(\bX,\lambda))\bigr)
$$
gives
\begin{align*}
\left\|\max_i\left|\frac{\partial g_1}{\partial X_i}\right|\right\|_{L^2(\bbP_{\mathrm G})}
&\le\frac1n\left\|C(1+\max_i|X_i-\mu|)e^{C\max_i|X_i-\mu|}\right\|_{L^2(\bbP_{\mathrm G})}
=n^{-1+o(1)},\\
\left\|\sum_i\left|\frac{\partial^2 g_2}{\partial X_i\partial\lambda_2}\right|\right\|_{L^2(\bbP_{\mathrm G})}
&\le C(1+\lambda_2)\left\|(1+\max_i|X_i-\mu|)^2\right\|_{L^2(\bbP_{\mathrm G})}
\le C(1+\log n)^2.
\end{align*}
Using the gate moment bounds,
\begin{align*}
\left|\sum_k\bbE_{\mathrm G}\left[w_2w_k(g_k-\mu)\frac{\partial g_2}{\partial\lambda_2}\right]\right|
&\le\frac Cn\left|\bbE_{\mathrm G}\left[(g_1-\mu)\frac{\partial g_2}{\partial\lambda_2}\right]\right|
+\frac{C\log n}{n^{3/2}}n^{-1/2+o(1)}+n^{-2+o(1)}\\
&\le n^{-2+o(1)}.
\end{align*}
It follows that
$$
\frac{\partial\cL}{\partial\lambda_2}
\le-\frac{e^{-4\sqrt{1-\sigma_n^2}\lambda_2}}{C n}+n^{-2+o(1)}
\le-\frac{e^{-4\sqrt{1-\sigma_n^2}\lambda_2}}{2C n}.
$$
The same estimates, including the mean term, give
$$
\left\|\frac{\partial}{\partial\lambda_2}\frac12(\widehat\mu-\mu)^2\right\|_{L^p(\bbP_j)}
\le\begin{cases}
\dfrac{C_p(1+\log n)^C}{n\lambda_2^2},&\text{Uniform},\ j\in\{\mathrm G,\mathrm U\},\\[1ex]
n^{-1+o(1)}(\sigma_n^2+e^{-4\sqrt{1-\sigma_n^2}\lambda_2}),&\text{Gaussian mixture},\ j=\mathrm M,\\
n^{-3/2+o(1)},&\text{Gaussian mixture},\ j=\mathrm G.
\end{cases}
$$

For $m=\widetilde\Omega(n^{1+\epsilon})$ independent tasks and sufficiently large $n$, the moment inequality for centered independent sums and the preceding bounds give, for every fixed $\beta>0$, uniformly on any fixed polynomial-size net in $(\theta,\lambda)$, with probability at least $1-\eta/4$,
$$
\left|\frac{\partial\hat\cL_m}{\partial\lambda_2}-\frac{\partial\cL}{\partial\lambda_2}\right|
\le\frac1{\sqrt m}\begin{cases}
\dfrac{n^{\beta}}{n\lambda_2^2},&\text{Uniform},\\[1ex]
n^{-1+\beta}(\sigma_n^2+e^{-4\sqrt{1-\sigma_n^2}\lambda_2})+n^{-3/2+\beta},&\text{Gaussian mixture}.
\end{cases}
$$
For a net with at most $n^C$ points, choose $p>2(C+1)/\beta$.
Markov's inequality and a union bound give
$$
\bbP\left(\text{the preceding bound fails at some net point}\right)
\le C_p n^C n^{-p\beta/2}\le C_p/n.
$$
Since $\|\bz\|_2=1$ and $\bz$ is independent of $\lambda$,
$$
\sup_{\theta,\lambda}\left\|\nabla_{(\theta,\lambda)}
\left[(\widehat\mu-\mu)w_q\frac{\partial g_q}{\partial\lambda_q}\right]\right\|_2
\le C(1+|\mu|+\max_i|X_i-\mu|)^4.
$$
The expectation of the right-hand side is $\cO((1+\log n)^2)$, so Markov's inequality gives, for sufficiently large $n$, with probability at least $1-\eta/4$,
$$
\frac1m\sum_{b=1}^m(1+|\mu_b|+\max_{\ell\in[n]}|\bX_{b\ell}-\mu_b|)^4\le n.
$$
$(\theta,\lambda)$ has $2d+2$ coordinates with ranges of polynomial length.
A net of sufficiently fine inverse-polynomial mesh therefore extends the gradient inequalities to the full stated domain.
Choose $\beta<\min\{1/4,\epsilon/2\}$.
Relative to the respective population gradient lower bounds, the empirical errors are at most
$$
\begin{cases}
C n^{-1/2-\epsilon/2+\beta},&\text{Uniform},\\
C\left(n^{-1/2-\epsilon/2+\beta}+n^{-3\epsilon/2+\beta}\right),&\text{Gaussian mixture},
\end{cases}
$$
which tend to zero.

Uniformly for $|\lambda_1|\le2$, the gate moment bounds, \autoref{expert_moments_gradients}, and $|g_2-\mu|\le\max_i|X_i-\mu|$ give
\[
\left\|(\widehat\mu-\mu)w_1\frac{\partial g_1}{\partial\lambda_1}\right\|_{L^p(\bbP_j)}\le C_p/n.
\]
H\"older's inequality bounds the population error, while Markov's inequality and a union bound on the parameter net above bound the empirical error at the net points.
Extending between net points using the derivative bound above gives, for sufficiently large $n$, with probability at least $1-\eta/4$, uniformly over all $(\theta,\lambda)$ satisfying the bounds in \autoref{first_layer_empirical_gradient},
\[
\left|\frac{\partial\cL}{\partial\lambda_1}
-\frac12\bbE_{\mathrm G}\left[(g_1-\mu)\frac{\partial g_1}{\partial\lambda_1}\right]\right|
+\left|\frac{\partial\hat\cL_m}{\partial\lambda_1}-\frac{\partial\cL}{\partial\lambda_1}\right|
\le n^{-3/2+o(1)}.
\]
The Gaussian gradient formula in \autoref{expert_moments_gradients} proves the expansion.
When also $\lambda_2\le1$, the bounded-parameter expert moments and the gate moment bounds give $\|(\widehat\mu-\mu)^2\|_{L^p(\bbP_j)}\le C_p/n$.
Applying the same net argument to the loss gives $\hat\cL_m\le C/n$ uniformly over these parameters, for sufficiently large $n$, with probability at least $1-\eta/4$.
A union bound gives all conclusions jointly with probability at least $1-\eta$.
\end{proof}

\subsection{Geometry of \texorpdfstring{$\bz$}{z}}
\begin{lemma}[Geometry of $\bz$]
\label{feature_geometry}
Let $d\ge3$ and sample $\boldsymbol\omega$ as in \autoref{TransformerSimplified}.
Recall $m_j$, $\bar{\boldsymbol\phi}_j$, and $\bar\bz_j$ from \autoref{sec:dynamics}, with $m_j(0)=0$ by continuity, and let $Z_{jk}=\inner{\bar\bz_j,\bar\bz_k}$.
Explicitly,
\begin{align*}
m_{\mathrm G}(a)&=a,\\
m_{\mathrm A}(a)&=\sqrt3\coth(\sqrt3a)-1/a
&&\text{(uniform alternative)},\\
m_{\mathrm A}(a)&=\sigma_n^2a+\sqrt{1-\sigma_n^2}\tanh\rbr{\sqrt{1-\sigma_n^2}\,a}
&&\text{(Gaussian-mixture alternative)}.
\end{align*}
Then, almost surely over $\boldsymbol\omega$,
$$
\|P_d\bar{\boldsymbol\phi}_j\|_2>0,
\qquad
Z_{\mathrm{GG}}=Z_{\mathrm{AA}}=1,
\qquad
0<Z_{\mathrm{GA}}<1 .
$$
In particular $Z\succ0$, $Z_{\mathrm{GA}}\ge0$, and $Z_{\mathrm{GA}}<\min\wbr{Z_{\mathrm{GG}},Z_{\mathrm{AA}}}$.
For the uniform alternative these quantities do not depend on $n$.
For the Gaussian-mixture alternative they depend on $n$ only through $\sigma_n$, and if $\sigma_n\to0$ then, almost surely, $Z_{\mathrm{GA}}$ converges to a limit strictly below one.
Consequently, $\lambda_{\min}(Z)$ is bounded below by a positive constant independent of $n$, for all sufficiently large $n$.
\end{lemma}
\begin{proof}
Let $M_j(a)=\bbE_j[e^{a(X-\mu)}]$, which is finite for every $a$ because the centered observations are $1$-sub-Gaussian.
Then $m_j=(\log M_j)'$ and $m_j'(a)=\Var(Y_a)>0$, where $Y_a$ denotes the centered observation tilted by $a$.
Hence, each $m_j$ is strictly increasing on $\bbR$.
The displayed formulas are the cumulant generating functions of $\cN(0,1)$, of $\mathrm{Unif}[-\sqrt3,\sqrt3]$, and of the mixture in \autoref{gsr_moments}, differentiated once.

The $\omega_r$ are almost surely distinct, so strict monotonicity makes the entries of $\bar{\boldsymbol\phi}_j$ distinct.
In particular, they are not all equal and $P_d\bar{\boldsymbol\phi}_j\ne0$.
This proves the first claim, makes $\bar\bz_j$ well defined, and gives $Z_{jj}=\|\bar\bz_j\|_2^2=1$.

For the lower bound, write $\overline{m_j}=d^{-1}\sum_rm_j(\omega_r)$ and use the identity
\begin{align*}
\inner{P_d\bar{\boldsymbol\phi}_{\mathrm G},P_d\bar{\boldsymbol\phi}_{\mathrm A}}
&=\sum_{r=1}^d\rbr{m_{\mathrm G}(\omega_r)-\overline{m_{\mathrm G}}}\rbr{m_{\mathrm A}(\omega_r)-\overline{m_{\mathrm A}}}\\
&=\frac1{2d}\sum_{r,r'=1}^d
\rbr{m_{\mathrm G}(\omega_r)-m_{\mathrm G}(\omega_{r'})}\rbr{m_{\mathrm A}(\omega_r)-m_{\mathrm A}(\omega_{r'})}.
\end{align*}
Both $m_{\mathrm G}$ and $m_{\mathrm A}$ are strictly increasing, so every summand with $\omega_r\ne\omega_{r'}$ is strictly positive and the rest vanish.
Since the $\omega_r$ are almost surely distinct and $d\ge2$, the sum is strictly positive, whence $Z_{\mathrm{GA}}>0$.

For the upper bound, Cauchy--Schwarz gives $Z_{\mathrm{GA}}\le1$ with equality only if $P_d\bar{\boldsymbol\phi}_{\mathrm A}$ is a nonnegative multiple of $P_d\bar{\boldsymbol\phi}_{\mathrm G}$, that is, only if there exist $\alpha,\beta\in\bbR$ with
$$
m_{\mathrm A}(\omega_r)=\alpha\,\omega_r+\beta
\qquad\text{for every }r\in[d],
$$
using $m_{\mathrm G}(a)=a$.
In both cases $m_{\mathrm A}$ is real analytic on $\bbR$ and is not affine there: for the uniform alternative $m_{\mathrm U}'(a)=a^{-2}-3\csch^2(\sqrt3a)$ is nonconstant, and for the mixture $m_{\mathrm M}''(a)=-2(1-\sigma_n^2)^{3/2}\sech^2(\sqrt{1-\sigma_n^2}a)\tanh(\sqrt{1-\sigma_n^2}a)$ vanishes only at $a=0$.
Condition on $(\omega_1,\omega_2)$ with $\omega_1\ne\omega_2$, which determines $\alpha$ and $\beta$.
The remaining requirement is that $\omega_3$ be a zero of $a\mapsto m_{\mathrm A}(a)-\alpha a-\beta$, a real-analytic function that is not identically zero and therefore has at most countably many zeros in $\bbR$.
As $d\ge3$ and $\omega_3$ is independent of $(\omega_1,\omega_2)$ with a density, this event has probability zero.
Hence $Z_{\mathrm{GA}}<1$ almost surely, and $\lambda_{\min}(Z)=1-Z_{\mathrm{GA}}>0$.

Only $m_{\mathrm A}$ depends on $n$, and only through $\sigma_n$ in the Gaussian-mixture case.
If $\sigma_n\to0$ then $m_{\mathrm M}(\omega_r)\to\tanh(\omega_r)$ for each of the finitely many sampled coordinates $r\in[d]$, so $Z_{\mathrm{GA}}$ converges to the value obtained from $\tanh$, which is strictly below one by the same argument because $\tanh$ is real analytic and not affine.
Hence $Z_{\mathrm{GA}}$ is bounded away from one, and $\lambda_{\min}(Z)$ away from zero, for all sufficiently large $n$, with constants depending only on $\boldsymbol\omega$.
\end{proof}

\subsection{Expert covariance}
\begin{lemma}[Expert covariance approximation]
\label{SoftmaxO1overndeviation}
Here, $\bSigma$ is the covariance obtained from the multivariate delta method.
$Y$ is sub-Gaussian.
\begin{align}
\bbE\sbr{
n
\rbr{
\frac{\sum_{i=1}^n \exp(\lambda_k Y_i) Y_i}{\sum_{i=1}^n \exp(\lambda_k Y_i)}
-
\frac{\bbE\sbr{\exp(\lambda_k Y) Y}}{\bbE\sbr{\exp(\lambda_k Y)}}
}^2
-
\Sigma_{kk}
}
=
\cO\rbr{\frac{1}{n}}
\label{eq:expcov1}
\end{align}
\end{lemma}
\begin{proof}
Set
\[
a=\frac{\bbE[e^{\lambda_kY}Y]}{\bbE[e^{\lambda_kY}]},
\qquad B_i=e^{\lambda_kY_i},
\qquad A_i=B_i(Y_i-a),
\qquad B^*=\bbE[B_i],
\]
and write $A_n=n^{-1}\sum_iA_i$ and $B_n=n^{-1}\sum_iB_i$.
Then $\bbE[A_i]=0$ and the expression inside the expectation in
\eqref{eq:expcov1} is $n(A_n/B_n)^2$.
Taylor's theorem gives
\begin{equation}
B_n^{-2}
=(B^*)^{-2}-2(B^*)^{-3}(B_n-B^*)
+3\tilde B^{-4}(B_n-B^*)^2,
\label{eq:inverse-denominator-expansion}
\end{equation}
where $\tilde B$ lies between $B_n$ and $B^*$.
The constant term in \eqref{eq:inverse-denominator-expansion} contributes
\[
\frac{n\bbE[A_n^2]}{(B^*)^2}
=\frac{\operatorname{Var}(e^{\lambda_kY}(Y-a))}{\bbE[e^{\lambda_kY}]^2}
=\Sigma_{kk}.
\]
For the linear term, independence and centering leave only coincident indices, so
\[
-\frac{2n}{(B^*)^3}\bbE\!\left[A_n^2(B_n-B^*)\right]
=-\frac{2}{n(B^*)^3}\bbE\!\left[A_1^2(B_1-B^*)\right]
=\cO(n^{-1}).
\]
It remains to bound the quadratic remainder.
Because $\tilde B$ lies between two positive numbers,
$\tilde B^{-8}\le B_n^{-8}+(B^*)^{-8}$.
Jensen's inequality gives $B_n\ge e^{\lambda_k\bar Y_n}$, and sub-Gaussianity therefore implies
\[
\bbE[\tilde B^{-8}]
\le \bbE[e^{-8\lambda_k\bar Y_n}]+(B^*)^{-8}
=\cO(1).
\]
The variables $A_i$ and $B_i$ have moments of every fixed order, and the standard moment bound for centered i.i.d.\ averages yields
\[
\bbE[A_n^8]=\cO(n^{-4}),
\qquad
\bbE[(B_n-B^*)^8]=\cO(n^{-4}).
\]
Hölder's inequality with exponents $2,4,4$ now gives
\begin{align*}
3n\bbE\!\left[\tilde B^{-4}A_n^2(B_n-B^*)^2\right]
&\le 3n\,\bbE[\tilde B^{-8}]^{1/2}
\bbE[A_n^8]^{1/4}
\bbE[(B_n-B^*)^8]^{1/4}\\
&=\cO(n^{-1}).
\end{align*}
Combining the three terms proves \eqref{eq:expcov1}.
\end{proof}

\subsubsection{Expert hierarchy}
The goal of this section is to verify the expert hierarchy of \autoref{expertHierarchy} for Gaussian, uniform and Gaussian-mixture data.

We denote the moment generating function as
$$
M_Y(\lambda) = \bbE\sbr{\exp\rbr{\lambda Y}}.
$$
\begin{lemma}[Asymptotic normality]
\label{lemAsympNormality}
This lemma concerns the full stacked vector $(\lambda_1,-\lambda_1,\lambda_2,-\lambda_2,\omega_1,\ldots,\omega_d)\in\bbR^{4+d}$.
For readability, index this stacked vector by $\lambda_1,\ldots,\lambda_{4+d}$ within this lemma only.
Let $Y_1,\ldots,Y_n \overset{\mathrm{i.i.d.}}{\sim} Q_j$.
Recall the one-sided attention output $h$ from \eqref{eq:attention_output}.
Write $\bh(\bY)=(h(\bY,\lambda_1),\ldots,h(\bY,\lambda_{4+d}))^\top$ for these stacked head outputs.
For $\lambda=(\lambda_1,\dots,\lambda_{4+d})$,
$$\sqrt{n} \rbr{ \bh(\bY) - m_j(\lambda)} \xrightarrow{d} \mathcal{N}\rbr{0, \bSigma_h},$$
where $m_j(\lambda)$ is the vector with components $m_j(\lambda_k) = \frac{M_Y'(\lambda_k)}{M_Y(\lambda_k)}$.
The covariance matrix $\bSigma_h$ is given by:
$$\Sigma_h(\lambda_k, \lambda_l) = \frac{M_Y''(\lambda_k + \lambda_l) - (m_j(\lambda_k) + m_j(\lambda_l))M_Y'(\lambda_k + \lambda_l) + m_j(\lambda_k)m_j(\lambda_l)M_Y(\lambda_k + \lambda_l)}{M_Y(\lambda_k)M_Y(\lambda_l)}.$$
\end{lemma}

\begin{proof}
By the law of large numbers, as $n \to \infty$, the sample averages converge to expectations:
$$ h(\bY,\lambda) \xrightarrow{a.s.} \frac{\bbE[Y e^{\lambda Y}]}{\bbE[e^{\lambda Y}]} = \frac{M_Y'(\lambda)}{M_Y(\lambda)} = m_j(\lambda). $$
To find the asymptotic variance, we use the multivariate delta method. Let $V_i(\lambda) = \begin{pmatrix} Y_i e^{\lambda Y_i} \\ e^{\lambda Y_i} \end{pmatrix}$.
By the central limit theorem, $\sqrt{n} ( \bar{V} - \bbE[V] ) \xrightarrow{d} \mathcal{N}(0, \bSigma_V)$.
Define the mapping $\phi(u, v) = \frac{u}{v}$. The gradient is:
$$ \nabla \phi(u, v) = \begin{pmatrix} 1/v \\ -u/v^2 \end{pmatrix}. $$
Evaluated at the expectations $(u,v) = (M_Y'(\lambda), M_Y(\lambda))$, the gradient is:
$$ \nabla \phi = \frac{1}{M_Y(\lambda)} \begin{pmatrix} 1 \\ -m_j(\lambda) \end{pmatrix}. $$
The linear approximation (influence function) for the estimator at $\lambda$ is:
$$ \text{IF}_i(\lambda) = \nabla \phi^\top (V_i(\lambda) - \bbE[V_i(\lambda)]) = \frac{e^{\lambda Y_i}}{M_Y(\lambda)} (Y_i - m_j(\lambda)). $$
The asymptotic covariance between the estimators at $\lambda_k$ and $\lambda_l$ is the covariance of their respective influence functions:
\begin{align*}
\Sigma_h(\lambda_k, \lambda_l) &= \Cov(\text{IF}(\lambda_k), \text{IF}(\lambda_l)) = \mathbb{E}[\text{IF}(\lambda_k) \text{IF}(\lambda_l)] \\
&= \mathbb{E}\left[ \rbr{ \frac{e^{\lambda_k Y}}{M_Y(\lambda_k)} (Y - m_j(\lambda_k)) } \rbr{ \frac{e^{\lambda_l Y}}{M_Y(\lambda_l)} (Y - m_j(\lambda_l)) } \right]
\end{align*}
Expanding yields
\begin{align*}
\Sigma_h(\lambda_k, \lambda_l) &= \frac{1}{M_Y(\lambda_k)M_Y(\lambda_l)} \mathbb{E}\left[ e^{(\lambda_k + \lambda_l)Y} (Y - m_j(\lambda_k))(Y - m_j(\lambda_l)) \right] \\
&= \frac{1}{M_Y(\lambda_k)M_Y(\lambda_l)} \mathbb{E}\left[ e^{(\lambda_k + \lambda_l)Y} (Y^2 - Y(m_j(\lambda_k)+m_j(\lambda_l)) + m_j(\lambda_k)m_j(\lambda_l)) \right]
\end{align*}
Using the property that $\mathbb{E}[Y^k e^{tY}] = M_Y^{(k)}(t)$, we obtain the final expression.
\end{proof}

\subsubsection{Gaussian expert covariance}
\begin{lemma}[Gaussian distribution]
If $Y \sim \cN(0, 1)$, then $M_Y(t) = e^{t^2/2}$,
$$m_{\mathrm G}(\lambda) = \lambda,$$
$$\Sigma_h(\lambda_k, \lambda_l) = (1 + \lambda_k \lambda_l) e^{\lambda_k \lambda_l}.$$
\end{lemma}

\begin{proof}
Let $Y \sim \cN(0, 1)$. The moment generating function is $M_Y(t) = e^{t^2/2}$.
The first two derivatives are:
\begin{align*}
M_Y'(t) &= t e^{t^2/2} = t M_Y(t), \\
M_Y''(t) &= (1 + t^2) e^{t^2/2} = (1 + t^2) M_Y(t).
\end{align*}
The mean function is $m_{\mathrm G}(\lambda) = M_Y'(\lambda)/M_Y(\lambda) = \lambda$.
We apply the general covariance formula from \autoref{lemAsympNormality}:
$$\Sigma_h(\lambda_k, \lambda_l) = \frac{M_Y''(\lambda_k + \lambda_l) - (m_{\mathrm G}(\lambda_k) + m_{\mathrm G}(\lambda_l))M_Y'(\lambda_k + \lambda_l) + m_{\mathrm G}(\lambda_k)m_{\mathrm G}(\lambda_l)M_Y(\lambda_k + \lambda_l)}{M_Y(\lambda_k)M_Y(\lambda_l)}. $$
Substituting the derivatives and $m_{\mathrm G}(\lambda) = \lambda$:
\begin{align*}
& M_Y''(\lambda_k + \lambda_l) - (m_{\mathrm G}(\lambda_k) + m_{\mathrm G}(\lambda_l))M_Y'(\lambda_k + \lambda_l) + m_{\mathrm G}(\lambda_k)m_{\mathrm G}(\lambda_l)M_Y(\lambda_k + \lambda_l) \\
= & (1 + (\lambda_k + \lambda_l)^2)M_Y(\lambda_k + \lambda_l) - (\lambda_k + \lambda_l)^2 M_Y(\lambda_k + \lambda_l) + \lambda_k \lambda_l M_Y(\lambda_k + \lambda_l) \\
= & M_Y(\lambda_k + \lambda_l) \left[ 1 + (\lambda_k + \lambda_l)^2 - (\lambda_k + \lambda_l)^2 + \lambda_k \lambda_l \right]\\
= & M_Y(\lambda_k + \lambda_l) (1 + \lambda_k \lambda_l).
\end{align*}
Dividing by the denominator $M_Y(\lambda_k)M_Y(\lambda_l) = e^{\lambda_k^2/2} e^{\lambda_l^2/2}$:
\begin{align*}
\Sigma_h(\lambda_k, \lambda_l) &= (1 + \lambda_k \lambda_l) \frac{e^{(\lambda_k + \lambda_l)^2/2}}{e^{(\lambda_k^2 + \lambda_l^2)/2}} \\
&= (1 + \lambda_k \lambda_l) \exp\left( \frac{\lambda_k^2 + 2\lambda_k\lambda_l + \lambda_l^2 - \lambda_k^2 - \lambda_l^2}{2} \right) \\
&= (1 + \lambda_k \lambda_l) e^{\lambda_k \lambda_l}.
\end{align*}
\end{proof}

\begin{definition}[Symmetrized expert covariance]
\label{sym_expert_covariance}
Recall $g_k(\bX) = \frac12\rbr{h(\bX,\lambda_k) + h(\bX,-\lambda_k)}$.
For $a,b\in\bbR$, the asymptotic covariance of $h(\cdot,a)$ and $h(\cdot,-a)$'s symmetrization against the analogous pair at $b$ is denoted $\Sigma_g(a,b)$.
Since each $Q_j$ is symmetric about the origin, it follows from the linearity of the influence functions that:
\begin{equation}
\Sigma_g(a, b) = \frac{1}{2} \left( \Sigma_h(a, b) + \Sigma_h(a, -b) \right)
\end{equation}
where $\Sigma_h(\cdot, \cdot)$ is the one-sided covariance defined in \autoref{lemAsympNormality}.
\end{definition}

\subsubsection{Uniform expert covariance}
\label{uniform_expert_covariance}
\begin{lemma}[Uniform distribution]
\label{lem:uniform_dist}
If $Y\sim\mathrm{Unif}[-\sqrt3,\sqrt3]$, then
\[
m_{\mathrm U}(\lambda)=\sqrt3\coth(\sqrt3\lambda)-\frac1\lambda,
\]
and
\begin{align*}
\Sigma_h(\lambda_k,\lambda_l)
={}&\frac{2\sqrt3\lambda_k\lambda_l
\{\coth(\sqrt3\lambda_k)+\coth(\sqrt3\lambda_l)\}}
{(\lambda_k+\lambda_l)^3}\\
&-\frac{3\{\lambda_k^2\csch^2(\sqrt3\lambda_k)
+\lambda_l^2\csch^2(\sqrt3\lambda_l)\}}
{(\lambda_k+\lambda_l)^2}.
\end{align*}
\end{lemma}
\begin{proof}
Put $a=\sqrt3$.
The moment generating function and its tilted mean are
\[
M_Y(\lambda)=\frac{\sinh(a\lambda)}{a\lambda},
\qquad
m_{\mathrm U}(\lambda)=\frac{M_Y'(\lambda)}{M_Y(\lambda)}
=a\coth(a\lambda)-\frac1\lambda,
\]
with $m_{\mathrm U}'(\lambda)=-a^2\csch^2(a\lambda)+\lambda^{-2}$.
The general covariance identity from \autoref{lemAsympNormality} can be written
\begin{align*}
\Sigma_h(\lambda_k,\lambda_l)
={}&\frac{M_Y(\lambda_k+\lambda_l)}{M_Y(\lambda_k)M_Y(\lambda_l)}
\Bigl[m_{\mathrm U}'(\lambda_k+\lambda_l)\\[-2pt]
&\quad+\{m_{\mathrm U}(\lambda_k+\lambda_l)-m_{\mathrm U}(\lambda_k)\}
       \{m_{\mathrm U}(\lambda_k+\lambda_l)-m_{\mathrm U}(\lambda_l)\}\Bigr].
\end{align*}
Substitute the displayed formulas for $M_Y,m_{\mathrm U},m_{\mathrm U}'$ and use
\[
\frac{\sinh(a(\lambda_k+\lambda_l))}
{\sinh(a\lambda_k)\sinh(a\lambda_l)}
=\coth(a\lambda_k)+\coth(a\lambda_l).
\]
Collecting the rational and hyperbolic terms gives exactly the stated expression.
\end{proof}

\begin{lemma}[Monotonicity for uniform distribution]
\label{uniform_self_variance_monotone}
Let $Y\sim\mathrm{Unif}[-\sqrt3,\sqrt3]$.
Then $\Sigma_g(\lambda,\lambda)$ is strictly decreasing in $|\lambda|$.
\end{lemma}
\begin{proof}
For $t=\sqrt3\lambda>0$, \autoref{lem:uniform_dist} and symmetry give
\[
\Sigma_g(t/\sqrt3,t/\sqrt3)
=\frac34\left(\frac{\coth t}{t}-\csch^2t\right)
+\frac12t^2\csch^2t
\left[1-3\left(\coth t-\frac1t\right)^2\right].
\]
\begin{align*}
&-64t^2\sinh^5t\,\frac{d}{dt}\Sigma_g(t/\sqrt3,t/\sqrt3)\\
&\quad=3\cosh(5t)-(9+120t^2+32t^4)\cosh(3t)\\
&\qquad\quad +(6+120t^2-352t^4)\cosh t
+(12t+128t^3)\sinh(3t)\\
&\qquad\quad +(-36t+384t^3)\sinh t
=\sum_{m\ge0}A_mt^m.
\end{align*}
\[
A_0=A_2=A_4=A_6=0,
\qquad A_{2k+1}=0\quad(k\ge0),
\]
\[
\begin{array}{c|cccccc}
m&8&10&12&14&16&18\\ \hline \\[-5pt]
A_m&\dfrac{256}{5}&\dfrac{3712}{315}&\dfrac{6368}{4725}
&\dfrac{3952}{31185}&\dfrac{539606}{42567525}&\dfrac{1979}{1658475}
\end{array}
\]
For even $m\ge20$, dropping positive terms and bounding falling factorials by powers gives
\[
\frac{m!A_m}{5^m}
\ge3-\left(\frac35\right)^m
\left(9+\frac{40}{3}m^2+\frac{32}{81}m^4\right)
-\frac{352m^4+36m}{5^m}.
\]
Each subtracted term decreases for $m\ge20$, since its ratio at consecutive integers is at most
\[
\frac35\left(\frac{m+1}{m}\right)^4
\le\frac35\left(\frac{21}{20}\right)^4
=\frac{583443}{800000}<1.
\]
Thus
\[
\frac{m!A_m}{5^m}
\ge3-\frac{239026832372329}{95367431640625}>\frac13,
\qquad
\frac{d}{dt}\Sigma_g(t/\sqrt3,t/\sqrt3)<0\quad(t>0).
\]
Evenness and continuity at zero give the claim.
\end{proof}

\begin{lemma}[Symmetrized cross-covariance for uniform distribution]
Let $Y\sim\mathrm{Unif}[-\sqrt3,\sqrt3]$ and $\lambda_1,\lambda_2\in\bbR$.
Then, with the right-hand side interpreted by continuity when $\lambda_1\lambda_2=0$ or $\lambda_1=\pm\lambda_2$,
\begin{align*}
\Sigma_g(\lambda_1,\lambda_2)
={}&\sqrt3\lambda_1\lambda_2
\left[
\frac{\coth(\sqrt3\lambda_1)+\coth(\sqrt3\lambda_2)}
     {(\lambda_1+\lambda_2)^3}
-\frac{\coth(\sqrt3\lambda_1)-\coth(\sqrt3\lambda_2)}
     {(\lambda_1-\lambda_2)^3}
\right]\\
&-\frac32
\left[\lambda_1^2\csch^2(\sqrt3\lambda_1)
     +\lambda_2^2\csch^2(\sqrt3\lambda_2)\right]\\[-2pt]
&\hspace{18mm}\times
\left[\frac1{(\lambda_1+\lambda_2)^2}
     +\frac1{(\lambda_1-\lambda_2)^2}\right].
\end{align*}
\end{lemma}
\begin{proof}
By symmetry,
$\Sigma_g(\lambda_1,\lambda_2)
=\tfrac12\{\Sigma_h(\lambda_1,\lambda_2)
+\Sigma_h(\lambda_1,-\lambda_2)\}$.
Substitute the formula in \autoref{lem:uniform_dist} and use
$\coth(-x)=-\coth(x)$ and $\csch^2(-x)=\csch^2(x)$.
Collecting the common terms gives the displayed expression.
Continuity of the covariance formula in \autoref{lemAsympNormality} gives the stated extensions.
For $\lambda_1=\pm\lambda_2$, evenness gives the self-variance $\Sigma_g(\lambda_1,\lambda_1)$.
\end{proof}

\begin{corollary}[Expert hierarchy for the uniform distribution]
\label{gsr_uniform_hierarchy_bound}
For the Uniform distribution and $\lambda_i \in [-1,1]$, $\Sigma_g$ satisfies the expert hierarchy of \autoref{expertHierarchy}.
\end{corollary}
\begin{proof}
$\Sigma_g$ is even in each argument, so it suffices to treat $0\le\lambda_1<\lambda_2\le1$.
Write
\[
M_Y(\lambda)=\frac{\sinh(\sqrt3\lambda)}{\sqrt3\lambda},
\qquad
\chi_\lambda(x)=M_Y(\lambda)x\cosh(\lambda x)-M_Y'(\lambda)\sinh(\lambda x),
\]
with continuous extensions at zero.
For $u>0$,
\begin{align*}
\frac{d}{du}(u\coth u)
&=\frac{\sinh(2u)-2u}{2\sinh^2u}>0,\\
\frac{2u}{3}-\frac{d}{du}(u\coth u)
&=\frac{2u\cosh(2u)+4u-3\sinh(2u)}{6\sinh^2u}\\
&=\frac1{6\sinh^2u}
\sum_{k\ge2}\frac{2^{2k+1}(2k-2)}{(2k+1)!}u^{2k+1}>0.
\end{align*}
Hence
\[
\lim_{t\downarrow0}t\coth t=1,
\qquad
0<s\coth s-t\coth t<\frac{s^2-t^2}{3}
\quad(0\le t<s).
\]
For $0<x<\sqrt3$ and $0<\lambda\le1$,
\begin{align*}
\frac{\chi_\lambda(x)}{M_Y(\lambda)^2}
&=\frac{\sinh(\lambda x)}{\lambda M_Y(\lambda)}
\left[1-\sqrt3\lambda\coth(\sqrt3\lambda)+\lambda x\coth(\lambda x)\right]\\
&>\frac{\sinh(\lambda x)}{\lambda M_Y(\lambda)}(1-\lambda^2)\ge0,\\
\frac{\partial}{\partial\lambda}\frac{\chi_\lambda(x)}{M_Y(\lambda)^2}
&=\frac{\sinh(\lambda x)}{\lambda^2M_Y(\lambda)}
\Bigl[\lambda^2(x^2-3)
+2\{\sqrt3\lambda\coth(\sqrt3\lambda)-1\}\\
&\hspace{35mm}\times\{\sqrt3\lambda\coth(\sqrt3\lambda)-\lambda x\coth(\lambda x)\}\Bigr]\\
&\le-\frac{(3-x^2)\sinh(\lambda x)}{M_Y(\lambda)}
\left(1-\frac{2\lambda^2}{3}\right)<0.
\end{align*}
Using $\chi_0(x)=x$, oddness in $x$, and \autoref{lemAsympNormality},
\begin{align*}
\Sigma_g(\lambda_2,\lambda_2)
&=\frac1{\sqrt3}\int_0^{\sqrt3}\frac{\chi_{\lambda_2}(x)^2}{M_Y(\lambda_2)^4}\,dx\\
&<\frac1{\sqrt3}\int_0^{\sqrt3}
\frac{\chi_{\lambda_2}(x)\chi_{\lambda_1}(x)}{M_Y(\lambda_2)^2M_Y(\lambda_1)^2}\,dx
=\Sigma_g(\lambda_1,\lambda_2)\\
&<\frac1{\sqrt3}\int_0^{\sqrt3}\frac{\chi_{\lambda_1}(x)^2}{M_Y(\lambda_1)^4}\,dx
=\Sigma_g(\lambda_1,\lambda_1).
\end{align*}
\end{proof}

Taking $\lambda_1=\lambda_2$ above also gives $\frac{d}{d\lambda}\Sigma_g(\lambda,\lambda)<0$ on $(0,1)$.
\autoref{uniform_self_variance_monotone} is needed for the larger $\lambda$ reached in the second stage, where the hierarchy itself fails and no bounded-range argument applies.

\subsubsection{Gaussian-mixture expert covariance}
\label{gsr_expert_covariance}
We now give the analogous proof for the Gaussian mixture in
\autoref{adaptiveMeanEstimation}.
Write $a_n=\sqrt{1-\sigma_n^2}$.
The centered observation has the representation
$Y\overset{D}{=}a_nS+\sigma_nG$, where $S$ is Rademacher, $G$ is standard
Gaussian, and $S,G$ are independent.
Thus $Y$ has mean zero and variance one.

\begin{lemma}[Mixture moments]
\label{gsr_moments}
\begin{align*}
M_Y(t)&=e^{\sigma_n^2t^2/2}\cosh(a_nt),\\
m_{\mathrm M}(t)&=\sigma_n^2t+a_n\tanh(a_nt),\\
m_{\mathrm M}'(t)&=\sigma_n^2+a_n^2\sech^2(a_nt).
\end{align*}
Moreover, $Y$ is $1$-sub-Gaussian.
\end{lemma}
\begin{proof}
Independence of $S$ and $G$ gives
\[
M_Y(t)=\bbE[e^{ta_nS}]\bbE[e^{t\sigma_nG}]
=\cosh(a_nt)e^{\sigma_n^2t^2/2}.
\]
Differentiating $\log M_Y$ twice gives
$m_{\mathrm M}(t)=M_Y'(t)/M_Y(t)$ and
$m_{\mathrm M}'(t)=M_Y''(t)/M_Y(t)-m_{\mathrm M}(t)^2$.
Finally, $\cosh(a_nt)\le e^{a_n^2t^2/2}$ gives
$M_Y(t)\le e^{t^2/2}$.
\end{proof}

\begin{lemma}[Gaussian mixture expert covariance]
\label{gsr_covariance}
\begin{align*}
\Sigma_h(\lambda_k,\lambda_l)
={}&\frac{M_Y(\lambda_k+\lambda_l)}{M_Y(\lambda_k)M_Y(\lambda_l)}
\Bigl[m_{\mathrm M}'(\lambda_k+\lambda_l)\\
&+(m_{\mathrm M}(\lambda_k+\lambda_l)-m_{\mathrm M}(\lambda_k))
  (m_{\mathrm M}(\lambda_k+\lambda_l)-m_{\mathrm M}(\lambda_l))\Bigr].
\end{align*}
Writing
$\Sigma_g(\lambda_k,\lambda_l)=\tfrac12\{\Sigma_h(\lambda_k,\lambda_l)+
\Sigma_h(\lambda_k,-\lambda_l)\}$,
\begin{align}
\Sigma_g(\lambda_k,\lambda_l)
={}&\cosh(\sigma_n^2\lambda_k\lambda_l)
\Bigl[
 a_n^2\sech^2(a_n\lambda_k)\sech^2(a_n\lambda_l)+\sigma_n^2
 \nonumber\\[-2pt]
&\hspace{37mm}
 +\sigma_n^4\lambda_k\lambda_l\tanh(a_n\lambda_k)\tanh(a_n\lambda_l)
\Bigr]
\nonumber\\
&+\sinh(\sigma_n^2\lambda_k\lambda_l)
\Bigl[
 \sigma_n^2\tanh(a_n\lambda_k)\tanh(a_n\lambda_l)+\sigma_n^4\lambda_k\lambda_l
 \nonumber\\[-2pt]
&\hspace{15mm}
 +\sigma_n^2a_n\{\lambda_k\sech^2(a_n\lambda_k)\tanh(a_n\lambda_l)
 +\lambda_l\sech^2(a_n\lambda_l)\tanh(a_n\lambda_k)\}
\Bigr].
\label{eq:gsr_covariance}
\end{align}
In particular,
$\Sigma_g(\lambda_k,0)=\sigma_n^2+a_n^2\sech^2(a_n\lambda_k)$ and
$\Sigma_g(0,0)=1$.
\end{lemma}
\begin{proof}
The formula for $\Sigma_h$ is \autoref{lemAsympNormality} rewritten using
\autoref{gsr_moments}, exactly as in the uniform case.
Exponentially tilting $a_nS+\sigma_nG$ by $\lambda_k+\lambda_l$ shifts the Gaussian
mean by $\sigma_n^2(\lambda_k+\lambda_l)$ and tilts the Rademacher probabilities.
For $S\in\{-1,1\}$,
\begin{equation*}
\frac{e^{a_n\lambda_k S}}{\cosh(a_n\lambda_k)}
\bigl(a_nS-a_n\tanh(a_n\lambda_k)\bigr)
=a_n\sech^2(a_n\lambda_k)S.
\end{equation*}
Substitution gives
\begin{align*}
\Sigma_h(\lambda_k,\lambda_l)
=e^{\sigma_n^2\lambda_k\lambda_l}\Bigl[{}
&a_n^2\sech^2(a_n\lambda_k)\sech^2(a_n\lambda_l)\\
&+(\sigma_n^2+\sigma_n^4\lambda_k\lambda_l)
  \{1+\tanh(a_n\lambda_k)\tanh(a_n\lambda_l)\}\\
&+\sigma_n^2a_n
  \{\lambda_k\sech^2(a_n\lambda_k)\tanh(a_n\lambda_l)
   +\lambda_l\sech^2(a_n\lambda_l)\tanh(a_n\lambda_k)\}
\Bigr].
\end{align*}
Averaging this expression with its value at $-\lambda_l$ and using the definitions
of $\cosh$ and $\sinh$ gives \eqref{eq:gsr_covariance}.
\end{proof}

\begin{lemma}[Expert hierarchy for the Gaussian mixture]
\label{gsr_hierarchy}
For every fixed $0<\delta<1$ there exists $\sigma_\delta>0$ such that, for all
$0<\sigma_n<\sigma_\delta<1/\sqrt2$ and
\[
0\le|\lambda_1|<|\lambda_2|\le L_{\max}=(1-\delta)\log(1/\sigma_n),
\]
we have
\[
\Sigma_g(\lambda_2,\lambda_2)<\Sigma_g(\lambda_2,\lambda_1)<\Sigma_g(\lambda_1,\lambda_1).
\]
Quantitatively,
\[
\min\{\Sigma_g(\lambda_2,\lambda_1)-\Sigma_g(\lambda_2,\lambda_2),
       \Sigma_g(\lambda_1,\lambda_1)-\Sigma_g(\lambda_2,\lambda_1)\}
\ge \frac{e^{-4L_{\max}}}{8(1+L_{\max})}(\lambda_2^2-\lambda_1^2).
\]
Thus the expert hierarchy of \autoref{expertHierarchy} holds throughout $[-L_{\max},L_{\max}]$, and
$L_{\max}\to\infty$ as $\sigma_n\downarrow0$ for any fixed $\delta$.
\end{lemma}
\begin{proof}
Write $u=\lambda_2^2$, $v=\lambda_1^2$, and separate
\[
\Sigma_g(\sqrt u,\sqrt v)
=a_n^2\sech^2(a_n\sqrt u)\sech^2(a_n\sqrt v)
 +\sigma_n^2+R_{\sigma_n}(u,v).
\]
For $L\ge1$ and $0\le u,v\le L^2$, expanding
\eqref{eq:gsr_covariance}, using
$\cosh t-1\le t^2e^t/2$ and $1\le\sinh(t)/t\le e^t$, and bounding the
resulting derivatives gives a universal $C$ such that
\[
\left|\partial_uR_{\sigma_n}(u,v)\right|
+\left|\partial_vR_{\sigma_n}(u,v)\right|
\le C\sigma_n^4(1+L)^8e^{\sigma_n^2L^2}.
\]
For the dominant term, $a_n^4\ge1/4$, $\sech t\ge e^{-t}$, and
$\tanh(t)/t\ge1/(1+t)$ imply
\begin{align*}
-&\frac{\partial}{\partial u}
 \bigl\{a_n^2\sech^2(a_n\sqrt u)\sech^2(a_n\sqrt v)\bigr\}\\
&=a_n^4\frac{\tanh(a_n\sqrt u)}{a_n\sqrt u}
  \sech^2(a_n\sqrt u)\sech^2(a_n\sqrt v)\\
&\ge\frac{e^{-4L}}{4(1+L)}.
\end{align*}
Hence
$\partial_u\Sigma_g(\sqrt u,\sqrt v)\le-e^{-4L}/\{8(1+L)\}$,
and symmetrically in $v$, whenever
\[
8C\sigma_n^4(1+L)^9e^{4L+\sigma_n^2L^2}\le1.
\]
This condition holds at
$L=L_{\max}=(1-\delta)\log(1/\sigma_n)$ for sufficiently small $\sigma_n$,
since $\sigma_n^4e^{4L_{\max}}=\sigma_n^{4\delta}\to0$.
Integrating the derivative bounds in one coordinate between $\lambda_1^2$ and
$\lambda_2^2$ gives both quantitative gaps and hence the strict hierarchy.
\end{proof}

\subsection{Concentration bounds}
\subsubsection{Softmax concentration}
\begin{lemma}[Concentration of softmax]
\label{SMconcentration}
Let $Y_1,\ldots,Y_n$ be independent copies of a mean-zero $1$-sub-Gaussian variable $Y$, and fix $\lambda\in\bbR$ and an integer $k\ge1$.
There are constants $C(\lambda),C_2(\lambda)>0$, depending only on $\lambda,k$, such that, for $0<M\le1$ and $R\ge1$ with $\bbP(|Y|>R)\le M^2C_2(\lambda)$,
\begin{align*}
\bbP\left(
\left|\frac{\sumin e^{\lambda Y_i}Y_i^k}{\sumin e^{\lambda Y_i}}
-\frac{\bbE[Y^ke^{\lambda Y}]}{\bbE[e^{\lambda Y}]}\right|
\ge M C(\lambda)\right)
\le 2n e^{-R^2/2}
+4\exp\left(-\frac{nM^2}{32R^{2k}e^{2R|\lambda|}}\right).
\end{align*}
\end{lemma}
\begin{proof}
Set $A_i=Y_i^ke^{\lambda Y_i}\mathbf1_{\{|Y_i|\le R\}}$ and $B_i=e^{\lambda Y_i}\mathbf1_{\{|Y_i|\le R\}}$, and write $A_n=n^{-1}\sum_iA_i$, $B_n=n^{-1}\sum_iB_i$, $A^*=\bbE[Y^ke^{\lambda Y}]$, and $B^*=\bbE[e^{\lambda Y}]\ge1$.
The $A_i$ are independent, as are the $B_i$, and satisfy $|A_i|\le R^ke^{R|\lambda|}$ and $0\le B_i\le e^{R|\lambda|}$.
The second moments of $Y^ke^{\lambda Y}$ and $e^{\lambda Y}$ are bounded in terms of $\lambda,k$ by the sub-Gaussian assumption.
Cauchy--Schwarz and a sufficiently small $C_2(\lambda)$ give $|\bbE A_n-A^*|,|\bbE B_n-B^*|\le M/4$.
Hoeffding's inequality therefore gives $|A_n-A^*|,|B_n-B^*|\le M/2$, except with probability at most $4\exp(-nM^2/(32R^{2k}e^{2R|\lambda|}))$.
On this event, $B_n\ge1/2$, so
\begin{equation}
\left|\frac{A_n}{B_n}-\frac{A^*}{B^*}\right|
\le\frac{|A_n-A^*|}{B_n}
+\frac{|A^*||B_n-B^*|}{B_nB^*}
\le M(1+|A^*|).
\label{FractionDeviationBound}
\end{equation}
The sub-Gaussian moment bound allows $C(\lambda)>1+|A^*|$ to be chosen depending only on $\lambda,k$.
Finally, $\bbP(\max_i|Y_i|>R)\le2ne^{-R^2/2}$, and outside this event the truncated and original ratios agree.
\end{proof}

\section{Additional results}
\label{app:additional}
\begin{lemma}[Lyapunov exit time]
\label{LyapunovExitTime}
Consider $\dot x=f(x)$ and a region $R$.
Suppose $W$ is locally Lipschitz, $W(x)>C$ on $R$, and $\frac d{dt}W(x(t))\le-\gamma<0$ almost everywhere while $x(t)\in R$.
If $x(0)\in R$ and the solution exists up to exit, then its first exit time is at most
$$t^*=\frac{W(x(0))-C}{\gamma}.$$
Moreover, $W(x(t))\le W(x(0))$ until exit.
\end{lemma}
\begin{proof}
While the trajectory remains in $R$,
$$C<W(x(t))\le W(x(0))-\gamma t.$$
At $t=t^*$ these inequalities contradict each other.
\end{proof}

\begin{theorem}[Tikhonov's theorem]
\label{thm:tikhonov}
Consider
$$
\varepsilon\dot r=f(r,y),
\qquad \dot y=g(r,y),
\qquad (r(0),y(0))=(r_0,y_0).
$$
Let $\mathcal N_0$ be open and $r^\star$ be a twice continuously differentiable function on $\mathcal N_0$ with $f(r^\star(y),y)=0$ and graph contained in an open set $\mathcal N$.
Assume that, for every $(r,y)\in\mathcal N$ with $y\in\mathcal N_0$, the line segment from $(r^\star(y),y)$ to $(r,y)$ lies in $\mathcal N$.
Assume that $f,g$ and their derivatives through order two are uniformly bounded in $\mathcal N$, and that all eigenvalues of $\partial_r f(r^\star(y),y)$ have real parts at most $-a_0<0$, uniformly for $y\in\mathcal N_0$.
Let
$$
\dot y^0=g(r^\star(y^0),y^0),
\qquad y^0(0)=y_0,
$$
and let $L_g>0$ be a Lipschitz constant for $y\mapsto g(r^\star(y),y)$ on $\mathcal N_0$.
There exist $\varepsilon_0,c_0,\ldots,c_4,\kappa_0>0$ such that, for $0<\varepsilon<\varepsilon_0$ and $(r_0,y_0)\in\mathcal N$ with $\|r_0-r^\star(y_0)\|_2<c_0$,
\begin{align*}
\|r(s)-r^\star(y(s))\|_2
&\le c_1\varepsilon+c_2\|r_0-r^\star(y_0)\|_2e^{-\kappa_0s/\varepsilon},\\
\|y(s)-y^0(s)\|_2
&\le c_3\varepsilon e^{L_gs}+c_4\|r_0-r^\star(y_0)\|_2e^{-\kappa_0s/\varepsilon},
\end{align*}
for all $s\ge0$ such that $(r(u),y(u))\in\mathcal N$ and $y(u),y^0(u)\in\mathcal N_0$ for $0\le u\le s$.
\end{theorem}
\begin{proof}
See \citet[Theorem 4.1.2]{berglund2001perturbationtheorydynamicalsystems}.
\end{proof}

\section{Lower bound for the Gaussian mixture}
\label{app:lowerbound}
\subsection{Gaussian and uniform benchmarks}
\label{app:benchmark_proof}
\begin{proposition}[Gaussian and uniform benchmarks]
\label{prop:location_benchmarks}
For the variance-one Gaussian location family, $R_{\mathrm G,n}^{\star}=n^{-1}$, attained by the sample mean.
For the uniform location family, $R_{\mathrm U,n}^{\star}\asymp n^{-2}$, and the sample mid-range has risk
\[
 \bbE_{\mu,\mathrm U}\left[\left(\frac{X_{(1)}+X_{(n)}}2-\mu\right)^2\right]
 =\frac{6}{(n+1)(n+2)}.
\]
The sample mean instead has risk $n^{-1}$ under either family.
\end{proposition}
\begin{proof}
See \citep[Theorem 3.5]{lehmann1998theory} or \citep[Section 2]{shao2003mathematical}.
\end{proof}

\subsection{Gaussian-mixture benchmark}

\begin{lemma}[Fisher information and lower bounds for the location]
\label{gsr_fisher}
For known $\sigma_n\in(0,1)$, the observations have density
\begin{equation}
 f_\mu(x)
 =\frac1{\sqrt{2\pi}\sigma_n}
   \exp\rbr{-\frac{(x-\mu)^2+1-\sigma_n^2}{2\sigma_n^2}}
   \cosh\rbr{\frac{\sqrt{1-\sigma_n^2}(x-\mu)}{\sigma_n^2}},
 \label{eq:gsr_density}
\end{equation}
and the Fisher information in one observation is
\begin{equation}
 I_{\sigma_n}
 =\frac1{\sigma_n^2}
 -\frac{1-\sigma_n^2}{\sigma_n^4}
   \bbE\sbr{\sech^2\!\rbr{\frac{\sqrt{1-\sigma_n^2}(X-\mu)}{\sigma_n^2}}\mid\mu}.
 \label{eq:gsr_information}
\end{equation}
As $\sigma_n\downarrow0$,
\begin{equation}
 I_{\sigma_n}
 =\frac1{\sigma_n^2}\left[
 1-\sqrt{\frac\pi2}\frac{\sqrt{1-\sigma_n^2}}{\sigma_n}
       e^{-(1-\sigma_n^2)/(2\sigma_n^2)}\rbr{1+\cO(\sigma_n^2)}
 \right].
 \label{eq:gsr_information_asymptotic}
\end{equation}
By the Cram\'er--Rao inequality, every unbiased estimator of $\mu$ based on $n$ i.i.d.\ observations satisfies
\begin{equation}
 \Var_\mu(\widehat\mu)\ge\frac1{nI_{\sigma_n}}
 \sim\frac{\sigma_n^2}{n}
 \qquad\text{as }\sigma_n\downarrow0.
 \label{eq:gsr_crlb}
\end{equation}
\[
R_{\mathrm M,n}^{\star}\ge\frac{\sigma_n^2}{n}.
\]
\end{lemma}
\begin{proof}
Differentiating \eqref{eq:gsr_density} with respect to $\mu$ gives
\begin{align*}
 \frac{\partial}{\partial \mu}\log f_\mu(x)
 &=\frac{x-\mu-\sqrt{1-\sigma_n^2}\tanh(\sqrt{1-\sigma_n^2}(x-\mu)/\sigma_n^2)}{\sigma_n^2},\\
 \frac{\partial^2}{\partial \mu^2}\log f_\mu(x)
 &=-\frac1{\sigma_n^2}
   +\frac{1-\sigma_n^2}{\sigma_n^4}\sech^2\!\rbr{\frac{\sqrt{1-\sigma_n^2}(x-\mu)}{\sigma_n^2}}.
\end{align*}
Thus $I_{\sigma_n}=-\bbE[\frac{\partial^2}{\partial \mu^2}\log f_\mu(X)\mid\mu]$, proving \eqref{eq:gsr_information}.

Substituting $u=\sqrt{1-\sigma_n^2}(x-\mu)/\sigma_n^2$ gives
\begin{align*}
 \bbE\sbr{\sech^2\!\rbr{\frac{\sqrt{1-\sigma_n^2}(X-\mu)}{\sigma_n^2}}\mid\mu}
 &=\frac{\sigma_n}{\sqrt{1-\sigma_n^2}\sqrt{2\pi}}e^{-(1-\sigma_n^2)/(2\sigma_n^2)}
   \int_{\bbR}e^{-\sigma_n^2u^2/(2(1-\sigma_n^2))}\sech u\,du\\
 &=\sqrt{\frac\pi2}\frac{\sigma_n}{\sqrt{1-\sigma_n^2}} e^{-(1-\sigma_n^2)/(2\sigma_n^2)}
   \rbr{1+\cO(\sigma_n^2)}.
\end{align*}
The last equality follows from $\int_{\bbR}\sech u\,du=\pi$, $\int_{\bbR}u^2\sech u\,du<\infty$, and $0\le1-e^{-t}\le t$ for $t\ge0$.
Thus, we obtain \eqref{eq:gsr_information_asymptotic}.

For the last bound, reveal the mixture component labels and subtract the corresponding offsets, leaving $n$ independent $\cN(\mu,\sigma_n^2)$ observations conditional on $\mu$.
Revealing the labels cannot increase the location-minimax risk:
\[
R_{\mathrm M,n}^{\star}
\ge\inf_{\widetilde\mu}\sup_{\mu\in\bbR}
\bbE_{\bX\sim\cN(\mu,\sigma_n^2)^{\otimes n}}
[(\widetilde\mu(\bX)-\mu)^2]
=\frac{\sigma_n^2}{n}.
\]
See \autoref{prop:location_benchmarks}.
\end{proof}

\section{Additional experiments}
\label{app:experiments}

\paragraph{Joint-training setup.}
We provide further training details for Figures~\ref{fig:joint_softmax} and~\ref{fig:joint}.
The joint-training experiments below use the same setup unless stated otherwise.
We train the architecture of \autoref{TransformerSimplified} with both softmax MoE and GLU gates, updating $\theta$ and $\lambda$ jointly while keeping $\boldsymbol\omega$ fixed.
We use $m=8000$ fixed training tasks with contexts of length $n=150$ from \autoref{adaptiveMeanEstimation}, with $\sigma_\mu^2=0.5$, $K=2$, and $d=10$.
Each gradient-descent step uses all $m$ training tasks, with constant step sizes $20$ for $\theta$ and $100$ for $\lambda$, for $600{,}000$ steps.
We initialize $\theta_k(0)\sim\cN(0,0.25^2I_d)$ and sample the fixed gating queries $\omega_r\sim\cN(0,1)$ independently.
We draw the initial expert parameters independently and uniformly from $[0,1.9]$, rejecting draws with separation below $0.6$ and ordering them by magnitude.
Each joint-training experiment uses ten independent runs and $200$ validation tasks per family in each run.
In the joint-training plots, experts are relabeled so that expert $1$ is assigned to Gaussian contexts at the end of training.
Curves show medians, and shaded regions span the minimum and maximum across runs.

\subsection{Joint training, Gaussian-mixture alternative}
Figures~\ref{fig:app_joint_mixture_softmax} and~\ref{fig:app_joint_mixture} repeat Figures~\ref{fig:joint_softmax} and~\ref{fig:joint} with the Gaussian-mixture alternative and $\sigma_n^2=0.05$.
The gates move toward $\bar w\approx I_2$ and $|\lambda_1|$ decreases to zero, as with the uniform alternative.
Unlike with the uniform alternative, $|\lambda_2|$ settles near $2$, above the stopping target $\log(1/\sigma_n^2)/(4\sqrt{1-\sigma_n^2})\approx0.77$ of \autoref{convergence_of_first_layer}.
On Gaussian-mixture contexts, both architectures improve on the sample mean but remain above the component-mean error and the lower bound of \autoref{gsr_fisher}.
The component mean assigns observations to components according to their position relative to the sample mean, subtracts the corresponding offset $\pm\sqrt{1-\sigma_n^2}$, and averages the resulting values.

\begin{figure}[H]
    \centering
    \includegraphics[width=\linewidth]{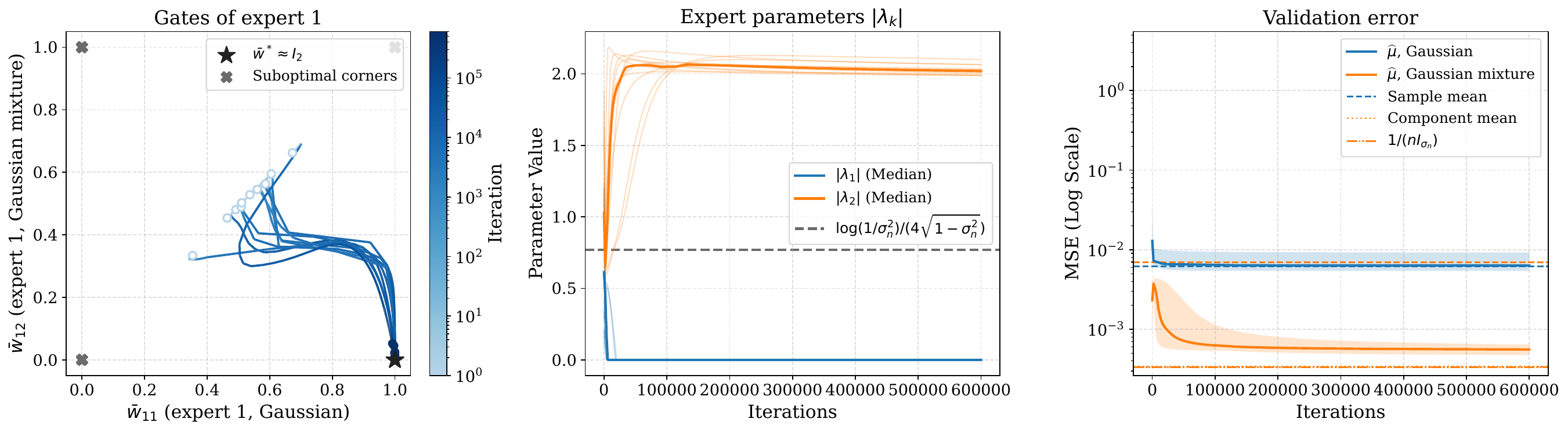}
    \caption{Softmax MoE, joint training with the Gaussian-mixture alternative, as in Figure~\ref{fig:joint_softmax}.
The dashed line in the middle panel marks $\log(1/\sigma_n^2)/(4\sqrt{1-\sigma_n^2})$.}
    \label{fig:app_joint_mixture_softmax}
\end{figure}

\begin{figure}[H]
    \centering
    \includegraphics[width=\linewidth]{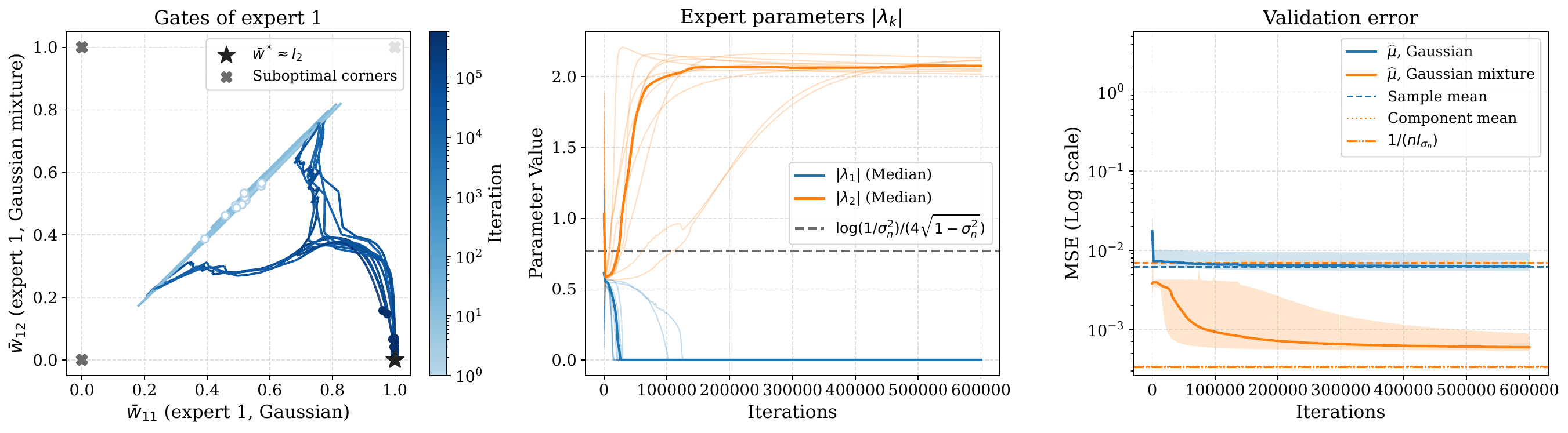}
    \caption{GLU, joint training with the Gaussian-mixture alternative, as in Figure~\ref{fig:joint}.
The dashed line in the middle panel marks $\log(1/\sigma_n^2)/(4\sqrt{1-\sigma_n^2})$.}
    \label{fig:app_joint_mixture}
\end{figure}

\subsection{Joint training: gating dynamics}
Figures~\ref{fig:app_joint_uniform_pi_bar} and~\ref{fig:app_joint_mixture_pi_bar} show the GLU gates and normalization error during joint training.
After an initial transient, the normalization error falls to about $10^{-2}$ before substantial family-specific selection develops.
The gates then move toward $\bar w\approx I_2$, with less complete selection and greater variation across runs for the uniform alternative.

\begin{figure}[H]
    \centering
    \includegraphics[width=\linewidth]{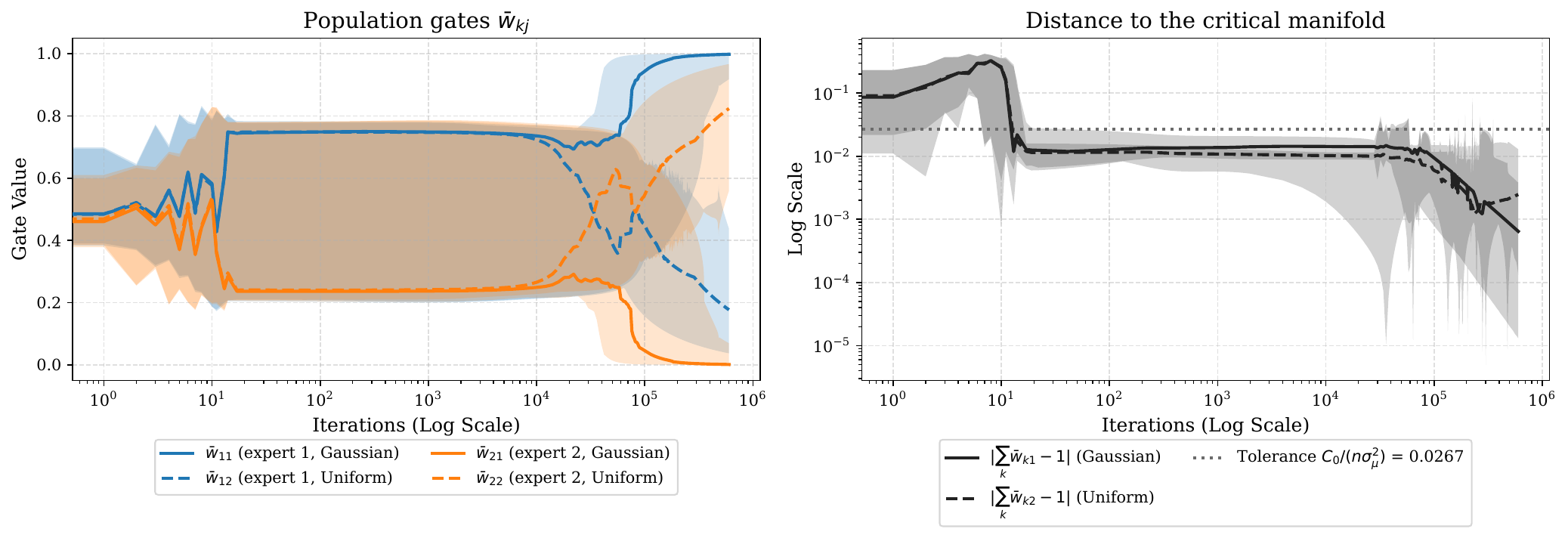}
    \caption{GLU, joint training with the uniform alternative: population gates $\bar w_{kj}$ (left) and normalization error $|\bar w_{1j}+\bar w_{2j}-1|$ (right).
Curves show medians, and shading spans all ten runs.}
    \label{fig:app_joint_uniform_pi_bar}
\end{figure}

\begin{figure}[H]
    \centering
    \includegraphics[width=\linewidth]{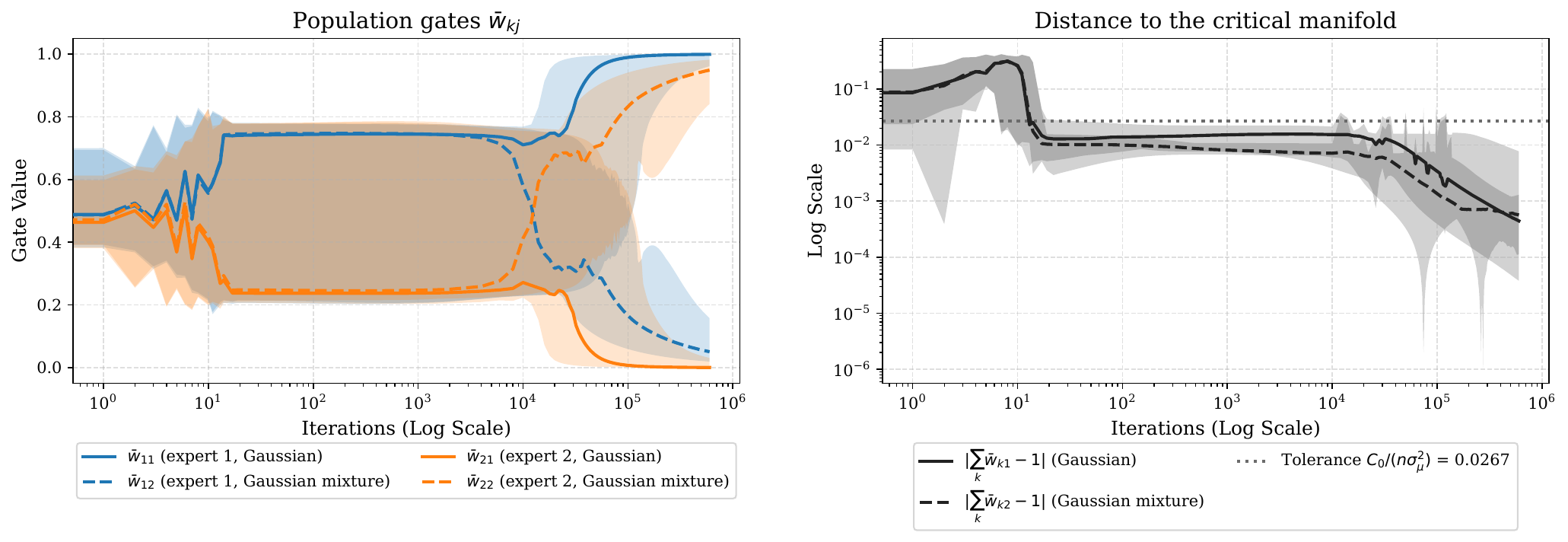}
    \caption{GLU, joint training with the Gaussian-mixture alternative, with panels and conventions as in Figure~\ref{fig:app_joint_uniform_pi_bar}.}
    \label{fig:app_joint_mixture_pi_bar}
\end{figure}

\subsection{Stagewise training}
\label{app:stagewise}
We train the GLU with \autoref{training_procedure}, using the uniform alternative with $\sigma_\mu^2=0.05$ and $n=500$.
The values of $m,K,d$, the initialization, and the numbers of runs and validation tasks are as in the joint-training setup above.
In the first stage, we minimize $\hat\cL_m$ over $\theta$ using $20{,}000{,}000$ gradient-descent steps with step size $100$.
In the second stage, we minimize $\hat\cL_m$ over $\lambda$ using $300{,}000$ gradient-descent steps with step size $1000$.

\paragraph{First stage: critical manifold.}
Figure~\ref{fig:manifold} shows $(\bar w_{1j},\bar w_{2j})$ during the first stage for Gaussian ($j=\mathrm G$) and uniform ($j=\mathrm U$) contexts.
The population gates first approach $\bar w_{1j}+\bar w_{2j} \approx 1$, before substantial specialization toward $e_1$ and $e_2$ for Gaussian and uniform contexts, respectively.
\begin{figure}[H]
    \centering
    \includegraphics[width=0.9\linewidth]{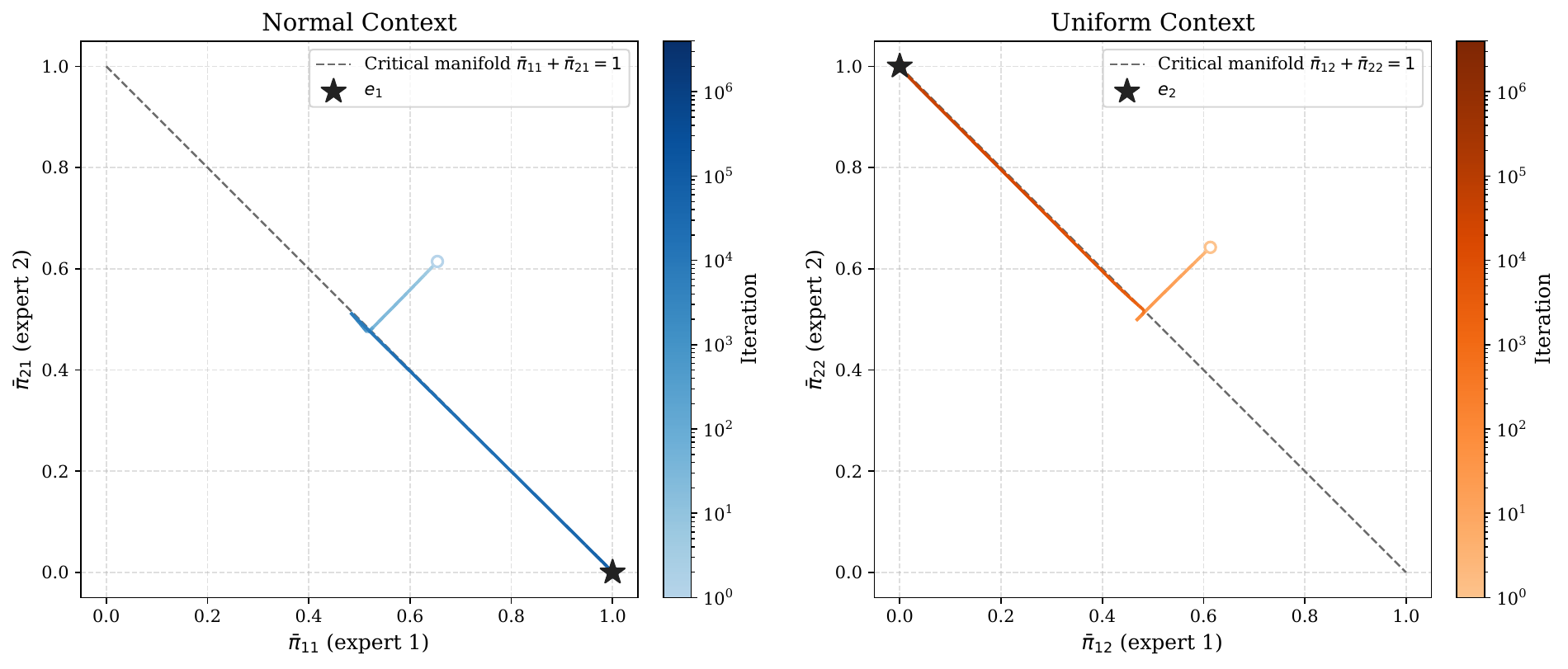}
    \caption{First stage with the uniform alternative: $(\bar w_{1j},\bar w_{2j})$ for Gaussian (left) and uniform (right) contexts across runs, colored by iteration.
The dashed line is the critical manifold $\bar w_{1j}+\bar w_{2j}=1$.}
    \label{fig:manifold}
\end{figure}

\paragraph{Per-expert accuracy.}
Figure~\ref{fig:expert_mse} shows $\bbE_j[(g_k-\mu)^2]$ for both experts and $\bbE_j[(\widehat\mu-\mu)^2]$ on validation tasks.
In the first stage, $\lambda$ is fixed, so the expert errors are constant and only the gating changes the error of $\widehat\mu$.
On Gaussian contexts, expert $1$ and $\widehat\mu$ stay near the sample-mean error.
On uniform contexts, the errors of expert $2$ and $\widehat\mu$ decrease further in the second stage toward the sample mid-range error.
Although expert $2$ becomes worse on Gaussian contexts during this stage, the error of $\widehat\mu$ remains approximately unchanged since $\bar w_{2\mathrm G}\approx0$ after the first stage.
The Gaussian-mixture results are in \autoref{app:stagewise_mixture}.

\begin{figure}[H]
    \centering
    \includegraphics[width=0.9\linewidth]{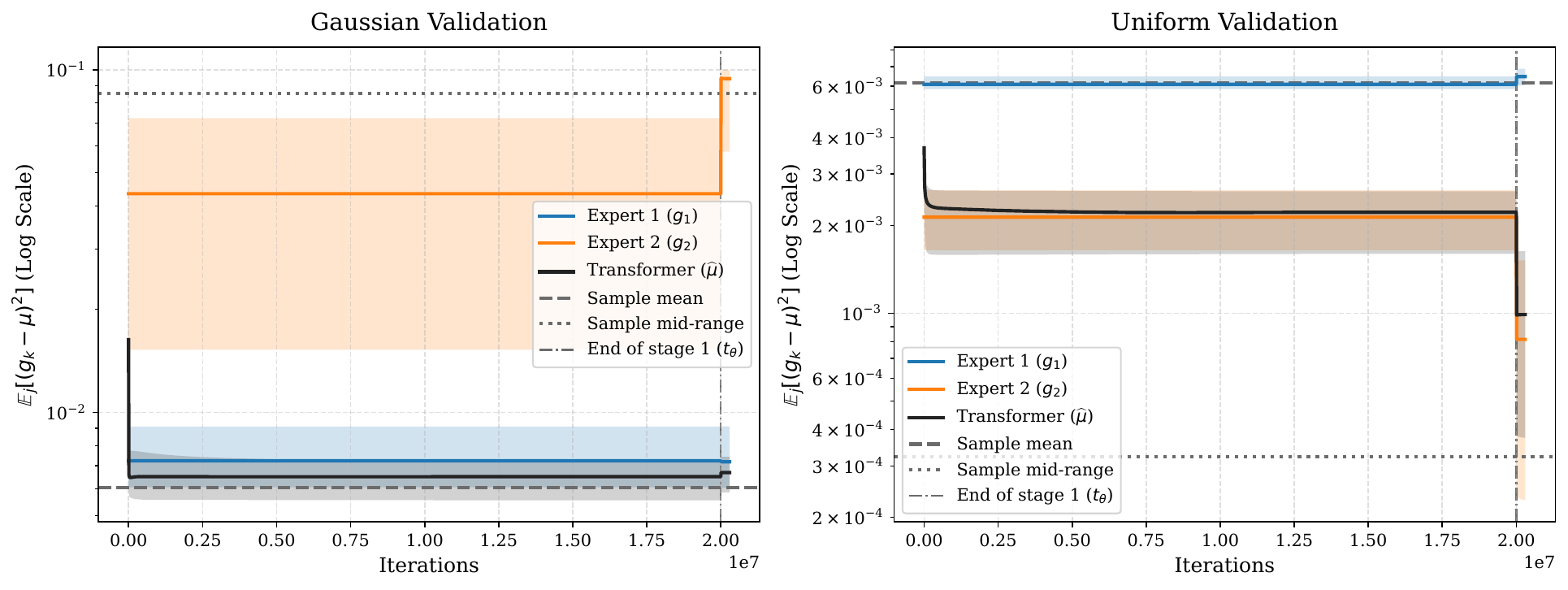}
    \caption{Uniform alternative: validation error of each expert $g_k$ and of $\widehat\mu$ on Gaussian (left) and uniform (right) contexts, with the sample mean and sample mid-range as references.
The vertical line marks $\ttheta$.}
    \label{fig:expert_mse}
\end{figure}

\subsection{Stagewise training, Gaussian-mixture alternative}
\label{app:stagewise_mixture}
We use the same Gaussian-mixture alternative as in the joint-training experiments, with the stagewise setup of \autoref{app:stagewise}.
Figure~\ref{fig:app_uniform_expert_mse} repeats the uniform errors for comparison with Figure~\ref{fig:app_mixture_expert_mse}.
In both cases, training the experts in the second stage reduces the alternative-family error while leaving the Gaussian error approximately unchanged.
For the mixture, the error remains above the component-mean error and the lower bound of \autoref{gsr_fisher}.
Figure~\ref{fig:app_mixture_manifold} shows the first-stage gates moving toward the appropriate expert for each family, with a residual normalization error as in Figure~\ref{fig:manifold}.

\begin{figure}[H]
    \centering
    \includegraphics[width=\linewidth]{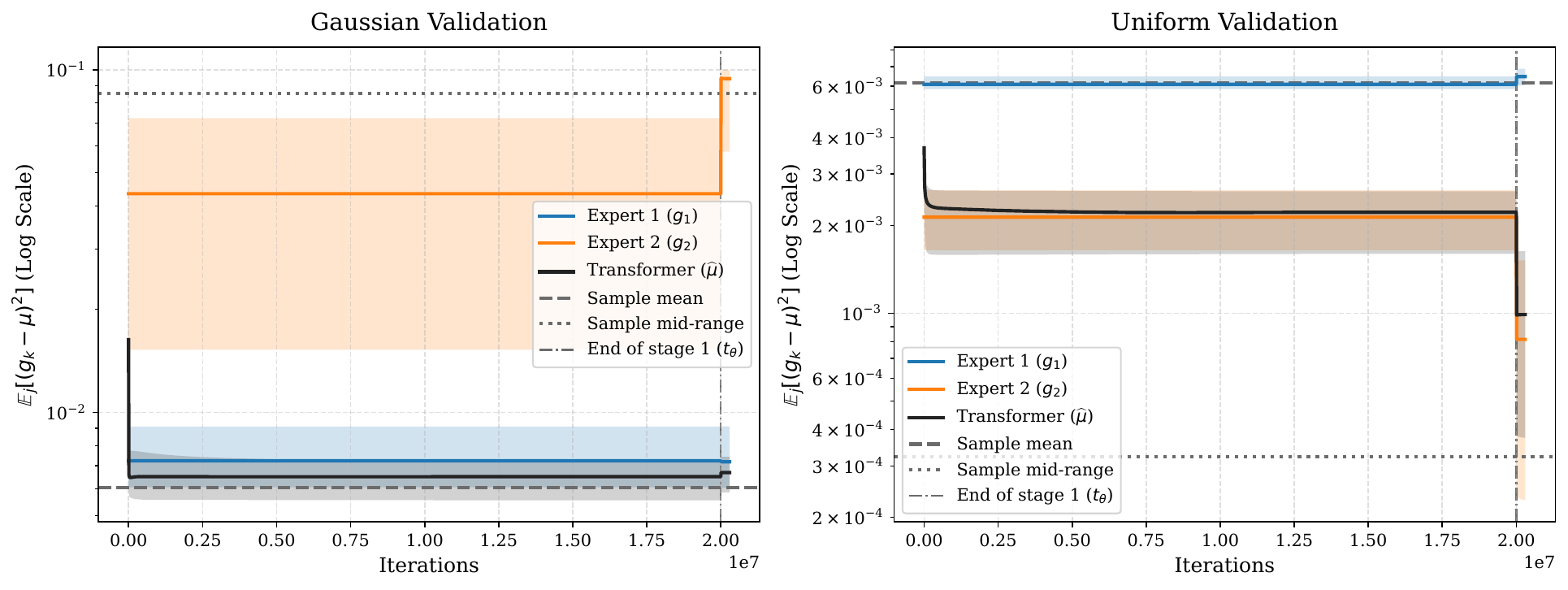}
    \caption{Uniform alternative: validation error of each expert $g_k$ and of $\widehat\mu$ on Gaussian (left) and uniform (right) contexts, repeated for comparison with Figure~\ref{fig:app_mixture_expert_mse}.
The vertical line marks $\ttheta$.}
    \label{fig:app_uniform_expert_mse}
\end{figure}

\begin{figure}[H]
    \centering
    \includegraphics[width=\linewidth]{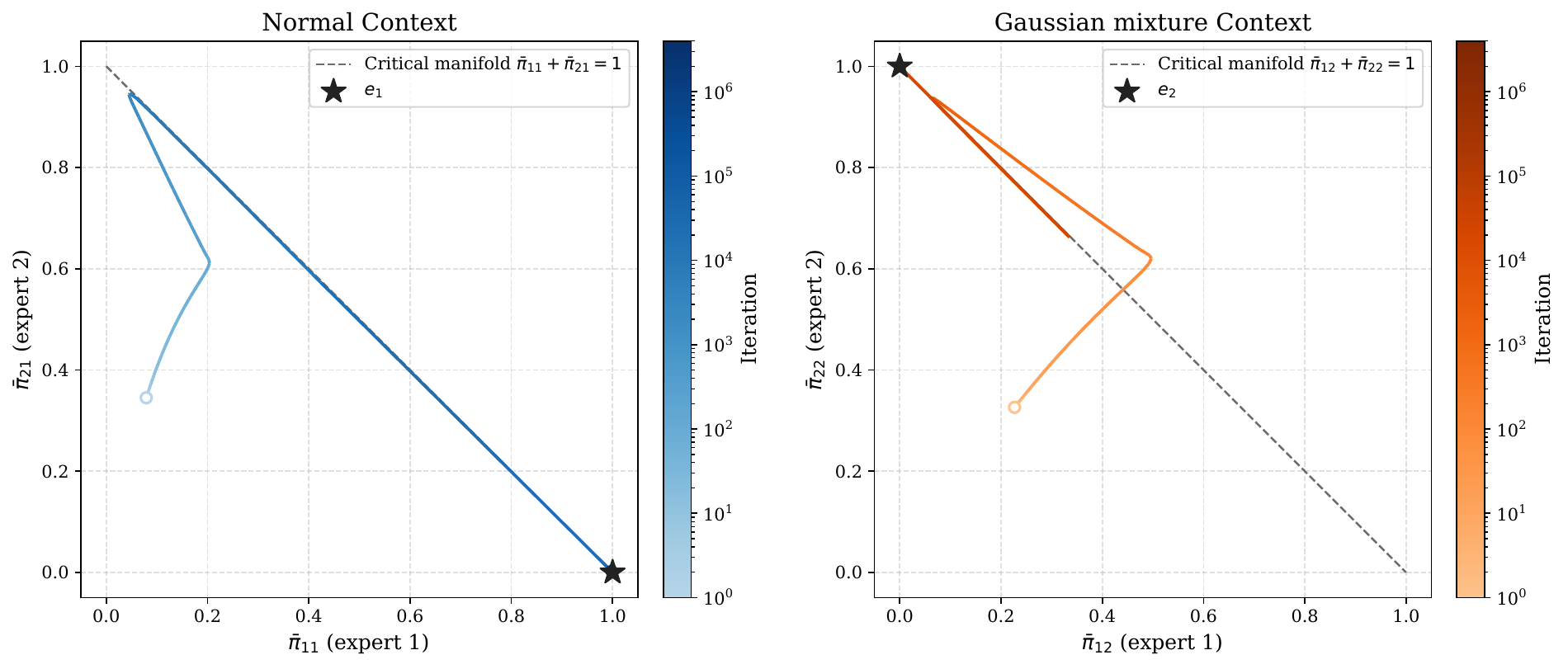}
    \caption{First stage with the Gaussian-mixture alternative: $(\bar w_{1j},\bar w_{2j})$ for Gaussian (left) and Gaussian-mixture (right) contexts across runs, colored by iteration.
The dashed line is the critical manifold $\bar w_{1j}+\bar w_{2j}=1$.}
    \label{fig:app_mixture_manifold}
\end{figure}

\begin{figure}[H]
    \centering
    \includegraphics[width=\linewidth]{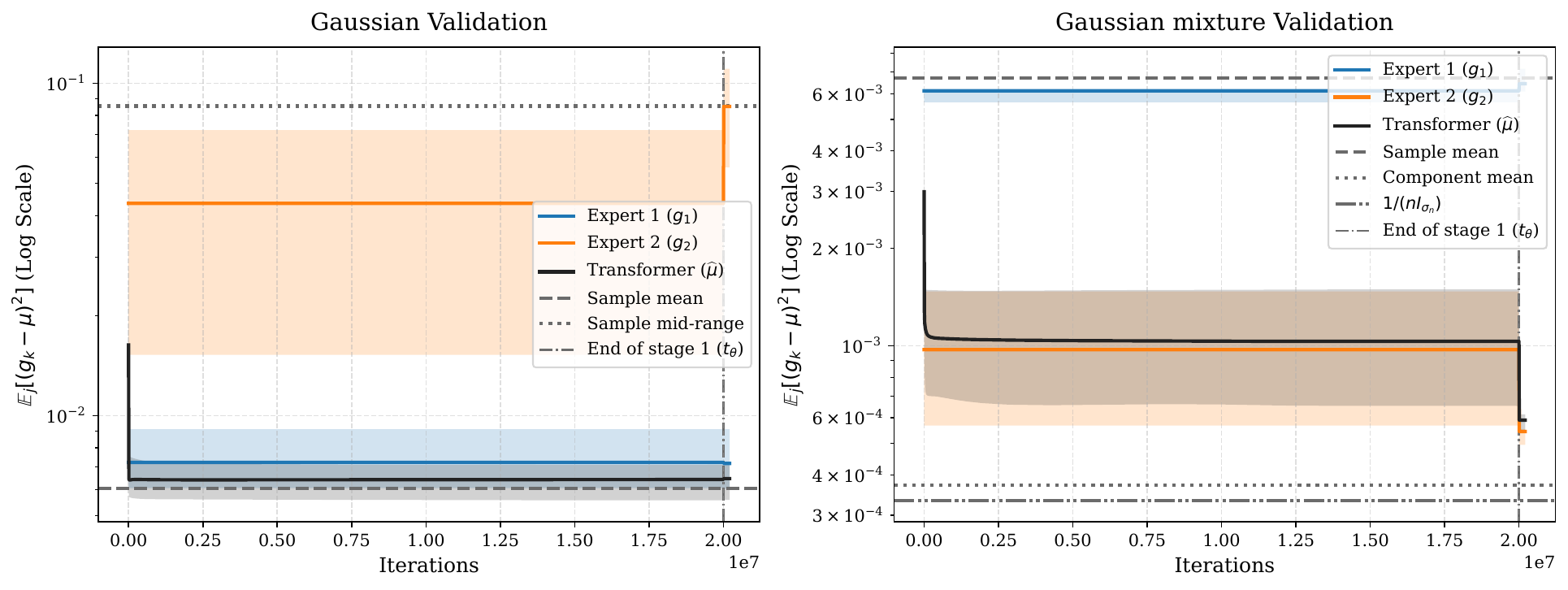}
    \caption{Gaussian-mixture alternative: validation error of each expert $g_k$ and of $\widehat\mu$ on Gaussian (left) and Gaussian-mixture (right) contexts.
The vertical line marks $\ttheta$.}
    \label{fig:app_mixture_expert_mse}
\end{figure}

\subsection{Relaxed architecture}
\label{app:relaxed}
We train ten attention parameters, whose attention outputs serve both as gating features and as inputs to the experts.
The attention outputs are not centered or normalized.
Each expert and each gate logit is a learned linear function of the ten attention outputs.
We use the uniform alternative and train with SGD with $n=200$, $m = 8{,}000$, learning rate $10$, $K=10$, and $\sigma_\mu^2=0.1$, and train all parameters jointly for $15{,}000$ epochs.
Each experiment uses five independent runs and $2{,}000$ held-out validation tasks.
Figure~\ref{fig:relaxed_loss} shows end-to-end training of this less constrained architecture.
The Gaussian validation error stays near the sample-mean error, while the uniform validation error falls well below it toward the sample mid-range error.

\begin{figure}[H]
    \centering
    \includegraphics[width=\linewidth]{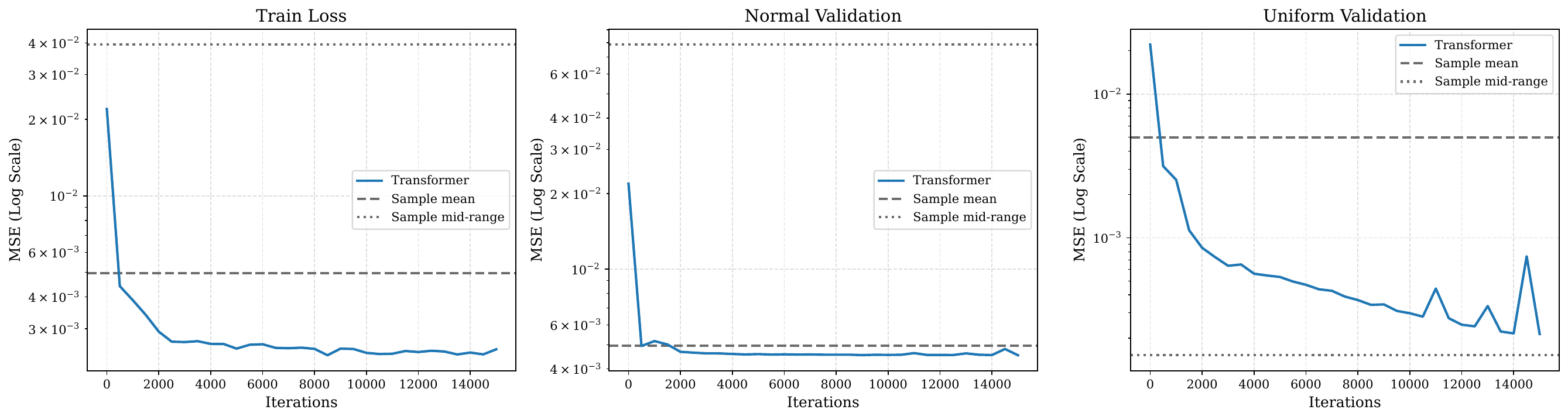}
    \caption{Less constrained architecture, end-to-end training with the uniform alternative.
Left: training loss.
Middle and right: validation error on Gaussian and uniform contexts, with the sample mean and sample mid-range as references.}
    \label{fig:relaxed_loss}
\end{figure}

\section{Role of LayerNorm}
\label{role_of_layernorm}
LayerNorm enforces the condition $Z_{\mathrm{GA}}<\min\wbr{Z_{\mathrm{GG}},Z_{\mathrm{AA}}}$ used in \autoref{avoidance_of_suboptimal_stationary_points}.
In our experiments, convergence is also observed without LayerNorm.
\autoref{fig:streamplots} illustrates the reduced flow studied in that lemma with and without LayerNorm.
For the convergence argument, it would suffice to establish convergence to the optimal assignment from a region containing the reduced initial states at $\tcrit$, with a margin for the approximation errors.
Since the gates are initialized near $1/2$, a neighborhood of $(\bar w_{1\mathrm G},\bar w_{1\mathrm A})=(1/2,1/2)$ containing these states could suffice.
The stream plot suggests convergence from such a central region without LayerNorm.
Under the strict expert hierarchy and $Z_{\mathrm{GA}}>Z_{\mathrm{AA}}>0$, the reduced flow nevertheless has a region converging to the suboptimal corner $(0,0)$: sufficiently close to this corner, when $\bar w_{1\mathrm G}/\bar w_{1\mathrm A}$ is sufficiently small, both coordinates and their ratio decrease, keeping the trajectory in this region and forcing both coordinates to zero.
This does not preclude convergence from the central region.

\begin{figure}[ht]
    \centering
    \includegraphics[width=0.8\linewidth]{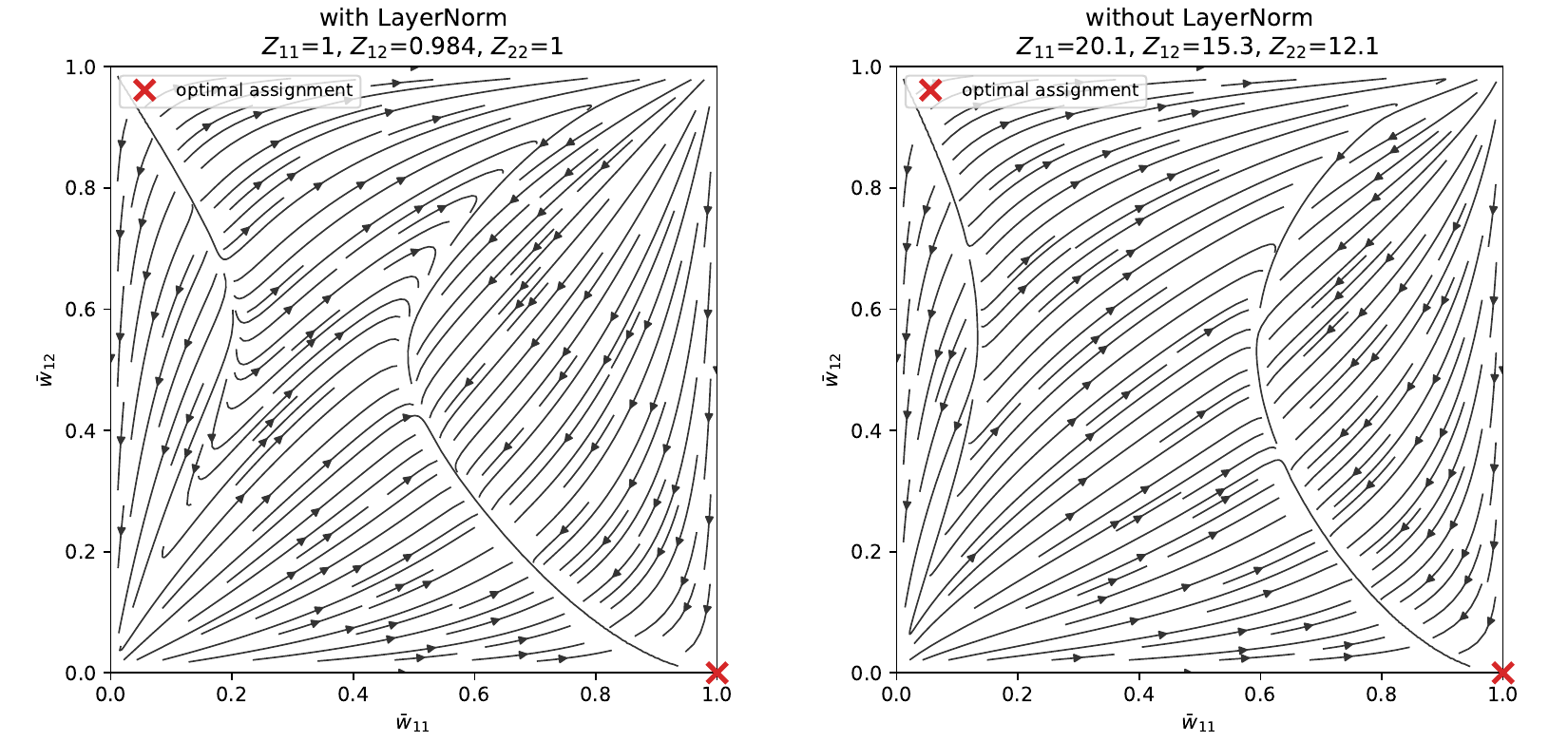}
    \caption{Stream plots of the reduced flow studied in \autoref{avoidance_of_suboptimal_stationary_points}.
Left: population features with LayerNorm.
Right: uncentered population features at $\mu=0$ without LayerNorm.}
    \label{fig:streamplots}
\end{figure}

\section{Extended related work}
\label{app:related}
\paragraph{In-context learning of statistical procedures.}
Prior-data fitted networks (PFNs) train transformers on synthetic tasks to approximate Bayesian prediction in context~\citep{muller2024transformersbayesianinference}.
Their adaptation to tabular data has produced models that perform strongly on standardized benchmarks, particularly for small and medium datasets~\citep{hollmann2023tabpfntransformersolvessmall,hollmann2025accuratepredictions,erickson2025tabarena,ye2025closer}.
Related pretrained models estimate heterogeneous treatment effects, survival curves, mutual information, and Poisson means~\citep{balazadeh2025causalpfn,qi2026survivalpfn,peyrard2025metastatisticallearningsupervisedlearning,teh2025empiricalbayes}, and amortized simulation-based inference likewise maps simulated datasets to parameters~\citep{Cranmer_2020}.
Most of these systems target prediction at a query point.
These works demonstrate that pretraining can produce effective predictors and parameter estimators.
Our focus is on how gradient-based pretraining learns an estimator with rates that depend on the unknown family.

\paragraph{Statistical theory of pretrained predictors.}
Constructive and empirical analyses study which algorithms trained transformers implement in context~\citep{akyurek2023learningalgorithmincontextlearning,chen2024transformersutilizemultiheadattention,li2024onelayertransformerprovablylearns}.
\citet{bai2023transformersstatisticiansprovableincontext} construct transformers that select among base algorithms, either by testing summary statistics of the context or by validating candidates on a held-out split, with guarantees for the empirical risk minimizer.
\citet{kim2024transformersminimaxoptimalnonparametric} show that a deep network followed by linear attention attains minimax rates for nonparametric regression over a known function class, and argue that representations fixed after pretraining limit adaptivity.
\citet{wakayama2025bayesian} and \citet{ma2025adaptivity} analyze pretrained transformers as approximations to Bayesian posterior predictors~\citep[see also][]{nagler2023statisticalfoundationspriordatafitted}.
In particular, \citet{ma2025adaptivity} use posterior contraction to obtain rates that adapt to unknown smoothness and effective dimension.
\citet{cannella2026universalpriorssolvingempirical} establish near-optimal regret for an idealized pretrained Bayes estimator in Poisson empirical Bayes, with concurrent guarantees for nonparametric Bayes estimators in the Gaussian case~\citep{ignatiadis2026compounddecisionsempiricalbayes}.
These results show that the empirical risk minimizer, or an idealized Bayes predictor, can be adaptive.
They do not follow the training trajectory, and they treat the adaptive map as a black box.
In our model, adaptation is carried by an identifiable internal parameter, the attention scale of each expert, which moves the estimator across a family of estimators.

\paragraph{Training dynamics of attention.}
A complementary line analyzes how gradient-based training shapes attention on in-context tasks, including linear attention for linear regression~\citep{zhang2023trainedtransformerslearnlinear,zhang2024incontextlearninglineartransformer}, softmax attention~\citep{huang2023incontextconvergencetransformers}, mean-field dynamics~\citep{kim2024transformerslearnnonlinearfeatures}, and learning of low-dimensional structure~\citep{oko2024pretrainedtransformer}.
Closest to our analysis, \citet{chen2024trainingdynamicsmultiheadsoftmax} show that gradient flow on multi-head softmax attention allocates each head to one task of a multi-task linear regression, after warm-up and emergence phases.
There, the task handled by each head is a fixed coordinate block determined by the task structure and the initialization.
Here, the family varies across contexts and must be inferred from each sample, so specialization must be paired with learned routing.
They also find that concentrated attention is suboptimal for noisy linear tasks, whereas the irregular uniform family calls for attention concentrated on the extremes.
For the GLU, our gating dynamics separate into fast normalization and slow selection, which we analyze with Tikhonov's theorem~\citep{berglund2001perturbationtheorydynamicalsystems}.

\paragraph{Adaptive estimation.}
In statistics, an estimator is adaptive over a collection of models if, for each model, it attains the rate achievable when that model is known~\citep{lepski1991adaptive,donoho1995adapting,tsybakov2009introduction}.
Location estimation provides a classical parametric instance: for symmetric densities with finite Fisher information, the location can be estimated asymptotically as efficiently without knowing the shape as with it~\citep{stein1956efficient,beran1974asymptotically,stone1975adaptive,bickel1982adaptive,bickel1993efficient}.
Differences in rate between location families instead arise from irregularity: the uniform family has parameter-dependent support and admits the faster rate $n^{-2}$~\citep{ibragimov1981statistical}.
Classical adaptive location estimators estimate the score function and apply a one-step correction, whereas the procedure learned here selects a point in a family of exponentially tilted averages.

\paragraph{Set architectures, gating, and mixtures of experts.}
Deep Sets~\citep{zaheer2018deepsets} and Set Transformers~\citep{lee2019settransformerframeworkattentionbased} provide architectures for functions of unordered data, and our model is a small Set Transformer specialized to scalar estimation.
Its second layer is either a softmax mixture of experts~\citep{jacobs1991adaptive,shazeer2017outrageously} or a gated linear unit~\citep{dauphin2017languagemodelinggatedconvolutional,shazeer2020glu}.
Comparing the two isolates the effect of enforcing a statistical constraint, namely normalization of the expert weights and hence exact translation equivariance, versus learning it.
\citet{bu2025provableincontextvectorarithmetic} analyze how gradient-based training learns to retrieve task concepts from context.
They show that training on question answer data enables this behavior, whereas training on their in-context demonstration distributions can fail.
Their results concern factual recall, while we study how learned routing and expert specialization yield family-specific estimation rates.

\end{document}